\documentclass{article}
\usepackage{iclr2027_conference,times}

\usepackage{amsmath,amsfonts,bm}

\def\eqref#1{equation~\ref{#1}}

\def\1{\bm{1}}

\DeclareMathAlphabet{\mathsfit}{\encodingdefault}{\sfdefault}{m}{sl}
\SetMathAlphabet{\mathsfit}{bold}{\encodingdefault}{\sfdefault}{bx}{n}

\usepackage{amssymb,amsthm,mathtools}
\usepackage{booktabs}
\usepackage{hyperref}
\usepackage{url}
\usepackage{graphicx}
\usepackage{subcaption}
\usepackage{wrapfig}
\usepackage{multirow}
\usepackage[letterpaper,margin=1.2in]{geometry}
\newtheorem{theorem}{Theorem}
\newtheorem{proposition}[theorem]{Proposition}
\newtheorem{lemma}[theorem]{Lemma}
\newtheorem{corollary}[theorem]{Corollary}
\newtheorem{assumption}[theorem]{Assumption}
\newtheorem{definition}[theorem]{Definition}
\newtheorem{remark}[theorem]{Remark}

\newcommand{\PD}{\mathbb{S}_{++}}
\newcommand{\Sym}{\mathbb{S}}
\newcommand{\M}{\mathcal{M}}
\newcommand{\Mbar}{\overline{\mathcal{M}}}
\newcommand{\Gclass}{\mathcal{G}_{p,L_{\max}}}
\newcommand{\diag}{\operatorname{diag}}
\newcommand{\tr}{\operatorname{tr}}
\newcommand{\cl}{\operatorname{cl}}
\newcommand{\od}{\odot}

\newenvironment{feedbackrevision}{\color{black}}{}
\newcommand{\feedbackedit}[1]{{\color{black}#1}}

\title {Differentiable Structure Learning for Cyclic Linear Gaussian Models with Latent Confounders}

\author{\begin{minipage}{\dimexpr\textwidth-2\tabcolsep\relax}
\textbf{Sadegh Khorasani\textsuperscript{1}, Ali Najar\textsuperscript{1},
Saber Salehkaleybar\textsuperscript{2},}
\textbf{Negar Kiyavash\textsuperscript{3}}\\[6pt]
\normalfont\small
\textsuperscript{1}School of Computer and Communication Sciences,
EPFL, Lausanne, Switzerland\\
\textsuperscript{2}Leiden Institute of Advanced Computer Science (LIACS), Leiden University, Leiden, The Netherlands\\
\textsuperscript{3}College of Management of Technology,
EPFL, Lausanne, Switzerland\\[4pt]
\texttt{sadegh.khorasani@epfl.ch}\quad\texttt{anajar13750@gmail.com}\quad
\\
\texttt{s.salehkaleybar@liacs.leidenuniv.nl} \quad\texttt{negar.kiyavash@epfl.ch}
\end{minipage}}

\iclrfinalcopy 
\begin{document}

\maketitle
\lhead{Preprint} 
\begin{abstract}

We study causal structure learning from observational data in linear Gaussian structural causal models in the presence of directed cycles and an unknown number of exogenous latent confounders, bounded by a given maximum. We derive the covariance of the observed variables and introduce marginal quasi-equivalence, which characterizes when different causal models share a full-dimensional subset of the observational distributions they can generate. We formulate structure learning as minimization of the Gaussian negative log-likelihood with a logarithmically scaled complexity penalty that counts directed edges and latent variables. For a fixed number of observed variables and a fixed upper bound on latent variables, we establish consistency of global score minimizers up to marginal quasi-equivalence under algebraic faithfulness, structural minimality, and model-overlap assumptions. We parameterize the inclusion of directed edges and candidate latent variables using Bernoulli gates, whose continuous probabilities are optimized jointly with the structural coefficients. Averaging the penalized negative log-likelihood over these gates yields an objective with a closed-form differentiable complexity penalty. We prove that this expected objective has the same global infimum as the corresponding discrete structure-learning objective. Experimental results show that our approach achieves lower recovery error than previous methods in several experimental settings.

\end{abstract}

\section{Introduction}
\label{sec:introduction}
\begingroup

Causal graphical models describe causal relationships among variables.
Learning their structure from observational data is a central problem in
causal inference. In many systems, these relationships involve feedback
loops and unobserved common causes, known as latent confounders. This
motivates structure-learning methods that allow both directed cycles and
latent variables.

Several lines of research address these two challenges separately.
Work on cyclic linear Gaussian models has developed distributional
characterizations and score-based learning methods under the assumption
of no latent confounders \citep{ghassami2020,yi2024}. In contrast,
observational methods for latent confounding commonly retain acyclicity
\citep{kaltenpoth2023,ma2024spot}. Approaches that allow both cycles and
confounding use different model or data assumptions. In particular,
\citet{amendola2020} consider mixed graphs, whose bidirected edges represent
possible correlations between structural noise variables. Their
simple-graph restriction permits at most one edge of any type between
each pair of observed variables. It therefore excludes reciprocal directed
edges and pairs joined by both a directed and a bidirected edge. Other
approaches use interventional data \citep{hyttinen2012,sethuraman2025} or
conditional-independence assumptions for nonlinear cyclic models
\citep{forre2018}.

The distributional-equivalence and consistency results for cyclic models without latent variables do not directly extend after latent variables are marginalized, since candidate graphs may contain different numbers of latent variables and have different structural parameterizations. Likewise, results established under the simple-graph restriction do not directly apply to our setting, where these more general structures are allowed. We do not impose this restriction and allow arbitrary directed cycles among observed variables together with explicit exogenous latent confounders.

We study observational linear Gaussian structural causal models with arbitrary directed cycles among observed variables and an unknown, but bounded, number of independent exogenous latent confounders. Rather than recovering a unique latent graph, which is generally not identifiable from the observed distribution, we formulate the problem directly in terms of the families of observed covariance matrices induced by candidate graphs. This perspective leads to a marginal notion of quasi-equivalence that captures when two graphs generate a full-dimensional common subset of observational distributions, while allowing their explicit latent structures to differ.

Our contributions are as follows.
\begin{itemize}
    \item
    We derive the observed covariance family induced by a cyclic linear
    Gaussian SCM with exogenous latent variables and establish its basic
    geometric properties. We define marginal distribution equivalence
    and marginal quasi-equivalence in observed covariance space,
    extending the positive-overlap perspective of
    \citet[Definition~9]{ghassami2020} to candidate graphs with different
    latent-variable structures (Section~\ref{sec:marginal-models}).

    \item
    We formulate structure learning as joint optimization over the
    discrete graph structure and its continuous parameters, using the
    Gaussian negative log-likelihood and a logarithmically scaled penalty
    that counts directed edges and latent variables. For a fixed number
    of observed variables and a fixed upper bound on latent variables,
    we show that global score minimizers asymptotically select covariance
    families whose closures contain the true covariance. 
    Under additional conditions (in particular, under algebraic faithfulness, structural minimality, and full-dimensional model overlap), we obtain consistency up to marginal quasi-equivalence (Section~\ref{sec:consistency-result}).

    \item
    We use Bernoulli gates for observed edges, latent-variable activation,
    and latent-to-observed edges, optimizing their probabilities jointly
    with the continuous model parameters. Averaging the penalized objective over
    these gates yields a differentiable expected objective with a
    closed-form complexity penalty. We prove that this expected objective
    has the same global infimum as the corresponding discrete
    structure-learning problem, so the continuous probability
    parameterization preserves the globally optimal objective value
    (Section~\ref{sec:bernoulli}).

    \item
    We evaluate the approach on synthetic cyclic SCMs with varying graph
    density and latent confounding, as well as on simulated GeneNetWeaver
    networks. Across several settings with latent variables, the method achieves lower bidirectional covariance-family compatibility error than the considered baselines, while scaling to substantially larger observed graphs.
    (Section~\ref{sec:experiments}).
\end{itemize}
Appendix~\ref{app:related-work} provides further comparisons with related work.
\endgroup

\section{Problem Setting}
\label{sec:problem-setting}

We study causal structure learning from observational data in linear Gaussian
structural causal models (SCMs) with feedback loops among observed variables
and exogenous latent confounders, meaning latent variables with no incoming
edges. The unknown structure consists of edges among observed variables,
edges from latent to observed variables, and the
number of latent variables.

\subsection{Notation and conventions}
\label{sec:notation}

For a positive integer $d$, let $I_d$ denote the $d\times d$ identity matrix,
$\Sym^d$ the space of real symmetric $d\times d$ matrices, and
$\PD^d\subset\Sym^d$ the set of symmetric positive-definite matrices. For a
symmetric matrix $A$, the notation $A\succ0$ and $A\succeq0$ means positive
definiteness and positive semidefiniteness, respectively. We write $A^\top$
for the transpose of $A$ and, when $A$ is square, $\det(A)$, $\tr(A)$, and
$\rho(A)$ for its determinant, trace, and spectral radius (the largest
absolute value of its eigenvalues). For invertible
$A$, we use $A^{-\top}:=(A^{-1})^\top$. The operator
$\diag(a_1,\ldots,a_d)$ forms a diagonal matrix, and $\od$ denotes the
elementwise product. For a matrix $A$, $|A|$ denotes entrywise absolute
value, $\|A\|_1=\sum_{i,j}|A_{ij}|$ the entrywise $\ell_1$ norm,
$\|A\|_2$ the spectral norm, and $\|A\|_F$ the Frobenius norm. For a finite
set, vertical bars instead denote cardinality. Random vectors are treated
as column vectors, and $\operatorname{Cov}(Y)$ denotes the covariance
matrix of a random vector $Y$.

For subsets $T\subseteq S$ of a Euclidean space, $\cl_S(T)$ denotes the
closure of $T$ relative to $S$: all limits in $S$ of convergent sequences
from $T$. An \emph{algebraic set} is the set of common solutions to finitely
many polynomial equalities with real coefficients. It is \emph{irreducible}
if it is nonempty and cannot be written as the union of two strictly
smaller algebraic sets. A \emph{semialgebraic set} is a finite union of sets,
each defined by finitely many polynomial equalities and inequalities.
To say that a polynomial \emph{vanishes on} a set means that it equals zero
at every point in that set.

For a nonempty semialgebraic set $T\subseteq\mathbb R^m$, $\dim T$ is the
largest $k\in\{0,\ldots,m\}$ for which projecting $T$ onto some $k$
coordinates gives a set containing a nonempty open ball in $\mathbb R^k$.
Here projection means keeping the selected coordinates and discarding the
others. We set $\dim\varnothing=-1$. For nonempty real algebraic sets, this
dimension agrees with algebraic dimension \citep[Section~3.3]{coste2000}. For sets of
symmetric $p\times p$ matrices, the coordinates are the $p(p+1)/2$
upper-triangular entries. For semialgebraic sets $T\subseteq S$, $T$ is called
\emph{full-dimensional in $S$} when $\dim T=\dim S$.
Appendix~\ref{app:geometry-definitions} gives formal definitions (including
open balls), references, and examples.

A candidate graph is denoted by $G^+=(V,E)$, where
$V=O\mathbin{\dot\cup}L$ is partitioned into the observed vertex set $O$ and
the latent vertex set $L$. We write $p=|O|$, $\ell=|L|$, and
$0\leq\ell\leq L_{\max}$, where $L_{\max}<\infty$ is a prescribed upper
bound on the number of latent variables. The superscript $+$
indicates that the graph includes both types of vertices. Matrices are
indexed with observed variables before latent variables. We use the
row-tail, column-head convention: $B_{ij}$ is the coefficient associated
with the directed edge $i\to j$. Graph compatibility requires $B_{ij}=0$
whenever $(i,j)\notin E$. Conversely, an edge in $E$ permits, but does not
require, its coefficient to be nonzero. Thus a graph specifies which
coefficients are allowed to be nonzero, not which must be nonzero at every parameter
value. Covariance-family notation and the two notions of graph equivalence
are defined in Section~\ref{sec:marginal-models}.

\subsection{Cyclic structural equations with exogenous confounders}

Let $X=(X_O^\top,X_L^\top)^\top\in\mathbb R^{p+\ell}$ collect the observed
and latent variables, and partition the structural noise vector $N$ in the
same order. The full zero-mean Gaussian SCM is
\begin{equation}
    X=B^\top X+N,
    \qquad N\sim\mathcal{N}(0,\Omega),
    \qquad \Omega\succ0\text{ diagonal},
    \label{eq:sem}
\end{equation}
where $B\in\mathbb R^{(p+\ell)\times(p+\ell)}$ respects the graph-imposed
zero restrictions and $B_{ii}=0$. The diagonal Gaussian noise covariance
makes the components of $N$ mutually independent.

We allow arbitrary directed cycles among observed variables and require
only that $I_{p+\ell}-B$ be invertible. For each noise realization, the
unique solution is
$X=(I_{p+\ell}-B)^{-\top}N$, which defines a nonsingular Gaussian
distribution.

Latent vertices are exogenous: they have no parents, and their only
permitted children are observed vertices. We additionally require each
latent vertex to have at least two permitted edges to observed vertices.
The resulting block structure is
\begin{equation}
    B=\begin{bmatrix}B_{OO}&0\\ \Lambda&0\end{bmatrix},
    \qquad
    \Omega=\begin{bmatrix}\Omega_O&0\\0&\Omega_L\end{bmatrix},
    \label{eq:block-sem}
\end{equation}
where $B_{OO}\in\mathbb R^{p\times p}$,
$\Lambda\in\mathbb R^{\ell\times p}$, and $\Omega_O$ and $\Omega_L$ are
positive diagonal matrices of sizes $p\times p$ and $\ell\times\ell$,
respectively. The $h$th row of $\Lambda$ contains the coefficients from the
$h$th latent variable to its observed children; these coefficients are
called \emph{latent loadings}. Equation~\ref{eq:block-sem}
gives
\[
    X_O=B_{OO}^\top X_O+\Lambda^\top X_L+N_O,
    \qquad X_L=N_L.
\]
Thus the latent variables are mutually independent and independent of the
observed structural noises $N_O$. The block-triangular
structure also gives
$\det(I_{p+\ell}-B)=\det(I_p-B_{OO})$, so the structural equations have a
unique solution for every noise realization exactly when
$\det(I_p-B_{OO})\neq0$.

The two-child restriction removes zero- and one-child latent vertices
without excluding any observed distribution representable with such
vertices. A latent variable with no child has no observational effect;
the variance contribution of a latent variable with only one child can be
absorbed into that child's diagonal noise variance. This restriction is
imposed on the permitted graph edges, not on the number of nonzero
loadings at each parameter value; other latent variables may still be
removable without changing the observed distribution. When $\ell=0$, the
latent blocks are empty and their
covariance contribution is understood to be the zero $p\times p$ matrix.

\subsection{Observed marginal distribution}
\label{sec:observed-marginal}

We assume that only $X_O$ is observed. The statistical analysis uses independent and
identically distributed (i.i.d.) observations
from its zero-mean marginal distribution. Eliminating $X_L$ introduces the
effective noise
\[
    \varepsilon_O:=N_O+\Lambda^\top N_L,
    \qquad X_O=B_{OO}^\top X_O+\varepsilon_O.
\]
Although the components of $N$ are independent, the components
of $\varepsilon_O$ need not be independent.

\begin{lemma}[Observed marginal covariance]
\label{lem:marginal-covariance}
Under the preceding assumptions,
\begin{equation}
    X_O\sim\mathcal{N}(0,\Sigma_O),\qquad
    \Sigma_O=(I_p-B_{OO})^{-\top}\Psi(I_p-B_{OO})^{-1},
    \label{eq:marginal-covariance}
\end{equation}
where the effective noise covariance is
\begin{equation}
    \Psi=\operatorname{Cov}(\varepsilon_O)
    =\Omega_O+\Lambda^\top\Omega_L\Lambda.
    \label{eq:effective-noise}
\end{equation}
Both $\Psi$ and $\Sigma_O$ are positive definite.
\end{lemma}

The matrix $\Omega_O$ describes the observed structural noise covariance,
whereas $\Lambda^\top\Omega_L\Lambda$ describes the covariance contribution
of the latent variables. The proof is given in
Appendix~\ref{app:proof-marginal-covariance}.

\section{Marginal Equivalence and Score}
\label{sec:equivalence-score}

\begin{feedbackrevision}
We compare graphs through their observed covariance families and define
the penalized likelihood used for structure learning.

\end{feedbackrevision}
\subsection{Marginal covariance families and equivalence}
\label{sec:marginal-models}

For a fixed graph $G^+$, let $\mathcal P(G^+)$ denote the set of parameters
$\theta=(B_{OO},\Lambda,\Omega_O,\Omega_L)$ satisfying its graph-imposed zero
restrictions, positive diagonal noise variances, and
$\det(I_p-B_{OO})\neq0$. We call these parameter values \emph{admissible}.
The graph's observed covariance family is
\begin{equation}
    \M(G^+)=\{\Sigma_O(\theta):\theta\in\mathcal{P}(G^+)\}
    \subseteq\PD^p.
    \label{eq:covariance-model}
\end{equation}
Because the observed distributions are Gaussian with known zero mean,
$\M(G^+)$ determines the entire observed distribution family.

We use two closures of this covariance family.
Write $\Mbar(G^+)=\cl_{\PD^p}\M(G^+)$ for its relative Euclidean closure.
It contains all positive-definite limits of covariances generated by
admissible parameter sequences, but does not include singular covariance
matrices.
In contrast, the \emph{real Zariski closure}
$V(G^+)=\overline{\M(G^+)}^{\mathrm{Zar}}\subseteq\Sym^p$ consists of all
symmetric matrices satisfying every polynomial equality that holds for
every covariance matrix in $\M(G^+)$. This definition
uses only polynomial equalities. Matrices in $V(G^+)$ need not be positive
definite or satisfy the other inequalities required to belong to $\M(G^+)$.
Appendix~\ref{app:closure-properties} explains the containment relations
between these sets and the semialgebraicity of the relative closure.

\begin{proposition}[Model geometry]
\label{prop:model-geometry}
For each fixed graph, $\M(G^+)$ is semialgebraic, $V(G^+)$ is irreducible,
and
\begin{equation}
    \dim\M(G^+)=\dim\Mbar(G^+)=\dim V(G^+).
    \label{eq:model-dimensions}
\end{equation}
\end{proposition}

\begin{feedbackrevision}
The proof is in Appendix~\ref{app:proofs}. 

\end{feedbackrevision}
\begin{definition}[Marginal equivalence]
Let $G_1^+$ and $G_2^+$ have the same labeled observed vertex set; their
latent vertex sets may differ. They are \emph{marginally distribution
equivalent}, written $G_1^+\equiv_O G_2^+$, if
$\M(G_1^+)=\M(G_2^+)$. They are \emph{marginally quasi-equivalent}, written
$G_1^+\cong_O G_2^+$, if
\begin{equation}
\begin{split}
    \dim\bigl(\M(G_1^+)\cap\M(G_2^+)\bigr)
    &=\dim\M(G_1^+)=\dim\M(G_2^+).
\end{split}
\label{eq:quasi-equivalence}
\end{equation}
\end{definition}

\begin{feedbackrevision}
Distribution equivalence implies quasi-equivalence. For quasi-equivalence,
the intersection must have the dimension of each family, not necessarily
$p(p+1)/2$, the dimension of $\Sym^p$. A lower-dimensional intersection
in either family is insufficient.

\end{feedbackrevision}
This definition follows the positive-measure overlap idea of
\citet[Definition~9]{ghassami2020}, expressed here through the dimension of
the intersection of observed covariance families.
Appendix~\ref{app:quasi-equivalence-rationale}
explains this formulation for models with latent variables and why our
quasi-equivalence guarantees compare each estimated graph directly with the true
graph.

\subsection{Candidate graph class and structural complexity}
\label{sec:candidate-class}

For fixed $p$ and $L_{\max}<\infty$, let $\Gclass$ be the class of admissible
causal graphs with the same $p$ labeled observed vertices and at most
$L_{\max}$ exogenous latent vertices. These graphs have no self-loops or
edges into latent vertices, may contain directed cycles among observed
variables, and have at least two edges from each latent vertex to observed
vertices. With a fixed convention for labeling latent vertices, only
finitely many such graph structures are possible\footnote{Finiteness is used in
the statistical analysis; it does not require enumerating these graphs
during optimization.}. Relabeling latent vertices changes neither the
observed covariance family nor the structural complexity.

We measure structural complexity by
\begin{equation}
    c(G^+)=|E(G^+)|+\ell(G^+),
    \label{eq:complexity}
\end{equation}
where $|E(G^+)|$ counts edges among observed variables and edges from latent
to observed variables, and $\feedbackedit{\ell(G^+):=|L(G^+)|}$ counts latent variables\feedbackedit{;}
\feedbackedit{we write $\ell$ when the graph is fixed.} Each edge
and each latent variable contributes one unit to $c(G^+)$. An edge is
counted whenever it belongs to the graph, even if its fitted coefficient
is zero. \feedbackedit{Its additive form gives the closed-form expected penalty in
Section~\ref{sec:bernoulli}; a dimension-based penalty would require
separate covariance-family dimension calculations.
Appendix~\ref{app:complexity-dimension} explains this distinction with
a two-variable example and discusses the dimension-based alternative.}

\subsection{Observed-data objective and graph score}
\label{sec:observed-score}

Let $X_O^{(1)},\ldots,X_O^{(n)}$ be i.i.d.\ observations from
$\mathcal N(0,\Sigma^*)$, with $\Sigma^*\succ0$. The mean is known to be
zero, so the sample covariance is
$S_n=\frac1n\sum_{k=1}^nX_O^{(k)}X_O^{(k)\top}.$ We consider $n\geq p$, so that $S_n\succ0$ almost surely \feedbackedit{(with
probability one)}. For a candidate
covariance $\Sigma\succ0$, the observed-data Gaussian negative
log-likelihood is
\begin{equation}
    \mathcal L_n^{\mathrm{full}}(\Sigma)
    =\frac n2\bigl\{p\log(2\pi)+\log\det\Sigma
    +\tr(\Sigma^{-1}S_n)\bigr\}.
\label{eq:full-observed-nll}
\end{equation}
The term $np\log(2\pi)/2$ comes from the Gaussian density normalization
and is independent of the graph ($p$ is fixed) and its parameters. The superscript
``full'' indicates that this constant is included.
Appendix~\ref{app:observed-score-derivation} derives this likelihood and
its expression in structural parameters.

For $G^+\in\Gclass$ and $\theta\in\mathcal P(G^+)$, write
$\mathcal L_n^{\mathrm{full}}(\theta)
:=\mathcal L_n^{\mathrm{full}}(\Sigma_O(\theta))$. Our joint objective is
\begin{equation}
    J_n(G^+,\theta)
    =\mathcal L_n^{\mathrm{full}}(\theta)+\lambda_n c(G^+),
    \qquad \lambda_n=\frac12\log n.
    \label{eq:joint-objective}
\end{equation}
We optimize $J_n$ over both $G^+\in\Gclass$ and
$\theta\in\mathcal P(G^+)$. The penalty balances fit against the number
of edges and latent variables. Its scaling follows the Bayesian information
criterion (BIC), but $c(G^+)$ need not equal the covariance-family dimension;
we therefore do not claim a standard BIC interpretation.
{\color{black}In contrast, \citet{amendola2020} use a dimension-based BIC term;
Appendix~\ref{app:amendola-comparison} gives the comparison.}

For theoretical analysis, define the \emph{penalized graph score} by taking
the infimum of the joint objective over the parameters of a fixed graph:
\begin{equation}
    s_n(G^+):=\inf_{\theta\in\mathcal P(G^+)}J_n(G^+,\theta).
    \label{eq:graph-score-definition}
\end{equation}
Minimizing this score over graphs gives the infimum of the joint objective.
This theoretical definition does not require enumerating graphs during
computation; Section~\ref{sec:bernoulli} introduces Bernoulli gates for
joint structure and parameter optimization.

To express this same score in covariance space, define the Gaussian
Kullback--Leibler (KL) divergence for $S,\Sigma\in\PD^p$ by
\begin{equation}
    D(S\Vert\Sigma)=\frac12\{\tr(\Sigma^{-1}S)
    -\log\det(\Sigma^{-1}S)-p\}.
    \label{eq:gaussian-kl}
\end{equation}
This is the KL divergence from $\mathcal N(0,S)$ to
$\mathcal N(0,\Sigma)$; it is nonnegative and equals zero exactly when
$S=\Sigma$. Let
$\delta_n(G^+)=\inf_{\Sigma\in\Mbar(G^+)}D(S_n\Vert\Sigma)$.
Thus $\delta_n(G^+)$ measures the best fit to $S_n$ within the graph's
covariance-family closure. Taking the infimum over this closure gives the
same value as taking it over admissible parameters, even when no parameter
value attains the infimum; Appendix~\ref{app:observed-score-derivation}
proves this equality.

Let $C_n:=\mathcal L_n^{\mathrm{full}}(S_n)$, the negative log-likelihood
evaluated at the sample covariance. Subtracting $C_n$ from
Equation~\ref{eq:full-observed-nll} gives $nD(S_n\Vert\Sigma)$. Thus the
same graph score can be written as
\begin{equation}
    s_n(G^+)=C_n+n\delta_n(G^+)+\lambda_n c(G^+),
    \label{eq:graph-score}
\end{equation}
The constant $C_n$ depends only on the observed data and has no effect on
graph selection. This equivalent expression is used in the consistency
analysis below.

\section{Method and Theory}
\label{sec:method}

We first establish statistical guarantees for global minimizers of the
graph score induced by Equation~\ref{eq:joint-objective}. We then formulate
optimization of the same joint objective using independent Bernoulli gates.

\subsection{Closure consistency and geometric identification}
\label{sec:consistency-result}

Let $G^{*+}\in\Gclass$ be the data-generating graph and choose a
minimum-complexity distribution-equivalent representative
\begin{equation}
    G^\dagger\in\operatorname*{arg\,min}_{G^+\in\Gclass}
    \{c(G^+):G^+\equiv_O G^{*+}\}.
    \label{eq:minimum-representative}
\end{equation}
The minimum exists because $\Gclass$ is finite and contains $G^{*+}$.
\begin{feedbackrevision}
Write $c^\dagger:=c(G^\dagger)$. The reference $G^\dagger$ is not an input
to the estimator. All such representatives have the same family
$\M(G^{*+})$, closure $V(G^{*+})$, and complexity $c^\dagger$, so the
conditions below do not depend on the choice.

\end{feedbackrevision}
Define the set of candidate graphs whose covariance families contain or
can approximate the true covariance arbitrarily closely and whose complexity
is at most $c^\dagger$:
\begin{equation}
    \mathcal G_0(\Sigma^*)
    :=\left\{G^+\in\Gclass:
    \Sigma^*\in\Mbar(G^+),\ c(G^+)\leq c^\dagger\right\}.
    \label{eq:relevant-candidate-set}
\end{equation}

\begin{theorem}[Closure-selection consistency]
\label{thm:closure-consistency}
Suppose $p$ and $L_{\max}$ are fixed and the data are i.i.d. from
$\mathcal N(0,\Sigma^*)$, where $\Sigma^*\in\M(G^{*+})$.  Then
\begin{equation}
    \Pr\!\left\{
    \operatorname*{arg\,min}_{G^+\in\Gclass}s_n(G^+)
    \subseteq\mathcal G_0(\Sigma^*)
    \right\}\longrightarrow1.
    \label{eq:closure-selection-conclusion}
\end{equation}
Equivalently, with probability tending to one, every global minimizer
$\widehat G_n^+$ satisfies
\[
    \Sigma^*\in\Mbar(\widehat G_n^+),
    \qquad c(\widehat G_n^+)\leq c^\dagger.
\]
The same conclusion holds for any deterministic penalty sequence satisfying
$\lambda_n\to\infty$ and $\lambda_n/n\to0$.
\end{theorem}

\begin{feedbackrevision}
The proof is in Appendix~\ref{app:proofs}. This result concerns global
score minimizers, not arbitrary numerical outputs, and requires none of
the identification assumptions below. Recovery up to marginal
quasi-equivalence needs an additional condition.

\end{feedbackrevision}
\begin{assumption}[Identification at the true covariance]
\label{ass:pointwise-identification}
Assume that every $G^+\in\mathcal G_0(\Sigma^*)$ satisfies
$G^+\cong_O G^\dagger$.
\end{assumption}

\begin{corollary}[Consistency up to marginal quasi-equivalence]
\label{cor:quasi-consistency}
Under the conditions of Theorem~\ref{thm:closure-consistency} and
Assumption~\ref{ass:pointwise-identification},
\begin{equation}
    \Pr\!\left\{
    \operatorname*{arg\,min}_{G^+\in\Gclass}s_n(G^+)
    \subseteq\{G^+\in\Gclass:G^+\cong_O G^{*+}\}
    \right\}\longrightarrow1.
    \label{eq:quasi-consistency-conclusion}
\end{equation}
\end{corollary}

The proof is given in Appendix~\ref{app:proof-quasi-consistency}.
Assumption~\ref{ass:pointwise-identification} is sufficient for this
conclusion; necessity is not claimed. The following three assumptions are
sufficient to establish it. They concern inclusions between Zariski
closures, structural complexity, and intersections of covariance families.

\begin{assumption}[Relevant-candidate algebraic faithfulness]
\label{ass:algebraic-faithfulness}
For every $G^+\in\Gclass$ satisfying $c(G^+)\leq c^\dagger$,
\[
    \Sigma^*\in\Mbar(G^+)
    \quad\Longrightarrow\quad
    V(G^\dagger)\subseteq V(G^+).
\]
\end{assumption}

\begin{assumption}[Structural minimality]
\label{ass:structural-minimality}
For every $G^+\in\Gclass$, if
\[
    V(G^\dagger)\subseteq V(G^+)
    \quad\text{and}\quad
    \dim\M(G^\dagger)<\dim\M(G^+),
\]
then $c(G^+)>c^\dagger$.
\end{assumption}

\begin{assumption}[Full-dimensional overlap]
\label{ass:model-overlap}
If $\Sigma^*\in\Mbar(G^+)$, $c(G^+)\leq c^\dagger$, and
$V(G^+)=V(G^\dagger)=V$, then
\begin{equation}
    \dim\bigl(\M(G^+)\cap\M(G^\dagger)\bigr)=\dim V.
    \label{eq:overlap-assumption}
\end{equation}
\end{assumption}

These assumptions connect closure-selection consistency to recovery up to
marginal quasi-equivalence. Lemma~\ref{lem:geometric-identification} gives the proof.
Appendix~\ref{app:assumption-scope} examines the scope of these
assumptions. It shows that algebraic faithfulness and full-dimensional
overlap hold for almost every admissible choice of the true model parameters
(genericity). Structural minimality requires separate
justification; the appendix gives sufficient conditions and a
counterexample showing that it can fail.

\subsection{Bernoulli formulation for joint structure optimization}
\label{sec:bernoulli}

To optimize the joint objective in Equation~\ref{eq:joint-objective} over
edges among observed variables, edges from latent to observed variables,
and the number of latent variables, we reserve $L_{\max}$ slots for
potential latent variables and use mutually independent binary gates:
\begin{equation}
\begin{gathered}
    M_{ij}\sim\operatorname{Bernoulli}(q_{ij})\quad(i\neq j),
    \qquad M_{ii}=q_{ii}=0,\\
    H_{hj}\sim\operatorname{Bernoulli}(r_{hj}),
    \qquad Z_h\sim\operatorname{Bernoulli}(\pi_h).
\end{gathered}
\label{eq:bernoulli-masks}
\end{equation}
Here $M_{ij}$ selects the edge $i\to j$ among observed variables, $Z_h$
activates latent slot $h$, and $H_{hj}$ selects its edge to observed
variable $j$ when that slot is active. The entries of
$Q=(q_{ij})$, $R=(r_{hj})$, and $\pi=(\pi_h)$ lie in $[0,1]$ and are the
probabilities that the corresponding gates equal one. We optimize these
probabilities; the sampled gates remain zero or one.

Let $\vartheta=(W,\Gamma,\Omega_O)$, where
$W\in\mathbb R^{p\times p}$,
$\Gamma\in\mathbb R^{L_{\max}\times p}$, and $\Omega_O\succ0$ is diagonal.
The matrix $W$ contains coefficients for edges among observed variables,
and $\Gamma$ contains latent loadings with latent noise standard deviations
absorbed into them. The same continuous parameters are used for all
sampled configurations. Define
\begin{equation}
\begin{split}
    B_{OO}(M)&=M\od W,\qquad A(M)=I_p-M\od W,\\
    \Psi(H,Z)&=\Omega_O+(H\od\Gamma)^\top
    \diag(Z)(H\od\Gamma).
\end{split}
\label{eq:masked-parameters}
\end{equation}
For each configuration, let $L_n(M,H,Z;\vartheta)$ be the Gaussian
negative log-likelihood in Equation~\ref{eq:full-observed-nll}, evaluated
at $A(M)^{-\top}\Psi(H,Z)A(M)^{-1}$ and with $np\log(2\pi)/2$ removed.
Set $L_n=+\infty$ if $A(M)$ is singular. Let $c(M,H,Z)$ count selected
edges among observed variables, active latent variables, and selected edges
from active latents to observed variables. Appendix~\ref{app:bernoulli-details} gives
their explicit formulas. The exact Bernoulli objective is
\begin{equation}
\begin{split}
    \overline F_n(Q,R,\pi,\vartheta)
    ={}&\mathbb{E}_{M,H,Z}[L_n(M,H,Z;\vartheta)] +\lambda_n\!\left[\sum_{i\neq j}q_{ij}
    +\sum_{h=1}^{L_{\max}}\pi_h
    \left(1+\sum_{j=1}^p r_{hj}\right)\right].
\end{split}
\label{eq:exact-bernoulli-objective}
\end{equation}
The penalty is the exact expected complexity because
$\mathbb E[Z_hH_{hj}]=\pi_hr_{hj}$. Only configurations with positive
probability contribute to the expectation; a singular $A(M)$ in any such
configuration makes the objective infinite. \feedbackedit{If all $q_{ij}\in(0,1)$
for $i\neq j$, every binary observed-edge mask has positive probability,
so $A(M)$ must be invertible for every such mask. For fixed continuous
parameters making every binary $A(M)$ invertible, the objective is
differentiable in the probabilities but need not be jointly convex.
Appendix~\ref{app:bernoulli-details}
gives further details on this.}
Active latent slots with fewer than two children violate the definition
of $\Gclass$. For each finite-likelihood configuration, the appendix
shows how to remove them without changing the covariance or increasing
the score.

\begin{proposition}[Global exactness of the Bernoulli formulation]
\label{prop:bernoulli-exactness}
Let
$\feedbackedit{F_n^{\mathrm{bin}}}(M,H,Z;\vartheta)=L_n(M,H,Z;\vartheta)
+\lambda_nc(M,H,Z)$.  Then
\begin{equation}
    \inf_{Q,R,\pi,\vartheta}
    \mathbb{E}[\feedbackedit{F_n^{\mathrm{bin}}}(M,H,Z;\vartheta)]
    =\min_{M,H,Z}\inf_{\vartheta}
    \feedbackedit{F_n^{\mathrm{bin}}}(M,H,Z;\vartheta).
    \label{eq:global-exactness}
\end{equation}
The common value equals $\min_{G^+\in\Gclass}s_n(G^+)-np\log(2\pi)/2$.
After removing inactive slots and active slots with fewer than two children,
both formulations have the same minimizing graphs. If the left infimum is
attained, a global minimizer with binary $Q$, $R$, and $\pi$ exists.
\end{proposition}

{\color{black}
\begin{remark} (Significance of global exactness).
The Bernoulli formulation
supports gradient-based joint optimization of structure probabilities and
continuous parameters.
Proposition~\ref{prop:bernoulli-exactness} ensures that introducing these
continuous probabilities preserves the global infimum of the corresponding
binary objective. This exactness concerns
the Bernoulli expectation. Appendix~\ref{app:bernoulli-details} proves the result.
\end{remark}}

\section{Experiments}
\label{sec:experiments}

We evaluate the proposed method on controlled synthetic linear Gaussian SCMs and on simulated GeneNetWeaver (GNW) benchmarks \citep{schaffter2011}. 
We compare against Ghassami et al.~\citep{ghassami2020}, DCCD-CONF \citep{sethuraman2025}, and Am\'endola et al.~\citep{amendola2020}. All methods evaluated at a shared setting receive the same generated data and trial seeds.

\subsection{Setup and evaluation}
\label{sec:experiment-setup}

Synthetic experiments use $n=10{,}000$ observational samples. Observed graphs are generated following the procedure used by \citet{ghassami2020}; for confounded settings, we augment them with randomly generated exogenous latent common causes (Appendix \ref{app:experimental-details}). Our gradient-based optimization algorithm is explained in Appendix \ref{sec:concrete}.

We evaluate performance by comparing the covariance families of the
reference graph $G^{*+}$ and the recovered graph $\widehat G^+$.
For the forward discrepancy $d_{\mathrm{fwd}}$, we generate covariance
matrices $\Sigma_k^*$ using different admissible parameter values of
the reference graph. We then fit the recovered graph's parameters to
each covariance matrix while keeping its structure fixed.
For the reverse discrepancy $d_{\mathrm{rev}}$, we generate covariance
matrices from the recovered graph and fit the reference graph's
parameters to each matrix. Each discrepancy is the largest fitting
error across the evaluated covariance matrices. We report
$d_{\mathrm{fwd}}+d_{\mathrm{rev}}$ as the primary bidirectional
compatibility error; lower values indicate better agreement at these
covariances. Appendix~\ref{app:compatibility-evaluation} provides the exact procedure. SHD measures agreement with one particular graph representation, whereas our compatibility error allows structurally different graphs that are distributionally equivalent. An equivalence-aware version of SHD requires characterizing the relevant equivalent graphs or searching over them, which is computationally impractical here.

\subsection{Synthetic Results}
\label{sec:experiment-size}

\begin{figure}[t]
    \centering
    \begin{subfigure}[b]{0.32\linewidth}
        \centering
        \includegraphics[width=\linewidth]{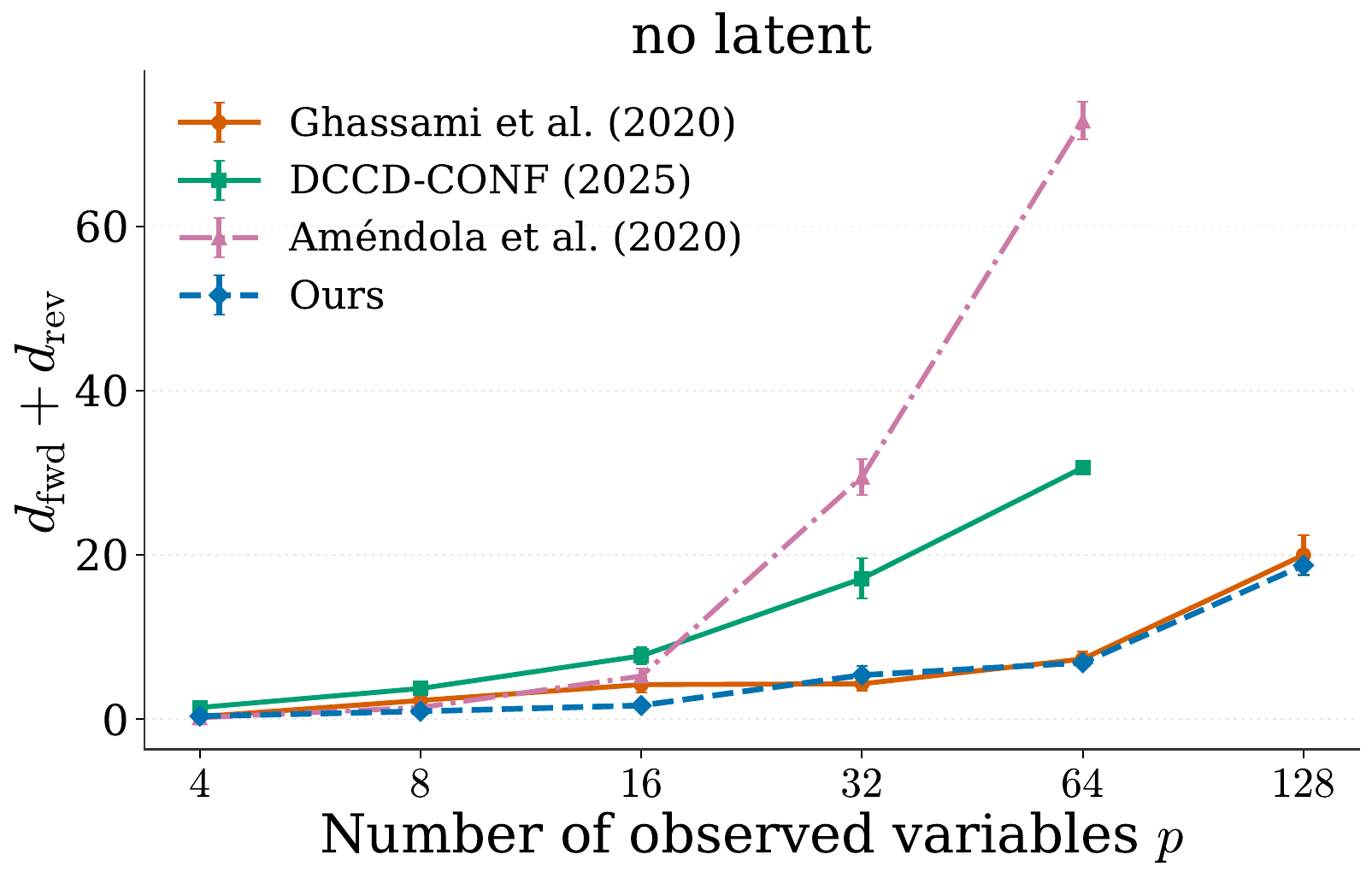}
        \caption{No latents, degree 4.}
    \end{subfigure}
    \hfill
    \begin{subfigure}[b]{0.32\linewidth}
        \centering
        \includegraphics[width=\linewidth]{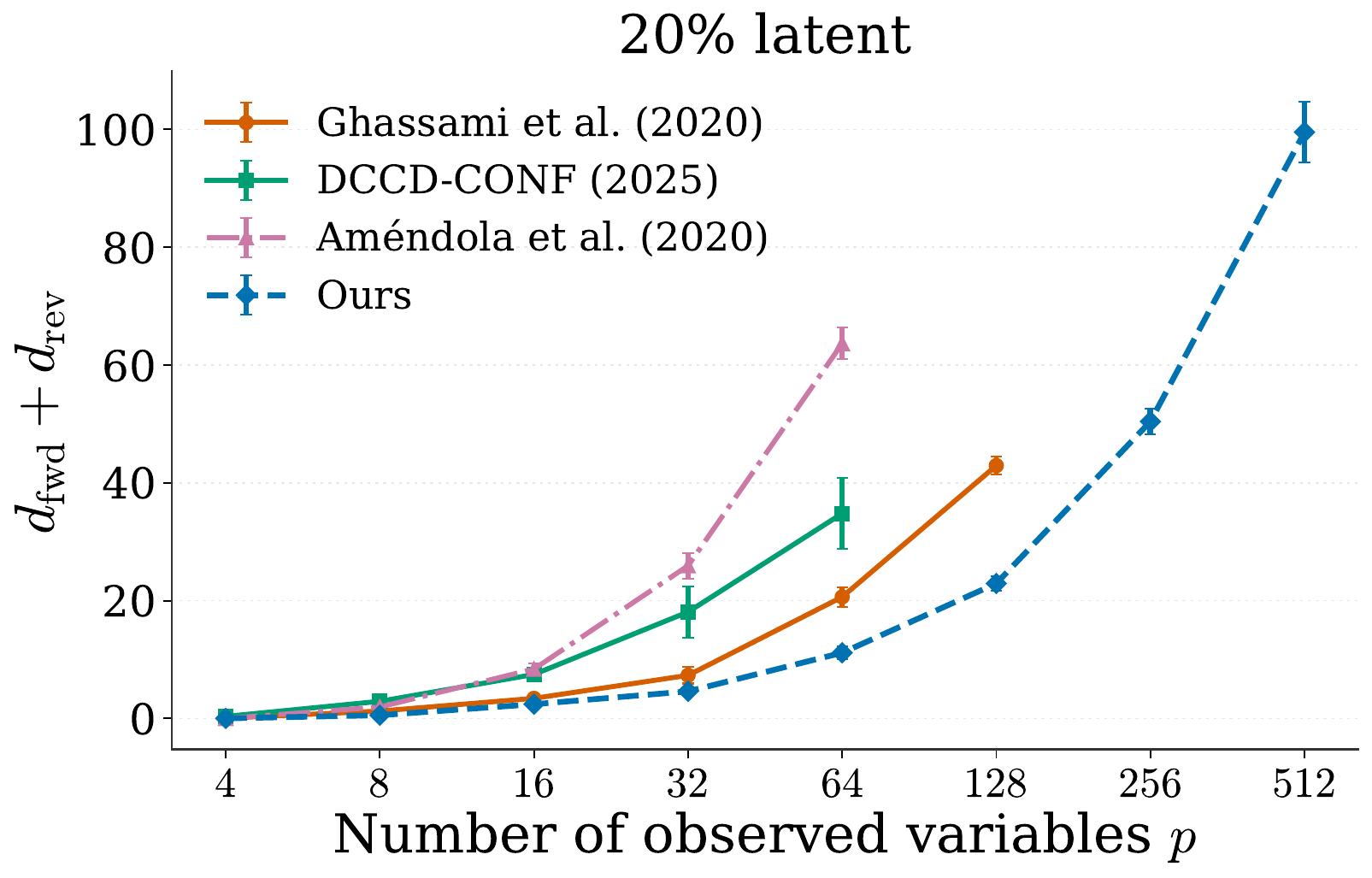}
        \caption{$20\%$ latent, degree 4.}
    \end{subfigure}
    \hfill
    \begin{subfigure}[b]{0.32\linewidth}
        \centering
        \includegraphics[width=\linewidth]{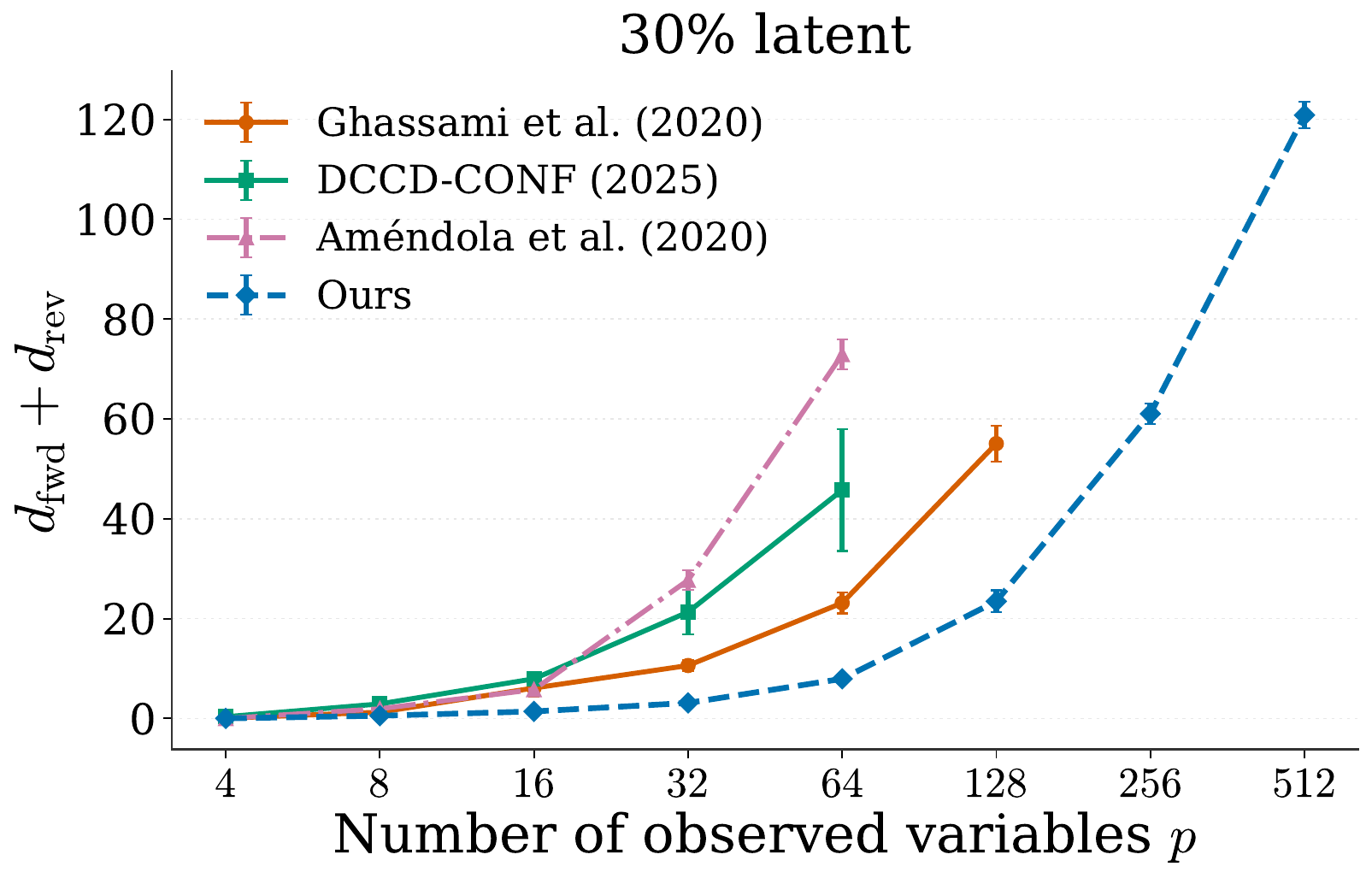}
        \caption{$30\%$ latent, degree 4.}
    \end{subfigure}

    \vspace{1mm}

    \begin{subfigure}[b]{0.32\linewidth}
        \centering
        \includegraphics[width=\linewidth]{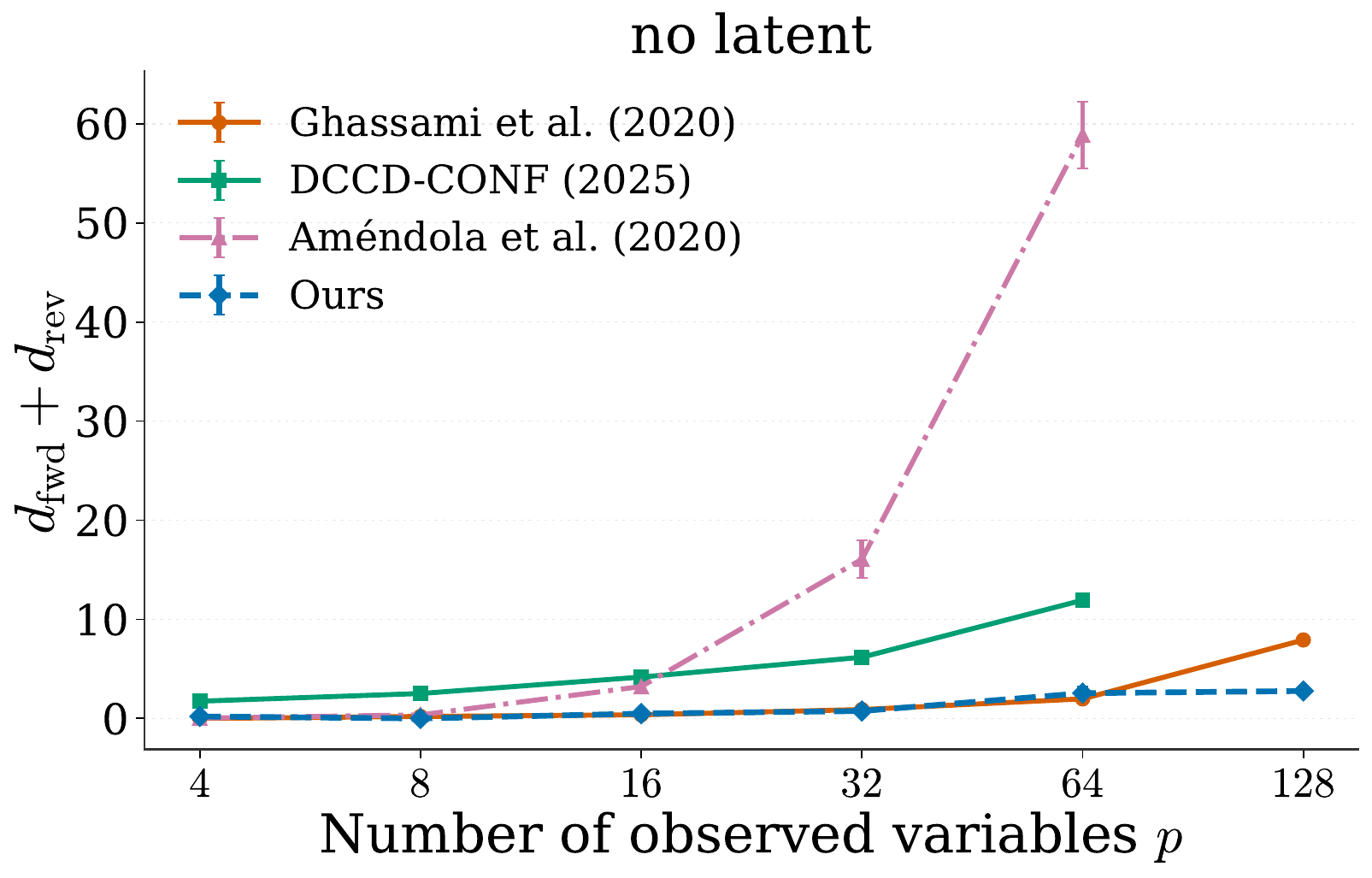}
        \caption{No latents, degree 2.}
    \end{subfigure}
    \hfill
    \begin{subfigure}[b]{0.32\linewidth}
        \centering
        \includegraphics[width=\linewidth]{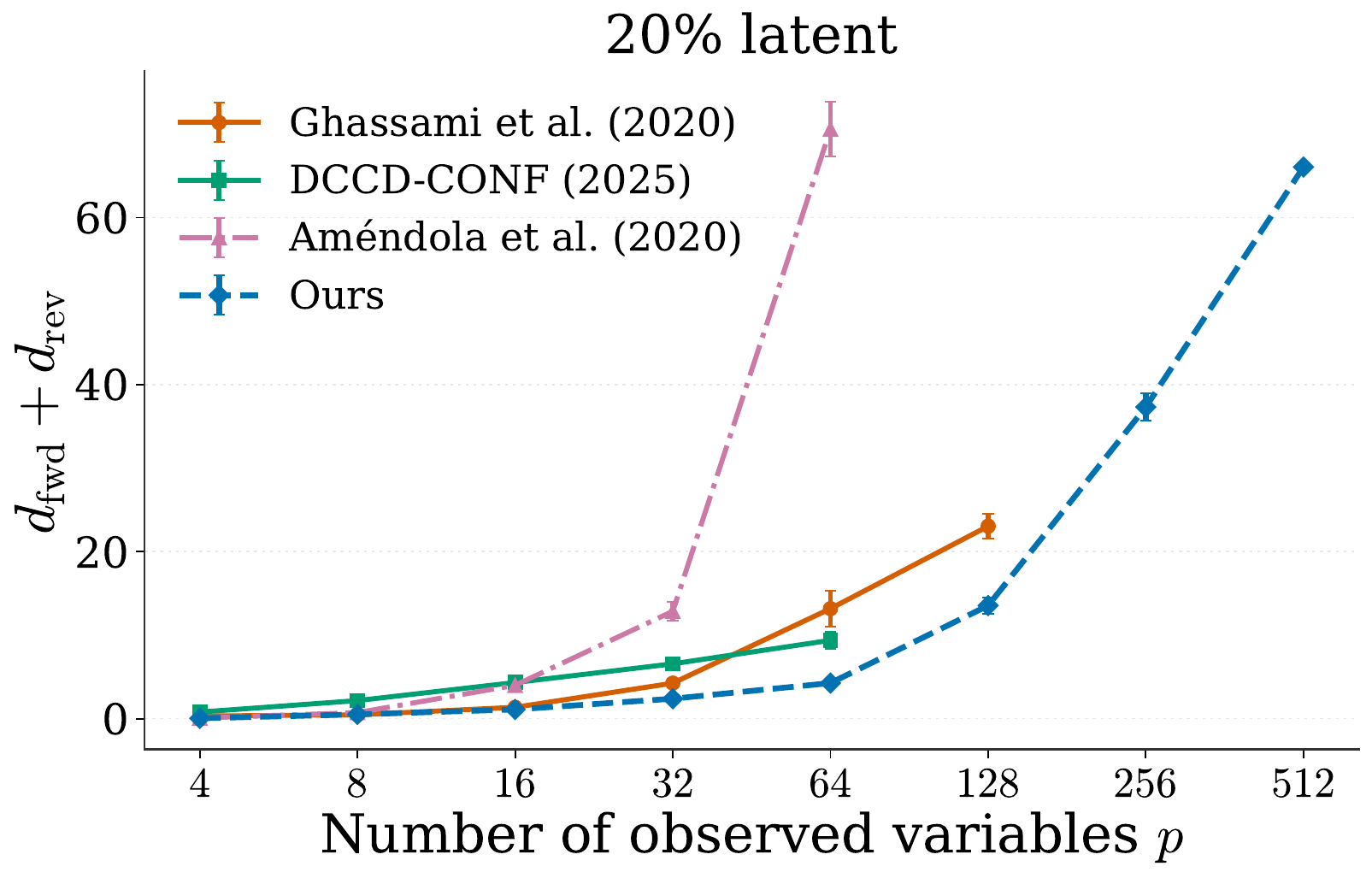}
        \caption{$20\%$ latent, degree 2.}
    \end{subfigure}
    \hfill
    \begin{subfigure}[b]{0.32\linewidth}
        \centering
        \includegraphics[width=\linewidth]{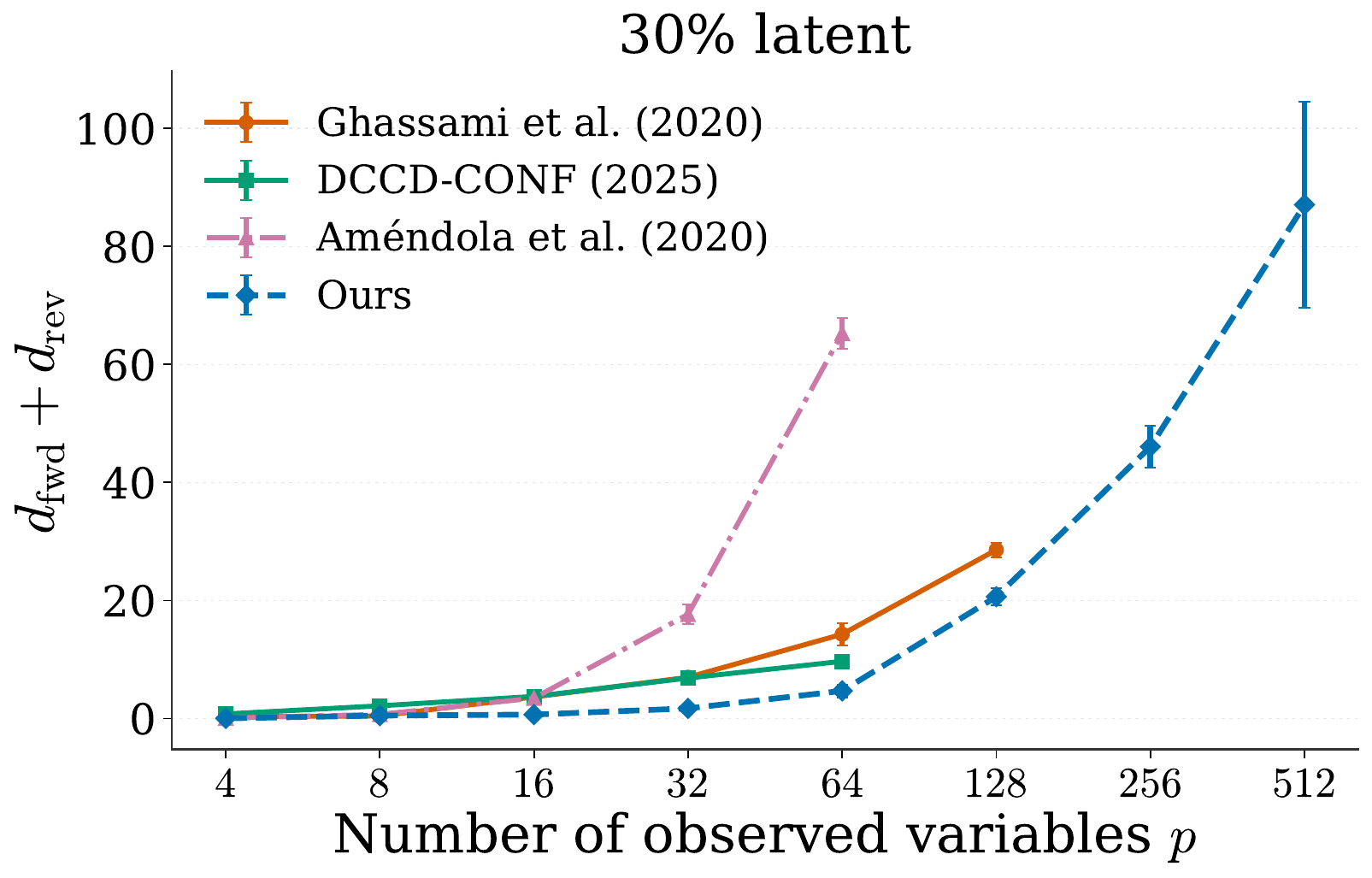}
        \caption{$30\%$ latent, degree 2.}
    \end{subfigure}

    \caption{Synthetic results across 10 trials. Compatibility error for varying number of observed variables, the maximum observed degree, and the latent-to-observed ratio. }
    \label{fig:scaling-size}
\end{figure}

\begin{wrapfigure}{r}{0.40\columnwidth}
    \centering
    \vspace{-3\baselineskip}
    \includegraphics[width=\linewidth]{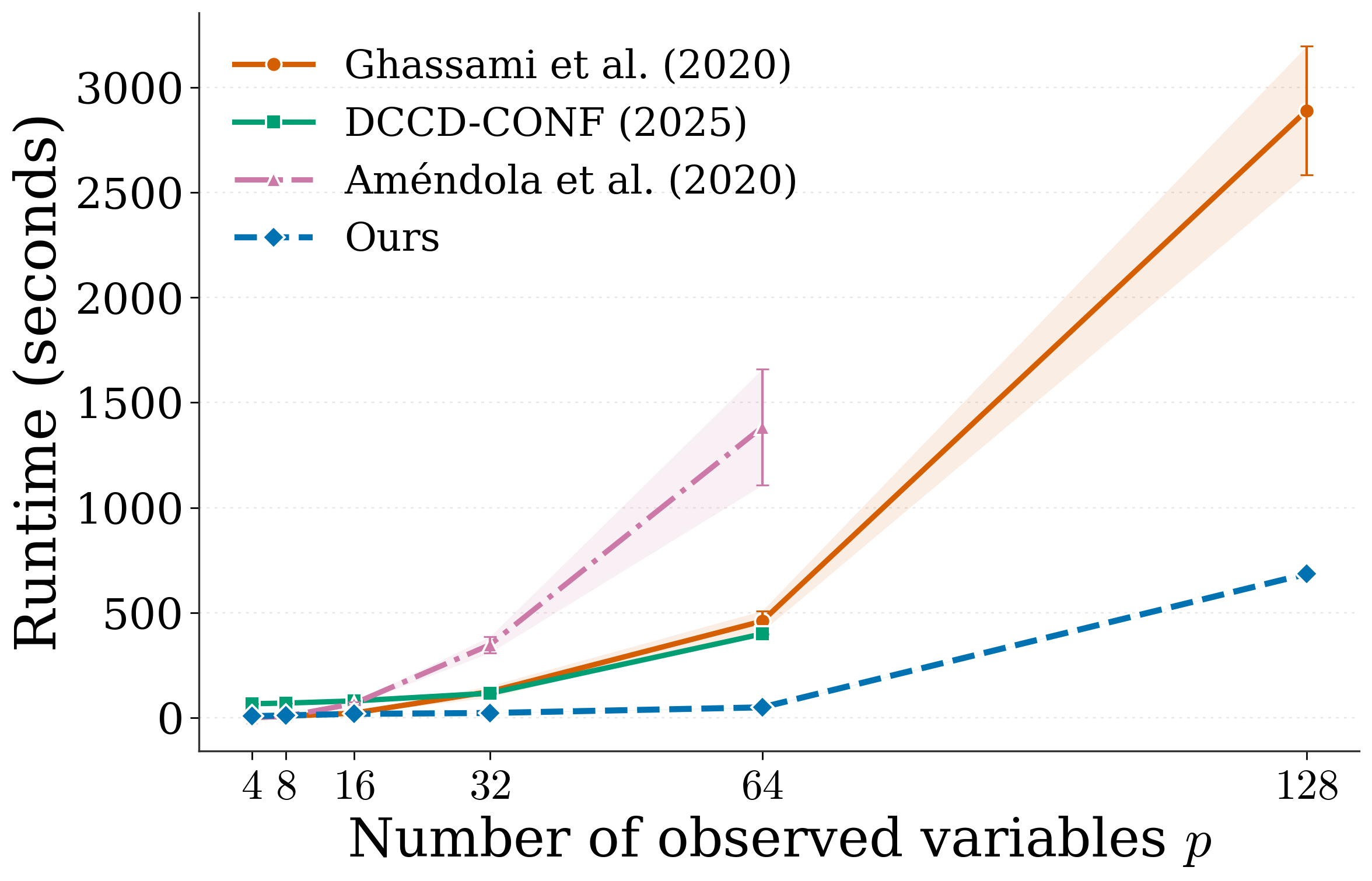}
    \caption{Wall-clock runtime.}
    \label{fig:runtime-scaling}
    \vspace{-1\baselineskip}
\end{wrapfigure}

Figure~\ref{fig:scaling-size} presents synthetic results across different
numbers of observed variables, maximum observed degrees, and ratios of
latent to observed variables.
Our method achieves lower mean compatibility error in several of the larger latent-variable settings where baseline results are available. Each baseline is evaluated up to the largest graph size for which its computation remains feasible. Our method is additionally evaluated at $p=256$ and $p=512$ to demonstrate scalability to substantially larger graphs. Appendix~\ref{sec:experiment-stress} provides more experiments over all degree and latent-to-observed ratio combinations. 
Figure~\ref{fig:runtime-scaling} compares wall-clock runtime.  For the baselines,
we could not report the wall-clock beyond $p=128$,
because they did not complete within our computational budget.
Our method additionally averaged $4544$ seconds at $p=256$ and $24770$ seconds at $p=512$, demonstrating scalability to substantially larger graphs.

\subsection{GeneNetWeaver benchmark}
\label{sec:experiment-gnw}

We additionally evaluate fixed GNW networks with 100 observed variables and either 20 or 30 explicitly constructed exogenous latent variables. GNW generates observations using gene-regulatory dynamics rather than our assumed linear Gaussian model, providing a complementary test under model misspecification while retaining a known cyclic graph structure. 
Table~\ref{tab:gnw-results} reports the compatibility error results. Across both latent settings, our method achieves a substantially lower total compatibility error, approximately half that of the closest baseline. It is also markedly faster than all baselines in both latent settings. See Appendix~\ref{app:gnw-details} for more details.

\begin{table}[t]
    \centering
    \caption{GeneNetWeaver results with 100 observed variables. Values are mean $\pm$ SEM across 10 trials.}
    \label{tab:gnw-results}
    \setlength{\tabcolsep}{2.8pt}
    \renewcommand{\arraystretch}{1.12}
    \begin{tabular}{@{}clcccc@{}}
        \toprule
        Latents & Method
        & $d_{\mathrm{fwd}}\downarrow$
        & $d_{\mathrm{rev}}\downarrow$
        & $(d_{\mathrm{fwd}}+d_{\mathrm{rev}})\downarrow$
        & Time (s)$\downarrow$ \\
        \midrule
        \multirow{4}{*}{20} & \citet{ghassami2020}
        & $16.57\pm0.23$
        & $22.78\pm2.87$
        & $39.35\pm2.71$
        & $615.40\pm45.48$ \\
        & DCCD-CONF
        & $28.01\pm1.67$
        & $10.03\pm0.43$
        & $38.05\pm1.91$
        & $1334.86\pm0.55$ \\
        & \citet{amendola2020}
        & $22.97\pm0.24$
        & $93.54\pm7.43$
        & $116.51\pm7.65$
        & $1464.31\pm166.35$ \\
        & Ours
        & $\mathbf{16.33\pm0.25}$
        & $\mathbf{3.31\pm0.13}$
        & $\mathbf{19.63\pm0.24}$
        & $\mathbf{47.92\pm0.07}$ \\
        \midrule
        \multirow{4}{*}{30} & \citet{ghassami2020}
        & $\mathbf{16.28\pm0.59}$
        & $24.75\pm1.59$
        & $41.03\pm1.80$
        & ${743.09\pm58.81}$ \\
        & DCCD-CONF
        & $36.61\pm4.93$
        & $11.12\pm0.93$
        & $47.73\pm5.08$
        & $1352.79\pm8.69$ \\
        & \citet{amendola2020}
        & $25.46\pm1.59$
        & $98.94\pm2.90$
        & $124.40\pm3.57$
        & $1822.18\pm308.10$ \\
        & Ours
        & $18.27\pm0.74$
        & $\mathbf{3.00\pm0.27}$
        & $\mathbf{21.27\pm0.73}$
        & $\mathbf{122.54\pm1.87}$ \\
        \bottomrule
    \end{tabular}
\end{table}

\begin{feedbackrevision}
\section{Conclusion}
\label{sec:conclusion}

We developed a framework for learning cyclic linear Gaussian models with exogenous latent confounders and established consistency up to marginal quasi-equivalence under explicit identification conditions.
Our Bernoulli formulation for cyclic Gaussian models with explicit latent
variables preserves the corresponding binary objective's global infimum.
Our recovery guarantee assumes identification conditions and global score
minimization; it does not imply unique graph or latent-count recovery.
Broader structural-minimality conditions and numerical optimization
guarantees remain open problems.

\end{feedbackrevision}
\clearpage

\begingroup
\let\feedbackOriginalBibitem\bibitem
\renewcommand{\bibitem}[2][]{\feedbackOriginalBibitem[#1]{#2}\color{black}\def\feedbackKey{#2}\def\feedbackHyttinen{hyttinen2012}\def\feedbackForre{forre2018}\def\feedbackSpot{ma2024spot}\def\feedbackBoyd{boyd2004}\ifx\feedbackKey\feedbackHyttinen\color{black}\fi
  \ifx\feedbackKey\feedbackForre\color{black}\fi
  \ifx\feedbackKey\feedbackSpot\color{black}\fi
  \ifx\feedbackKey\feedbackBoyd\color{black}\fi
}
\bibliography{iclr2027_conference}
\bibliographystyle{iclr2027_conference}
\endgroup

\newpage
\appendix

\section{Proofs}
\label{app:proofs}

\subsection{Proof of Lemma~\ref{lem:marginal-covariance}}
\label{app:proof-marginal-covariance}

\begin{proof}
Substituting the block matrices in Equation~\ref{eq:block-sem} into the
structural equations gives the two block equations
\[
    X_O=B_{OO}^\top X_O+\Lambda^\top X_L+N_O,
    \qquad
    X_L=N_L.
\]
The second identity follows because latent variables have no parents.  Using
$X_L=N_L$ in the observed equation and collecting the terms involving $X_O$
on the left gives
\[
    (I_p-B_{OO}^\top)X_O=\Lambda^\top N_L+N_O.
\]
By assumption, $I_p-B_{OO}$ is invertible.  Therefore,
\begin{equation}
    X_O=(I_p-B_{OO})^{-\top}(\Lambda^\top N_L+N_O).
    \label{eq:proof-observed-solution}
\end{equation}
The vector on the right is Gaussian and has mean zero.  Moreover, the block
diagonality of $\Omega$ implies that $N_O$ and $N_L$ are independent, so the
two cross-covariance terms vanish.  Consequently,
\begin{align*}
    \operatorname{Cov}(\Lambda^\top N_L+N_O)
    &={\Lambda^\top}\operatorname{Cov}(N_L)\Lambda
      +\operatorname{Cov}(N_O)\\
    &=\Lambda^\top\Omega_L\Lambda+\Omega_O
      =\Psi.
\end{align*}
Taking the covariance of both sides of
Equation~\ref{eq:proof-observed-solution} now yields
\[
    \operatorname{Cov}(X_O)
    =(I_p-B_{OO})^{-\top}\Psi(I_p-B_{OO})^{-1},
\]
which is Equation~\ref{eq:marginal-covariance}.

It remains to verify positive definiteness. Let $\Omega_L^{1/2}$ be the
diagonal matrix obtained by taking the positive square root of each latent
noise variance, and set $C=\Omega_L^{1/2}\Lambda\in\mathbb R^{\ell\times p}$.
The latent covariance contribution has the factorization
\[
    \Lambda^\top\Omega_L\Lambda
    =(\Omega_L^{1/2}\Lambda)^\top(\Omega_L^{1/2}\Lambda)
    =C^\top C.
\]
For every $v\in\mathbb R^p$,
$v^\top C^\top Cv=(Cv)^\top(Cv)\geq0$, so this contribution is positive
semidefinite. Since $C$ has $\ell$ rows and $p$ columns,
\[
    \operatorname{rank}(\Lambda^\top\Omega_L\Lambda)
    =\operatorname{rank}(C^\top C)
    \leq\operatorname{rank}(C)\leq\min(\ell,p).
\]
Because $\Omega_O\succ0$, every nonzero $v\in\mathbb R^p$ satisfies
\[
    v^\top\Psi v=v^\top\Omega_Ov+(Cv)^\top(Cv)>0.
\]
Thus $\Psi\succ0$, regardless of the rank of $\Lambda$. When $\ell=0$,
the latent contribution and its rank are zero, and $\Psi=\Omega_O\succ0$.
Finally, write $A=I_p-B_{OO}$. Its invertibility implies
$A^{-1}v\neq0$ whenever $v\neq0$, and hence
\[
    v^\top\Sigma_Ov=(A^{-1}v)^\top\Psi(A^{-1}v)>0.
\]
Therefore $\Sigma_O\succ0$, completing the proof.
\end{proof}

\subsection{Derivation of the observed-data score}
\label{app:observed-score-derivation}

We first derive the full Gaussian likelihood in
Equation~\ref{eq:full-observed-nll}, then express it in structural
parameters and establish the KL representation of the graph score in
Equation~\ref{eq:graph-score}.
Let $\mathbf X_O\in\mathbb R^{n\times p}$ have $k$th row
$X_O^{(k)\top}$. Throughout, the mean is known to be zero, so
$\mathbf X_O^\top\mathbf X_O=nS_n$ without sample-mean subtraction. Under
the nonsingular Gaussian sampling model, $\mathbf X_O$ has full column rank
almost surely when $n\geq p$, and hence $S_n\succ0$.

\paragraph{Gaussian likelihood.}
For a candidate covariance $\Sigma\succ0$, the density of one observation
$x\in\mathbb R^p$ is
\[
    f_\Sigma(x)=(2\pi)^{-p/2}(\det\Sigma)^{-1/2}
    \exp\!\left(-\frac12x^\top\Sigma^{-1}x\right).
\]
Independence of the observations makes their joint density a product.
Taking its negative logarithm gives
\begin{align*}
    \mathcal L_n^{\mathrm{full}}(\Sigma)
    &=-\sum_{k=1}^n\log f_\Sigma(X_O^{(k)})\\
    &=\frac{np}{2}\log(2\pi)+\frac n2\log\det\Sigma
      +\frac12\sum_{k=1}^nX_O^{(k)\top}\Sigma^{-1}X_O^{(k)}.
\end{align*}
Each observation contributes $p\log(2\pi)/2$ through the Gaussian
normalization factor, giving $np\log(2\pi)/2$ in total. This term depends
only on $n$ and the observed dimension $p$, not on the graph, the latent
count, or any continuous parameter. Using
$x^\top\Sigma^{-1}x=\tr(\Sigma^{-1}xx^\top)$ gives
\[
    \mathcal L_n^{\mathrm{full}}(\Sigma)
    =\frac{np}{2}\log(2\pi)+\frac n2\log\det\Sigma
     +\frac12\tr(\Sigma^{-1}\mathbf X_O^\top\mathbf X_O).
\]
Substituting $\mathbf X_O^\top\mathbf X_O=nS_n$ yields
Equation~\ref{eq:full-observed-nll}. As in the main text, we also write
$\mathcal L_n^{\mathrm{full}}(\theta)
=\mathcal L_n^{\mathrm{full}}(\Sigma_O(\theta))$.

\paragraph{Likelihood in structural parameters.}
For a fixed graph, write $A=I_p-B_{OO}$ and
$\Psi=\Omega_O+\Lambda^\top\Omega_L\Lambda$. By
Equation~\ref{eq:marginal-covariance},
\[
    \Sigma_O(\theta)=A^{-\top}\Psi A^{-1},\qquad
    \Sigma_O(\theta)^{-1}=A\Psi^{-1}A^\top,
\]
and
\[
    \log\det\Sigma_O(\theta)=\log\det\Psi-2\log|\det A|.
\]
Cyclic invariance of the trace gives
\begin{align*}
    \tr\!\left[\Sigma_O(\theta)^{-1}
        \mathbf X_O^\top\mathbf X_O\right]
    &=\tr\!\left[\Psi^{-1}A^\top
        \mathbf X_O^\top\mathbf X_O A\right]\\
    &=\tr\!\left[\Psi^{-1}
        (\mathbf X_OA)^\top(\mathbf X_OA)\right].
\end{align*}
Removing the Gaussian constant does not change the minimizing graphs or
parameters. Define the constant-free negative log-likelihood by
\begin{equation}
\begin{split}
    \mathcal L_n(\theta)
    &:=\mathcal L_n^{\mathrm{full}}(\theta)-\frac{np}{2}\log(2\pi)\\
    &=-n\log|\det A|+\frac n2\log\det\Psi
      +\frac12\tr\!\left[\Psi^{-1}
        (\mathbf X_OA)^\top(\mathbf X_OA)\right].
\end{split}
\label{eq:observed-nll}
\end{equation}
The second equality follows by substituting the preceding identities into
the full likelihood. Equivalently, the last term is
$\frac n2\tr(\Psi^{-1}A^\top S_nA)$, the form used for the masked
likelihood in Equation~\ref{eq:masked-likelihood}.
For one observation $x$, the vector
$r=A^\top x=x-B_{OO}^\top x$ is its \emph{structural residual}: what
remains after subtracting the observed-parent contributions in the
structural equations. Under the candidate SCM,
$A^\top X_O=\Lambda^\top N_L+N_O$ has covariance $\Psi$.
The rows of $\mathbf X_OA$ are therefore the transposes of these residuals.
The trace term sums their Gaussian quadratic terms
$r^\top\Psi^{-1}r/2$, and $\frac n2\log\det\Psi$ comes from the
Gaussian normalization. The linear change of variables $x\mapsto A^\top x$
multiplies volume by $|\det A|$. Consequently, converting the residual
density to the observation density contributes $-\log|\det A|$ to each
negative log-likelihood, giving $-n\log|\det A|$ in total.
These identities require only invertibility of $A$, not acyclicity.
The latent covariance contribution remains
$\Lambda^\top\Omega_L\Lambda$, with the loading restrictions specified
by the graph; $\Psi$ is not treated as an unrestricted covariance matrix.

\paragraph{KL representation.}
Expanding Equation~\ref{eq:gaussian-kl} gives
\[
    nD(S_n\Vert\Sigma)
    =\frac n2\{\tr(\Sigma^{-1}S_n)
      +\log\det\Sigma-\log\det S_n-p\}.
\]
Evaluating the full likelihood at $\Sigma=S_n$ gives the constant
\[
    C_n=\mathcal L_n^{\mathrm{full}}(S_n)
    =\frac n2\{p\log(2\pi)+\log\det S_n+p\}.
\]
Consequently,
\[
    \mathcal L_n^{\mathrm{full}}(\Sigma)=C_n+nD(S_n\Vert\Sigma).
\]
Here $C_n$ includes the Gaussian normalization term and the value of the
remaining likelihood terms at $S_n$. It is independent of the graph and
its parameters, so optimizing the likelihood is equivalent to minimizing
$D(S_n\Vert\Sigma)$, the KL divergence between the two zero-mean Gaussian
distributions with covariances $S_n$ and $\Sigma$.

\paragraph{Equality of the likelihood infima.}
For every admissible parameter tuple,
\[
    \mathcal L_n(\theta)
    =K_n+nD\bigl(S_n\Vert\Sigma_O(\theta)\bigr),
    \qquad
    K_n=\frac n2\{\log\det S_n+p\}.
\]
By definition, $\M(G^+)$ is exactly the image of $\mathcal P(G^+)$ under
the covariance map. Thus taking the infimum over admissible parameters
equals taking the infimum over $\M(G^+)$. To pass to its relative closure,
first note that set inclusion gives
\[
    \inf_{\Sigma\in\Mbar(G^+)}D(S_n\Vert\Sigma)
    \leq\inf_{\Sigma\in\M(G^+)}D(S_n\Vert\Sigma).
\]
Conversely, for any $\overline\Sigma\in\Mbar(G^+)$, there is a sequence
$\Sigma_k\in\M(G^+)$ converging to $\overline\Sigma\succ0$. Continuity of
$D(S_n\Vert\cdot)$ on $\PD^p$ implies
\[
    \inf_{\Sigma\in\M(G^+)}D(S_n\Vert\Sigma)
    \leq \lim_{k\to\infty}D(S_n\Vert\Sigma_k)
    =D(S_n\Vert\overline\Sigma).
\]
Taking the infimum over $\overline\Sigma\in\Mbar(G^+)$ proves the reverse
inequality. Consequently,
\begin{equation}
    \inf_{\theta\in\mathcal P(G^+)}\mathcal L_n(\theta)
    =K_n+n\inf_{\Sigma\in\Mbar(G^+)}D(S_n\Vert\Sigma)
    =K_n+n\delta_n(G^+).
    \label{eq:parameter-covariance-profile}
\end{equation}
Adding $\lambda_n c(G^+)$ and the omitted Gaussian constant gives
$\inf_{\theta\in\mathcal P(G^+)}J_n(G^+,\theta)=s_n(G^+)$ as defined
in Equation~\ref{eq:graph-score-definition}, and proves its equivalent
expression in Equation~\ref{eq:graph-score}. No attainment of the infimum
by admissible SCM parameters is required.

Every graph admits the covariance $I_p$, obtained by setting all edge
coefficients to zero and all noise variances to one. Hence
$0\leq\delta_n(G^+)\leq D(S_n\Vert I_p)<\infty$. Since $\Gclass$ is
finite, the graph score has at least one global minimizer even if a
particular graph's parameter infimum is not attained.

The choice $\lambda_n=\tfrac12\log n$ gives BIC-like scaling. The
consistency proof uses only $\lambda_n\to\infty$ and $\lambda_n/n\to0$:
the penalty eventually distinguishes unnecessarily complex candidates,
while a fixed positive population KL gap contributes an order-$n$
likelihood disadvantage that dominates any fixed complexity advantage.
This argument does not identify $c(G^+)$ with the covariance-family
dimension or require the regularity conditions of a standard BIC derivation.

\subsection{Algebraic and semialgebraic sets and their dimensions}
\label{app:geometry-definitions}

This subsection gives the formal definitions underlying
Section~\ref{sec:notation}. We work with subsets of $\mathbb R^m$ and
polynomials with real coefficients; $\mathbb R[x_1,\ldots,x_m]$ denotes
the set of such polynomials. Symmetric matrices are identified with vectors
of their upper-triangular entries. We first specify the topological terms
used in the proofs. The algebraic and dimension definitions that follow
include references to their sources. We temporarily use
$\dim_{\mathrm{sa}}$ and $\dim_{\mathrm{alg}}$ to distinguish the two
notions of dimension.

\paragraph{Open sets, interior, and neighborhoods.}
\label{app:topology-notation}
For $a\in\mathbb R^m$ and $\varepsilon>0$, the \emph{open ball} of
radius $\varepsilon$ centered at $a$ is
$B_\varepsilon(a)=\{x\in\mathbb R^m:\|x-a\|_2<\varepsilon\}$,
where $\|x\|_2=(\sum_i x_i^2)^{1/2}$ is the Euclidean norm of a vector.
A set $U\subseteq\mathbb R^m$ is \emph{Euclidean open} if, for every
$x\in U$, there exists $\varepsilon>0$ such that
$B_\varepsilon(x)\subseteq U$. Thus openness requires a ball around
every point of $U$, not merely one ball somewhere in $U$.
For $S\subseteq V\subseteq\mathbb R^m$, the \emph{interior of $S$
relative to $V$} is
\[
    \operatorname{int}_V S
    :=\{x\in S:\text{there exists }\varepsilon>0
       \text{ such that }B_\varepsilon(x)\cap V\subseteq S\}.
\]
A set $S\subseteq V$ is \emph{open relative to $V$} if
$\operatorname{int}_V S=S$. When $V=\mathbb R^m$, these definitions give
the ordinary Euclidean interior and openness. A \emph{neighborhood} of
$x\in V$ relative to $V$ is a subset of $V$ containing
$B_\varepsilon(x)\cap V$ for some $\varepsilon>0$; the neighborhood
itself need not be open. A set $S$ is \emph{closed relative to $V$} if
$\cl_V S=S$, using the closure defined in Section~\ref{sec:notation}.
For symmetric matrices, we use Frobenius-norm balls; they define the same
open sets as Euclidean balls in the upper-triangular coordinates.

For example, $[0,1]$ is not open in $\mathbb R$, but its interior is
$(0,1)$, which contains an open ball. Also, the line segment
$S=\{(t,0):-1<t<1\}$ has empty interior in $\mathbb R^2$ but satisfies
$\operatorname{int}_V S=S$ for $V=\{(t,0):t\in\mathbb R\}$.
Throughout the proofs, the specified space $V$ determines which interior
or neighborhood is meant.

\begin{definition}[Real algebraic sets and irreducibility]
\label{def:real-algebraic-set}
We use the real-set definitions given by \citet[Section~2]{laba2007}.
For finitely many polynomials $f_1,\ldots,f_r\in\mathbb R[x_1,\ldots,x_m]$,
define their common zero set by
\[
    Z(f_1,\ldots,f_r)
    :=\{x\in\mathbb R^m:f_i(x)=0\text{ for every }i=1,\ldots,r\}.
\]
A set $V\subseteq\mathbb R^m$ is \emph{real algebraic} if
$V=Z(f_1,\ldots,f_r)$ for some such finite collection. A nonempty real
algebraic set $V$ is \emph{irreducible} if, for any real algebraic subsets
$V_1,V_2\subseteq V$,
\[
    V=V_1\cup V_2\quad\Longrightarrow\quad V=V_1\text{ or }V=V_2.
\]
Thus an irreducible set cannot be decomposed into two strictly smaller
algebraic sets. A polynomial $f$ \emph{vanishes on} $V$ if
$f(x)=0$ for every $x\in V$.
\end{definition}

\begin{definition}[Semialgebraic sets]
\label{def:semialgebraic-set}
We use the finite-union formulation in
\citet[Section~2.1.1, Proposition~2.1]{coste2000}.
A set $S\subseteq\mathbb R^m$ is \emph{semialgebraic} if it can be written
as a finite union
\[
    S=\bigcup_{a=1}^{N}
    \{x\in\mathbb R^m:
    f_{ai}(x)=0\ (1\leq i\leq r_a),\quad
    g_{aj}(x)>0\ (1\leq j\leq s_a)\},
\]
where all $f_{ai}$ and $g_{aj}$ are real polynomials. The integers
$N,r_a,s_a$ are nonnegative: an empty union denotes the empty set, and an
empty list of conditions imposes no restriction. Weak inequalities are
also permitted through the identity
$\{g\geq0\}=\{g>0\}\cup\{g=0\}$.
\end{definition}

Every algebraic set is semialgebraic. The converse need not hold, because
inequalities can restrict a semialgebraic set to only part of an algebraic
set.
A map $f:S\to\mathbb R^k$ is called \emph{semialgebraic} if its graph
$\{(x,f(x)):x\in S\}$ is a semialgebraic subset of
$\mathbb R^{m+k}$, where $S\subseteq\mathbb R^m$
\citep[Chapter~2]{coste2000}.

\begin{definition}[Semialgebraic dimension]
\label{def:semialgebraic-dimension}
We use the coordinate-projection characterization of semialgebraic
dimension; see \citet[Section~3.3]{coste2000} for the standard dimension
theory.
For $I=\{i_1<\cdots<i_k\}\subseteq\{1,\ldots,m\}$, let
$\pi_I:\mathbb R^m\to\mathbb R^k$ be the coordinate projection
$\pi_I(x)=(x_{i_1},\ldots,x_{i_k})$. For nonempty semialgebraic
$S\subseteq\mathbb R^m$, define
\[
    \dim_{\mathrm{sa}}S
    :=\max\bigl\{|I|:\pi_I(S)\text{ contains a nonempty open ball
    in }\mathbb R^{|I|}\bigr\}.
\]
For $I=\varnothing$, the projection of a nonempty set is the singleton
$\mathbb R^0$, so dimension zero is included. We use the convention
$\dim_{\mathrm{sa}}\varnothing=-1$.
\end{definition}

Thus dimension is the largest number of selected coordinates whose
projection contains an open ball; it need not equal the total number of
coordinates used to describe the set.

\begin{definition}[Algebraic dimension]
\label{def:algebraic-dimension}
For a nonempty real algebraic set $V\subseteq\mathbb R^m$, its algebraic
dimension can be defined by chains of nested algebraic subsets
\citep[Section~2]{laba2007}:
\[
\begin{split}
    \dim_{\mathrm{alg}}V:=\max\bigl\{d\in\mathbb Z_{\geq0}:
    &\;Z_0\subsetneq Z_1\subsetneq\cdots\subsetneq Z_d\subseteq V,\\
    &\;Z_0,\ldots,Z_d\text{ are nonempty irreducible real algebraic sets}
    \bigr\}.
\end{split}
\]
The integer $d$ counts strict inclusions, not sets. We set
$\dim_{\mathrm{alg}}\varnothing=-1$. If $V$ is irreducible, a chain of
maximum length ends at $V$; for a reducible set, the final set $Z_d$ need
only be contained in $V$.
\end{definition}

\paragraph{Relation between the dimensions.}
For any $S\subseteq\mathbb R^m$, its real Zariski closure is the common
zero set of all polynomials that are zero throughout $S$
\citep[Section~3.3]{coste2000}:
\[
    \overline S^{\,\mathrm{Zar}}
    :=\{x\in\mathbb R^m:f(x)=0\text{ for every real polynomial }f
    \text{ satisfying }f(s)=0\text{ for all }s\in S\}.
\]
Equivalently, it is the smallest real algebraic set containing $S$.
For every nonempty semialgebraic set $S$, the standard dimension theorem
gives \citep[Theorem~3.20]{coste2000}
\[
    \dim_{\mathrm{sa}}S
    =\dim_{\mathrm{alg}}\overline S^{\,\mathrm{Zar}}.
\]
In particular, if $S$ is already algebraic, its Zariski closure equals
$S$ and the two dimensions agree. This justifies the common notation
$\dim$ used throughout the paper. For semialgebraic $T\subseteq S$,
\emph{full-dimensional in $S$} means $\dim T=\dim S$; it does not mean
$\dim T=m$.

For example, the open line segment
$S=\{(x,0)\in\mathbb R^2:0<x<1\}$ is semialgebraic and has dimension one:
its first-coordinate projection is $(0,1)$, but $S$ contains no open ball
in $\mathbb R^2$. Its Zariski closure is the entire line
$V=\{(x,0):x\in\mathbb R\}$, which also has dimension one. Indeed, a
polynomial that is zero on the segment restricts to a univariate polynomial
that is zero on $(0,1)$, and must therefore be zero on the entire line.
Thus $S$ is full-dimensional in $V$ even though $S\neq V$.

\subsection{Closure properties of covariance families}
\label{app:closure-properties}

We explain the containment and closure properties used in the model
geometry proof. The first two properties follow from the definition of
the real Zariski closure and continuity of polynomials; they do not require
semialgebraicity or irreducibility.

\paragraph{1. The covariance family is contained in its Zariski closure.}
By definition, $V(G^+)$ consists of the symmetric matrices satisfying every
polynomial equality that holds throughout $\M(G^+)$. Each matrix in
$\M(G^+)$ satisfies all these equalities, so
$\M(G^+)\subseteq V(G^+)$. This inclusion is a consequence of the
definition, not an additional model assumption.

\paragraph{2. Positive-definite limits are also in the Zariski closure.}
Let $\overline\Sigma\in\Mbar(G^+)$. By the definition of relative
Euclidean closure, there is a sequence $\Sigma_k\in\M(G^+)$ such that
$\Sigma_k\to\overline\Sigma\succ0$. If a polynomial $f$ in the
upper-triangular matrix entries equals zero throughout $\M(G^+)$, then
$f(\Sigma_k)=0$ for every $k$. Since polynomials are continuous,
\[
    f(\overline\Sigma)=\lim_{k\to\infty}f(\Sigma_k)=0.
\]
Thus $\overline\Sigma$ satisfies every polynomial equality defining
$V(G^+)$ and belongs to that set. Also, every element of $\M(G^+)$ belongs
to its relative closure, as witnessed by the constant sequence at that
element. Together these facts give
\[
    \M(G^+)\subseteq\Mbar(G^+)\subseteq V(G^+)\cap\PD^p.
\]
The same continuity argument shows that any real algebraic set contains
the limits of its convergent sequences, and is therefore Euclidean closed.
Containment in $V(G^+)$ alone would not justify the inclusion of limits;
its Euclidean closedness is the additional fact used here.

\paragraph{3. The relative closure is semialgebraic.}
Step~2 of the proof of Proposition~\ref{prop:model-geometry} below
establishes that $\M(G^+)$ is semialgebraic. Its Euclidean closure in
$\Sym^p$ is therefore semialgebraic by
\citet[Corollary~2.5]{coste2000}, a consequence of the Tarski--Seidenberg
theorem.

By Sylvester's criterion, positive definiteness can be expressed by
finitely many strict polynomial inequalities in the matrix entries
\citep[Section~7.2, Theorem~7.2.5]{horn2013}. Hence $\PD^p$ is
semialgebraic.

Finally, finite intersections of semialgebraic sets are semialgebraic
\citep[Section~2.1.2]{coste2000}. Since relative closure within $\PD^p$
keeps exactly the positive-definite matrices in the closure within
$\Sym^p$,
\[
    \Mbar(G^+)
    =\bigl(\cl_{\Sym^p}\M(G^+)\bigr)\cap\PD^p
\]
is semialgebraic.

\subsection{Proof of Proposition~\ref{prop:model-geometry}}

\begin{proof}
Fix $G^+$ throughout the proof. We first describe the parameter space and
its map to covariance matrices, and then establish the three claims.

\paragraph{1. Free parameters and the covariance map.}
Impose the graph-required zeros on $B_{OO}$ and $\Lambda$. List the
remaining allowed coefficients and the diagonal entries of $\Omega_O$ and
$\Omega_L$ in a vector $\theta\in\mathbb R^d$. Here, $d$ is the number of
free parameter coordinates, not necessarily the dimension of the covariance
family. The parameter space $\mathbb R^d$ is distinct from the covariance
space $\Sym^p$, which we identify with $\mathbb R^{p(p+1)/2}$ through the
upper-triangular matrix entries. Each admissible parameter vector $\theta$
determines one covariance matrix; different admissible parameter vectors
may determine the same one.

Write
\[
    A(\theta)=I_p-B_{OO},\qquad
    \Delta(\theta)=\det A(\theta),\qquad
    \Psi(\theta)=\Omega_O+\Lambda^\top\Omega_L\Lambda.
\]
In these free coordinates, $\mathcal P(G^+)$ is defined by
\[
    (\Omega_O)_{ii}>0,\qquad (\Omega_L)_{hh}>0,
    \qquad \Delta(\theta)^2>0.
\]
Each variance is a coordinate of $\theta$, and $\Delta(\theta)$ is a
polynomial because the determinant is polynomial in the matrix entries.
These finitely many polynomial inequalities therefore show that
$\mathcal P(G^+)$ is semialgebraic. All inequalities are strict, and their
left-hand sides are continuous. Hence, for every admissible $\theta$,
some $\varepsilon>0$ satisfies
$B_\varepsilon(\theta)\subseteq\mathcal P(G^+)$: the parameter set is
Euclidean open in the free-coordinate space $\mathbb R^d$.
It contains the point $\theta_0$ with $B_{OO}=0$, $\Lambda=0$, and unit
noise variances, and therefore contains an open ball centered at
$\theta_0$. Graph-required zero entries are fixed throughout; only the
free coordinates vary. This uses the convention that an allowed edge
coefficient may equal zero.

The covariance map is
\[
    \phi_{G^+}:\mathcal P(G^+)\longrightarrow\PD^p,
    \qquad
    \phi_{G^+}(\theta)=A(\theta)^{-\top}\Psi(\theta)A(\theta)^{-1}.
\]
Its values are positive definite by Lemma~\ref{lem:marginal-covariance},
and its image is exactly $\M(G^+)$.

\paragraph{2. The covariance family is semialgebraic.}
Consider the joint set of admissible parameters and their generated
covariance matrices:
\[
    \mathcal T=
    \bigl\{(\theta,\Sigma)\in\mathbb R^d\times\Sym^p:
    \theta\in\mathcal P(G^+),\;
    A(\theta)^\top\Sigma A(\theta)=\Psi(\theta)\bigr\}.
\]
The matrix equality represents finitely many scalar polynomial equalities
in the entries of $\theta$ and $\Sigma$. Together with the polynomial
inequalities defining $\mathcal P(G^+)$, these show that $\mathcal T$ is
semialgebraic. Since $A(\theta)$ is invertible on $\mathcal P(G^+)$, the
matrix equality is equivalent to $\Sigma=\phi_{G^+}(\theta)$. Therefore,
\[
    \M(G^+)
    =\{\Sigma\in\Sym^p:\text{there exists }\theta
      \text{ with }(\theta,\Sigma)\in\mathcal T\}.
\]
This is the coordinate projection of $\mathcal T$ onto the covariance
space. The Tarski--Seidenberg theorem states that a coordinate projection
of a semialgebraic set is semialgebraic
\citep[Theorem~2.3 and Corollary~2.4]{coste2000}. It therefore
proves that $\M(G^+)$ is semialgebraic. The theorem is needed because the
polynomial conditions defining $\mathcal T$ involve both parameters and
covariances, whereas $\M(G^+)$ is a set of covariance matrices alone. No
explicit elimination of the parameters is required.

\paragraph{3. The Zariski closure is irreducible.}
Write $V=V(G^+)$. Irreducibility means that $V$ cannot be written as the
union of two proper algebraic subsets. Suppose, to obtain a contradiction,
that $V=V_1\cup V_2$, where $V_1\neq V$ and $V_2\neq V$ are algebraic.
Choose a matrix $\Sigma_1\in V\setminus V_1$. Because $V_1$ is the common
zero set of a collection of polynomials in the symmetric matrix entries,
at least one defining polynomial, denoted by $f$, satisfies
$f(\Sigma_1)\neq0$. Thus $f$ is zero throughout $V_1$ but not throughout
$V$. Similarly, choose a polynomial $g$ that is zero throughout $V_2$ but
not throughout $V$. These are polynomials in covariance-space coordinates,
not in the structural parameters. Saying that a polynomial \emph{vanishes}
on a set means that it equals zero at every element of that set.

Every generated covariance belongs to $V_1$ or $V_2$. Consequently,
\[
    f\bigl(\phi_{G^+}(\theta)\bigr)
    g\bigl(\phi_{G^+}(\theta)\bigr)=0
    \qquad\text{for every }\theta\in\mathcal P(G^+).
\]
To turn this into a polynomial identity in the free parameters, use the
adjugate formula. Each entry of the adjugate
$J(\theta)=\operatorname{adj}A(\theta)$ is, up to sign, the determinant
of a matrix obtained by deleting one row and one
column of $A(\theta)$. Determinants are polynomials in the matrix entries,
so every entry of $J(\theta)$ is a polynomial in $\theta$. Since
$\Delta(\theta)=\det A(\theta)\neq0$ on the admissible parameter set,
\[
    A(\theta)^{-1}=\frac{J(\theta)}{\Delta(\theta)},
    \qquad
    \phi_{G^+}(\theta)
    =\frac{J(\theta)^\top\Psi(\theta)J(\theta)}{\Delta(\theta)^2}.
\]
Both $J(\theta)$ and $\Psi(\theta)$ have polynomial entries, so the
numerator entries are polynomials in $\theta$. Each covariance entry is
therefore a ratio of polynomials, called a \emph{rational function}.
The polynomials $f$ and $g$ are still polynomials in covariance entries;
substituting these ratios makes their values rational functions of
$\theta$. Putting each expression over a common denominator gives
\[
    f\bigl(\phi_{G^+}(\theta)\bigr)
       =\frac{P_f(\theta)}{\Delta(\theta)^{2r}},
    \qquad
    g\bigl(\phi_{G^+}(\theta)\bigr)
       =\frac{P_g(\theta)}{\Delta(\theta)^{2s}},
\]
where $P_f,P_g$ are polynomials in the free parameters and
$r=\deg f$, $s=\deg g$. Since $\Delta(\theta)\neq0$ on the admissible
parameter set, clearing denominators yields
\[
    P_f(\theta)P_g(\theta)=0
    \qquad\text{for every }\theta\in\mathcal P(G^+).
\]

A real polynomial that equals zero on an open ball in $\mathbb R^d$
is identically zero, meaning that all its coefficients are zero.
To see why, the ball contains a product of nonempty open intervals.
Fixing all coordinates but one gives a
polynomial in one variable that is zero throughout an interval. A nonzero
polynomial in one variable has only finitely many roots, so all its
coefficients must be zero. Applying this argument successively to the
remaining coordinates shows that every coefficient of the original
polynomial is zero. Step~1 therefore implies $P_fP_g\equiv0$ on
$\mathbb R^d$. A product of two nonzero real polynomials cannot be the
zero polynomial. Thus $P_f\equiv0$ or $P_g\equiv0$.
It follows that one fixed polynomial, either $f$ or $g$, is zero on every
generated covariance. By the definition of Zariski closure, any polynomial
that is zero on $\M(G^+)$ is also zero on $V(G^+)$. This contradicts the
choice of $f$ and $g$. Hence $V(G^+)$ is irreducible. The argument does
not require a covariance matrix to determine a unique parameter vector.

\paragraph{4. The dimension identities.}
Step~2 and Appendix~\ref{app:closure-properties} show that both
$\M(G^+)$ and $\Mbar(G^+)$ are semialgebraic, and that
$\M(G^+)\subseteq\Mbar(G^+)\subseteq V(G^+)$. For a semialgebraic set,
\citet[Theorem~3.20]{coste2000} identifies its semialgebraic dimension with
the algebraic dimension of its real Zariski closure. Thus
$\dim\M(G^+)=\dim V(G^+)$. The same theorem identifies the two dimensions
of the algebraic set $V(G^+)$. Dimension is monotone under inclusion:
any coordinate projection of a smaller set is contained in the same
projection of the larger set, so an open ball in the former also belongs
to the latter. Consequently,
\[
    \dim\M(G^+)\leq\dim\Mbar(G^+)\leq\dim V(G^+)
    =\dim\M(G^+).
\]
The dimensions must therefore all be equal, proving
Equation~\ref{eq:model-dimensions}.
\end{proof}

\subsection{Rationale and scope of marginal quasi-equivalence}
\label{app:quasi-equivalence-rationale}

\paragraph{Relation to the definition of Ghassami et al.}
\citet[Section~5, Definition~9]{ghassami2020} study linear Gaussian models
without latent confounders. They define quasi-equivalence through a
positive-measure set of distributions that both graphs can generate,
with measure expressed in parameters describing those distributions.
Their phrase ``linearly independent parameters'' refers to free
distributional coordinates, not statistically independent random variables.
Their definition is therefore phrased in distribution coordinates, rather
than as an intersection of the graphs' full structural parameter spaces.
The motivation is that two graphs may share
a substantial set of distributions without generating exactly the same
distribution family. This positive-overlap motivation is also the basis
of our definition.

\paragraph{Why we formulate the definition in observed covariance space.}
In our setting, candidates may have different numbers of latent variables
and different structural parameter spaces. Moreover, different structural
parameter values may generate the same observed covariance. To compare
graphs, we therefore use the sets of observed covariances themselves,
$\M(G_1^+)$ and $\M(G_2^+)$, in the common space $\PD^p$. Their dimensions
count the independent covariance coordinates, rather than the number of
structural parameters. Equation~\ref{eq:quasi-equivalence} makes the
positive-overlap idea precise without choosing separate distribution
parameterizations for the marginal models. This is a formulation of that
idea for our setting, not an assumption that the characterizations or
recovery results for models without latent variables automatically extend
to models with latent variables. A \emph{precision matrix} is the inverse
of a covariance matrix. Using covariance rather than precision matrices
does not change the dimension criterion: on $\PD^p$, matrix inversion is
one-to-one and is its own inverse. Its entries are rational functions on
this semialgebraic domain, so it preserves semialgebraic dimension
\citep[Theorem~3.18]{coste2000}. It also sends the intersection of two
covariance families to the intersection of their corresponding precision
families. Thus both the family dimensions and the intersection dimension
are unchanged.

\paragraph{Dimension and positive-measure overlap.}
Let $S=\M(G_1^+)\cap\M(G_2^+)$, and suppose both families have dimension
$d$. Proposition~\ref{prop:model-geometry} ensures that the families,
and hence their intersection, are semialgebraic. By
Definition~\ref{def:semialgebraic-dimension}, $\dim S=d$ means that some
projection of $S$ onto $d$ selected covariance coordinates contains a
nonempty open set in $\mathbb R^d$. Such a set has positive
$d$-dimensional Lebesgue measure, meaning positive ordinary volume in
those coordinates. Conversely, a semialgebraic subset of $\mathbb R^d$
with positive Lebesgue measure contains a nonempty open set
\citep[Section~3.3]{coste2000}. Thus the criterion requires the shared
covariances to have a coordinate projection containing an open ball in
$\mathbb R^d$, rather than only a subset of dimension below $d$.
We require the intersection
to have the dimension of \emph{both} families: a lower-dimensional family
contained in a larger one does not qualify. The criterion does not require
equality of the complete families; the parts that they do not share may
also have dimension $d$.

\paragraph{Why comparisons must be made directly.}
For arbitrary semialgebraic sets, full-dimensional overlap need not pass
through an intermediate set. For example, consider the one-dimensional
intervals
\[
    A=(0,2),\qquad B=(1,4),\qquad C=(3,5).
\]
Here $A\cap B=(1,2)$ and $B\cap C=(3,4)$ both have dimension one,
but $A\cap C=\varnothing$. Thus overlap of $A$ with $B$, and of $B$
with $C$, does not imply overlap of $A$ with $C$. This example explains
why transitivity, meaning this implication between three sets, cannot be
deduced from semialgebraicity alone. It is not a construction of three
covariance families in our graph class. Our proofs do not need such an
implication: they establish quasi-equivalence between each selected graph
and the data-generating graph directly. In
Lemma~\ref{lem:geometric-identification}, replacing $G^\dagger$ by the
true graph is justified by the equality
$\M(G^\dagger)=\M(G^{*+})$, not by transitivity of quasi-equivalence.

\paragraph{Why equal Zariski closures are insufficient.}
The same intervals also illustrate a separate distinction. All three
have Zariski closure $\mathbb R$, since a polynomial that equals zero on
a nonempty interval equals zero on the whole line. Nevertheless, $A$ and
$C$ are disjoint. Consequently, equality of Zariski closures alone does
not guarantee overlap for semialgebraic sets: it establishes that the
sets satisfy the same polynomial equalities, but does not determine which
points satisfy their inequalities. This is why
Assumption~\ref{ass:model-overlap} explicitly requires full-dimensional
intersection of the covariance families, even after their Zariski
closures have been shown to be equal.

\subsection{Sufficient conditions for geometric identification}

We first state the standard algebraic dimension fact used below; see
\citet[Section~3.3, p.~59]{coste2000}. We include its proof to make the
strict-dimension step explicit.

\begin{lemma}[Strict inclusion and algebraic dimension]
\label{lem:strict-inclusion-dimension}
Let $V_1,V_2\subseteq\mathbb R^m$ be nonempty irreducible real algebraic
sets. If $V_1\subsetneq V_2$, then $\dim V_1<\dim V_2$.
\end{lemma}

\begin{proof}
By Definition~\ref{def:algebraic-dimension}, algebraic dimension is the
largest number of strict inclusions in a chain of nonempty irreducible
algebraic subsets. Let
$d=\dim V_1$. Since $V_1$ is irreducible, a chain of maximal length in
$V_1$ ends at $V_1$; otherwise, appending $V_1$ would give a longer chain.
Thus there exist nonempty irreducible algebraic sets with
\[
    Z_0\subsetneq Z_1\subsetneq\cdots\subsetneq Z_d=V_1.
\]
The assumed strict inclusion and the irreducibility of $V_2$ allow us to
append one more set:
\[
    Z_0\subsetneq Z_1\subsetneq\cdots\subsetneq Z_d=V_1
    \subsetneq V_2.
\]
This is a chain of $d+1$ strict inclusions in $V_2$, so
$\dim V_2\geq d+1>\dim V_1$.

The same argument applies to any nonempty proper algebraic subset
$W\subsetneq V_2$, even if $W$ is not irreducible. A chain of maximal
length $d=\dim W$ consists of irreducible sets and ends at some
$Z_d\subseteq W$. Since $Z_d\subsetneq V_2$, appending $V_2$ again
gives $\dim V_2\geq d+1$. Thus every nonempty proper algebraic subset
of an irreducible algebraic set has strictly smaller dimension.
\end{proof}

\begin{lemma}[Geometric identification]
\label{lem:geometric-identification}
Under Assumptions~\ref{ass:algebraic-faithfulness}--\ref{ass:model-overlap},
Assumption~\ref{ass:pointwise-identification} holds.  Equivalently, for every
$G^+\in\Gclass$,
\begin{equation}
    \Sigma^*\in\Mbar(G^+),\quad c(G^+)\leq c^\dagger
    \quad\Longrightarrow\quad G^+\cong_O G^{*+}.
    \label{eq:geometric-identification}
\end{equation}
\end{lemma}

\begin{proof}
Fix a candidate $G^+$ satisfying the two premises in
Equation~\ref{eq:geometric-identification}.  Relevant-candidate algebraic
faithfulness gives
\begin{equation}
    V(G^\dagger)\subseteq V(G^+).
    \label{eq:proof-variety-inclusion}
\end{equation}

Suppose that the inclusion is strict. Proposition~\ref{prop:model-geometry}
shows that both real algebraic sets are irreducible. Applying
Lemma~\ref{lem:strict-inclusion-dimension} with
$V_1=V(G^\dagger)$ and $V_2=V(G^+)$ gives
\[
    \dim V(G^\dagger)<\dim V(G^+).
\]
Using the dimension identities in Proposition~\ref{prop:model-geometry}, this
is equivalent to
\[
    \dim\M(G^\dagger)<\dim\M(G^+).
\]
Structural minimality would therefore imply $c(G^+)>c^\dagger$, contradicting
the premise $c(G^+)\leq c^\dagger$.  Thus
\[
    V(G^+)=V(G^\dagger)=:V.
\]

Equality of the Zariski closures alone is insufficient because the two
semialgebraic covariance families may be disjoint subsets of $V$.
Full-dimensional overlap supplies
\begin{equation}
    \dim\bigl(\M(G^+)\cap\M(G^\dagger)\bigr)=\dim V.
    \label{eq:proof-overlap-dimension}
\end{equation}
Proposition~\ref{prop:model-geometry} also gives
\[
    \dim\M(G^+)=\dim V=\dim\M(G^\dagger).
\]
Consequently,
\[
\begin{split}
    \dim\bigl(\M(G^+)\cap\M(G^\dagger)\bigr)
    &=\dim\M(G^+)\\
    &=\dim\M(G^\dagger),
\end{split}
\]
which is exactly $G^+\cong_O G^\dagger$.

Finally, $\M(G^\dagger)=\M(G^{*+})$ by construction.  Replacing
$\M(G^\dagger)$ by this equal set in the preceding dimension identity proves
$G^+\cong_O G^{*+}$.  The argument applies to every member of
$\mathcal G_0(\Sigma^*)$, so Assumption
\ref{ass:pointwise-identification} holds.
\end{proof}

\subsection{KL separation and convergence of the minimum KL divergence}

\begin{feedbackrevision}
In the finite-dimensional space $\Sym^p$, a set is compact if and only if
it is closed and bounded in that space. A continuous real-valued function
on a nonempty compact set attains its minimum.

We first establish a compactness result used in both lemmas below. It
applies to a fixed first covariance and, more generally, to first
covariances varying over a compact set.

\begin{lemma}[Compactness for bounded Gaussian KL divergence]
\label{lem:kl-sublevel-compactness}
Let $K\subseteq\PD^p$ be nonempty and compact, and let $0\leq t<\infty$.
Then the set
\[
    K_t:=\left\{\Sigma\in\PD^p:
       \inf_{S\in K}D(S\Vert\Sigma)\leq t\right\}
\]
is compact in $\Sym^p$ and consists of positive-definite matrices.
In particular, taking $K=\{\Sigma^*\}$ shows that
$\{\Sigma\in\PD^p:D(\Sigma^*\Vert\Sigma)\leq t\}$ is compact.
\end{lemma}

\begin{proof}
We first bound the eigenvalues of every $\Sigma\in K_t$ away from zero
and above by a finite constant. Because $K$ is compact and its matrices
are positive definite, continuity of their smallest and largest
eigenvalues gives constants $0<m\leq M<\infty$ such that
\[
    mI_p\preceq S\preceq MI_p\qquad(S\in K).
\]
Write $u_1,\ldots,u_p>0$ for the eigenvalues of $\Sigma$. Expanding
Equation~\ref{eq:gaussian-kl}, and using
$\tr(\Sigma^{-1}S)\geq m\tr(\Sigma^{-1})$ and
$\log\det S\leq p\log M$, gives
\[
    D(S\Vert\Sigma)
    \geq\frac12\left\{\sum_{i=1}^p
      \left(\frac{m}{u_i}+\log u_i\right)-p\log M-p\right\}
    \qquad(S\in K).
\]
For $u>0$, let $h(u)=m/u+\log u$. Its derivative is
$h'(u)=(u-m)/u^2$, so its minimum is $h(m)=1+\log m$.
Also, $h(u)\to+\infty$ as $u\downarrow0$ or $u\to+\infty$.
The preceding lower bound is independent of $S$. Thus, if $\Sigma\in K_t$,
\[
    \sum_{i=1}^p h(u_i)\leq2t+p\log M+p.
\]
Each of the other $p-1$ terms is at least $h(m)$, so each eigenvalue satisfies
\[
    h(u_i)\leq2t+p\log M+p-(p-1)h(m).
\]
Since the right-hand side is finite and $h$ tends to infinity at both
ends of $(0,\infty)$, there are constants $0<a<b<\infty$, depending
only on $K$ and $t$, such that $a\leq u_i\leq b$ for every $i$.
Consequently,
\[
    K_t\subseteq B:=\{\Sigma\in\Sym^p:aI_p\preceq\Sigma\preceq bI_p\}.
\]
The set $B$ is closed and bounded, hence compact in the finite-dimensional
space $\Sym^p$, and $a>0$ ensures that $B\subseteq\PD^p$.

It remains to show that $K_t$ is closed in $B$. Define
$f(\Sigma)=\min_{S\in K}D(S\Vert\Sigma)$; this minimum exists by
continuity and compactness of $K$. For $\Sigma,\widetilde\Sigma\in B$,
\[
    |f(\Sigma)-f(\widetilde\Sigma)|
    \leq\sup_{S\in K}
       |D(S\Vert\Sigma)-D(S\Vert\widetilde\Sigma)|.
\]
Joint continuity of $D$ on the compact product $K\times B$ implies
uniform continuity there, so the right-hand side tends to zero as
$\widetilde\Sigma\to\Sigma$. Hence $f$ is continuous on $B$.
It follows that $K_t=\{\Sigma\in B:f(\Sigma)\leq t\}$ is a closed
subset of the compact set $B$, and is therefore compact.
\end{proof}
\end{feedbackrevision}

\begin{lemma}[Positive KL gap]
\label{lem:positive-kl-gap}
\feedbackedit{Let $\Sigma^*\in\PD^p$.} If $C\subseteq\PD^p$ is closed relative to $\PD^p$ and
$\Sigma^*\notin C$, then
$\inf_{\Sigma\in C}D(\Sigma^*\Vert\Sigma)>0$.
\end{lemma}

\begin{proof}
\begin{feedbackrevision}
If $C$ is empty, the infimum is $+\infty$ and the result is immediate.
Assume henceforth that $C$ is nonempty. Although every $\Sigma\in C$
differs from $\Sigma^*$ and therefore has positive KL divergence, this
alone does not rule out an infimum of zero. We first prove that the
infimum is attained; closedness of $C$ alone does not guarantee this.

Restrict the minimization without changing its infimum.
Choose any $\Sigma_0\in C$ and set
$t=D(\Sigma^*\Vert\Sigma_0)<\infty$. This candidate gives
$\inf_{\Sigma\in C}D(\Sigma^*\Vert\Sigma)\leq t$. Define
\[
    K_t=\{\Sigma\in\PD^p:D(\Sigma^*\Vert\Sigma)\leq t\}.
\]
The intersection $C\cap K_t$ contains $\Sigma_0$, whose divergence is $t$.
Every covariance in $C\setminus K_t$ has divergence greater than $t$,
so excluding these covariances does not change the infimum:
\[
    \inf_{\Sigma\in C}D(\Sigma^*\Vert\Sigma)
    =\inf_{\Sigma\in C\cap K_t}D(\Sigma^*\Vert\Sigma).
\]

Lemma~\ref{lem:kl-sublevel-compactness}, with $K=\{\Sigma^*\}$, proves
that $K_t$ is compact. Since $C$ is closed relative to $\PD^p$ and
$K_t\subseteq\PD^p$, the intersection $C\cap K_t$ is closed in $K_t$.
It is therefore a nonempty compact set. The continuous function
$\Sigma\mapsto D(\Sigma^*\Vert\Sigma)$ attains its minimum on this set
at some $\widehat\Sigma\in C\cap K_t$. Combining this fact with Step~1
gives
\[
    \inf_{\Sigma\in C}D(\Sigma^*\Vert\Sigma)
    =\min_{\Sigma\in C\cap K_t}D(\Sigma^*\Vert\Sigma)
    =D(\Sigma^*\Vert\widehat\Sigma).
\]

Suppose, for a contradiction, that this minimum equals zero. Gaussian KL
divergence is zero only when its two covariances are equal, so
$D(\Sigma^*\Vert\widehat\Sigma)=0$ implies $\widehat\Sigma=\Sigma^*$.
But $\widehat\Sigma\in C$, whereas $\Sigma^*\notin C$ by assumption.
This is a contradiction. The attained minimum is nonnegative and cannot
be zero, so it is strictly positive, as required.
\end{feedbackrevision}
\end{proof}

\begin{lemma}[Convergence of the minimum empirical KL divergence]
\label{lem:profile-convergence}
For fixed $G^+$, let
\begin{equation}
    \Delta(G^+)=\inf_{\Sigma\in\Mbar(G^+)}
    D(\Sigma^*\Vert\Sigma).
    \label{eq:population-gap}
\end{equation}
Then $\delta_n(G^+)\to\Delta(G^+)$ almost surely.  If
$\Sigma^*\notin\Mbar(G^+)$, then $\Delta(G^+)>0$.
\end{lemma}

\begin{feedbackrevision}
Here, almost-sure convergence means that the convergence event has
probability one:
\[
    \Pr\!\left(\lim_{n\to\infty}\delta_n(G^+)=\Delta(G^+)\right)=1.
\]
The probability is taken over the infinite sequence of i.i.d.\ observations.
\end{feedbackrevision}

\begin{proof}
Because the observations have finite second moments, the strong law applied
entrywise to $X_O^{(k)}X_O^{(k)\top}$ gives
\[
    S_n\longrightarrow\Sigma^*\qquad\text{almost surely}.
\]
Fix an outcome in this probability-one event for the remainder of the proof.
Choose $\varepsilon>0$ smaller than the smallest eigenvalue of
$\Sigma^*$ and let
\[
    K=\{S\in\Sym^p:\|S-\Sigma^*\|_F\leq\varepsilon\}.
\]
This closed ball is compact (closed and bounded in the finite-dimensional
space $\Sym^p$), lies in $\PD^p$, and contains $S_n$ for all sufficiently
large $n$.

The model closure $\Mbar(G^+)$ is nonempty because the admissible parameter
space is nonempty.  Fix any $\Sigma_0\in\Mbar(G^+)$.  Joint continuity of
$D(S\Vert\Sigma)$ on $\PD^p\times\PD^p$ and compactness of $K$ imply
\[
    R_0:=\sup_{S\in K}D(S\Vert\Sigma_0)<\infty.
\]
Thus $\inf_{\Sigma\in\Mbar(G^+)}D(S\Vert\Sigma)\leq R_0$ uniformly over
$S\in K$.

We next show that the infimum over $\Sigma$ can be restricted to a single
compact set for every $S\in K$. \begin{feedbackrevision}
By Lemma~\ref{lem:kl-sublevel-compactness}, applied with this $K$ and
$t=R_0+1$, the set
\[
    K_0=\{\Sigma\in\PD^p:
    \inf_{S\in K}D(S\Vert\Sigma)\leq R_0+1\}
\]
is compact and contained in $\PD^p$.
\end{feedbackrevision}

For every $S\in K$, the infimum over $\Sigma$ can now be restricted to
$\Mbar(G^+)\cap K_0$.  Indeed, if $\Sigma\notin K_0$, then
$D(S\Vert\Sigma)>R_0+1$, while the fixed feasible covariance $\Sigma_0$ has
$D(S\Vert\Sigma_0)\leq R_0$.  Thus points outside $K_0$ cannot improve the
infimum.  The restricted feasible set $\Mbar(G^+)\cap K_0$ is nonempty
(it contains $\Sigma_0$) and compact. Thus the minimum KL divergence over
the covariance-family closure is attained, although admissible SCM
parameters need not attain the same value.

The function $D$ is uniformly continuous on the compact product
$K\times K_0$.  Since $S_n\to\Sigma^*$ and both matrices eventually belong
to $K$, it follows that
\[
    \sup_{\Sigma\in K_0}
    |D(S_n\Vert\Sigma)-D(\Sigma^*\Vert\Sigma)|\to0
\]
on the fixed probability-one event.  Finally, for any two real-valued
functions $f_n$ and $f$ \feedbackedit{with finite infima on a common nonempty set} $C$,
\[
    \left|\inf_C f_n-\inf_C f\right|
    \leq\sup_C|f_n-f|.
\]
Applying this inequality with
$C=\Mbar(G^+)\cap K_0$ proves
\[
    \delta_n(G^+)\longrightarrow\Delta(G^+)
\]
almost surely.

If $\Sigma^*\notin\Mbar(G^+)$, the set $\Mbar(G^+)$ is closed relative to
$\PD^p$ by definition.  Lemma~\ref{lem:positive-kl-gap}, applied with
$C=\Mbar(G^+)$, then gives $\Delta(G^+)>0$.
\end{proof}

\subsection{Proof of Theorem~\ref{thm:closure-consistency}}

We use the following probability notation for real random variables $Y_n$
and positive deterministic numbers $a_n$. The notation $Y_n=O_p(a_n)$
means that, for every $\varepsilon>0$, there is a finite $M>0$ such that
$\Pr(|Y_n|>Ma_n)<\varepsilon$ for all sufficiently large $n$.
The notation $Y_n=o_p(a_n)$ means that
$\Pr(|Y_n|>\varepsilon a_n)\to0$ for every $\varepsilon>0$.
We write $Y_n\xrightarrow{p}y$ for convergence in probability to $y$,
meaning $Y_n-y=o_p(1)$, and $Y_n\xrightarrow{p}+\infty$ if
$\Pr(Y_n>T)\to1$ for every fixed real $T$.
For deterministic sequences, $a_n=o(n)$ means $a_n/n\to0$.

\begin{proof}
\medskip
\begin{feedbackrevision}
By definition, $G^\dagger$ is distribution-equivalent to the true graph
$G^{*+}$. Hence
\[
    \Sigma^*\in\M(G^{*+})=\M(G^\dagger)
    \subseteq\Mbar(G^\dagger).
\]
Recall that $\delta_n(G^\dagger)$ is the infimum of
$D(S_n\Vert\Sigma)$ over $\Sigma\in\Mbar(G^\dagger)$.
Every divergence in this infimum is nonnegative, giving the lower bound
below. Since $\Sigma^*$ belongs to the set over which the infimum is
taken, the infimum is no greater than the divergence evaluated at
$\Sigma^*$. Thus
\[
    0\leq\delta_n(G^\dagger)
    =\inf_{\Sigma\in\Mbar(G^\dagger)}D(S_n\Vert\Sigma)
    \leq D(S_n\Vert\Sigma^*).
\]
This upper bound does not require the infimum to be attained or assert
that $\Sigma^*$ minimizes the empirical divergence. It uses only that
$\Sigma^*$ is one of the covariances included in the infimum.
\end{feedbackrevision}

\begin{feedbackrevision}
We next bound $D(S_n\Vert\Sigma^*)$ to show that
$n\delta_n(G^\dagger)$ is bounded in probability. Taylor's theorem first
reduces this task to bounding the squared covariance error
$\|S_n-\Sigma^*\|_F^2$.

Define $f(S)=D(S\Vert\Sigma^*)$ for $S\in\PD^p$.
The Gaussian KL formula in Equation~\ref{eq:gaussian-kl} shows that $f$
is twice continuously differentiable on $\PD^p$.
Since $f(\Sigma^*)=0$ is its minimum and $\Sigma^*$ is an interior
point of $\PD^p$ in $\Sym^p$, the first derivative there is zero.
Choose $r>0$ so that the closed Frobenius ball
$\{S\in\Sym^p:\|S-\Sigma^*\|_F\leq r\}$ lies in $\PD^p$.
The continuous second derivatives are bounded on this compact ball.
The standard quadratic Taylor bound
\citep[Section~9.1.2, Equation~(9.13), p.~461]{boyd2004}, applied at
$\Sigma^*$, therefore gives a finite constant $C>0$ such that
\[
    0\leq D(S\Vert\Sigma^*)\leq C\|S-\Sigma^*\|_F^2
    \qquad\text{whenever }\|S-\Sigma^*\|_F\leq r.
\]
Both $r$ and $C$ are fixed independently of $n$.

To obtain the covariance-error rate, we use Markov's inequality: for any
nonnegative random variable $Z$ with finite expectation and any $a>0$,
\[
    \Pr(Z>a)\leq\frac{\mathbb E Z}{a}.
\]
We will apply this inequality to
$Z=\|S_n-\Sigma^*\|_F^2$, with threshold $a=M^2/n$ for $M>0$.
We therefore first compute $\mathbb E\|S_n-\Sigma^*\|_F^2$.

For this calculation, fix $i,j$ and set
$Y_{ij,k}=X_{O,i}^{(k)}X_{O,j}^{(k)}$. Then
$(S_n)_{ij}=n^{-1}\sum_{k=1}^nY_{ij,k}$.
The observations are independent and identically distributed across $k$,
so these products are also independent and identically distributed across
$k$; independence of the two coordinates within an observation is not
required. The known zero mean gives
$\mathbb E Y_{ij,k}=\Sigma^*_{ij}$.
Their variance is finite because Gaussian coordinates have finite fourth
moments and the Cauchy--Schwarz inequality gives
\[
    \mathbb E Y_{ij,1}^2
    \leq\left\{
      \mathbb E\bigl[(X_{O,i}^{(1)})^4\bigr]
      \mathbb E\bigl[(X_{O,j}^{(1)})^4\bigr]
    \right\}^{1/2}<\infty.
\]
Consequently, each sample-covariance entry is unbiased. Independence across
observations implies that the variance of the average is the sum of the
individual variances divided by $n^2$:
\begin{align*}
    \mathbb E\bigl[((S_n)_{ij}-\Sigma^*_{ij})^2\bigr]
    &=\operatorname{Var}((S_n)_{ij})\\
    &=\frac1{n^2}\sum_{k=1}^n\operatorname{Var}(Y_{ij,k})
      =\frac{\operatorname{Var}(Y_{ij,1})}{n}.
\end{align*}
The squared Frobenius norm is the sum of these squared entry errors.
Summing their expectations therefore gives
\[
    \mathbb E\|S_n-\Sigma^*\|_F^2=\frac Kn,
    \qquad
    K:=\sum_{i,j=1}^p\operatorname{Var}(Y_{ij,1})<\infty.
\]
Since $p$ and the true distribution are fixed, $K$ does not depend on $n$.
Substituting this expectation into Markov's inequality gives, for any
$M>0$,
\begin{align*}
    \Pr\!\left(\|S_n-\Sigma^*\|_F>\frac{M}{\sqrt n}\right)
    &=\Pr\!\left(\|S_n-\Sigma^*\|_F^2>\frac{M^2}{n}\right)\\
    &\leq\frac{\mathbb E\|S_n-\Sigma^*\|_F^2}{M^2/n}
      =\frac{K/n}{M^2/n}=\frac{K}{M^2}.
\end{align*}
For any $\varepsilon>0$, choosing $M>\sqrt{K/\varepsilon}$ makes this
probability less than $\varepsilon$ for every $n$. By the definition of
$O_p$ stated above, this proves
\[
    \|S_n-\Sigma^*\|_F=O_p(n^{-1/2}).
\]

Finally, combine the local Taylor bound with the expectation bound.
Outside the ball where the Taylor bound applies, Markov's inequality
gives $\Pr(\|S_n-\Sigma^*\|_F>r)\leq K/(nr^2)$.
Inside that ball, $nD(S_n\Vert\Sigma^*)$ is at most
$nC\|S_n-\Sigma^*\|_F^2$. Hence, for every $T>0$,
\[
    \Pr\!\left(nD(S_n\Vert\Sigma^*)>T\right)
    \leq\frac{K}{nr^2}+\frac{CK}{T}.
\]
For any prescribed probability bound, the second term can be made
arbitrarily small by choosing $T$ large, and the first by taking $n$
large. Thus $nD(S_n\Vert\Sigma^*)=O_p(1)$, or equivalently
$D(S_n\Vert\Sigma^*)=O_p(n^{-1})$.
The preceding inequality
$0\leq n\delta_n(G^\dagger)\leq nD(S_n\Vert\Sigma^*)$ now gives
\begin{equation}
    n\delta_n(G^\dagger)=O_p(1).
    \label{eq:comparator-rate}
\end{equation}
\end{feedbackrevision}

\medskip
\noindent\textbf{Complexity of every global minimizer.}
Let
\[
    \widehat G_n^+
    \in\operatorname*{arg\,min}_{G^+\in\Gclass}s_n(G^+)
\]
be any global minimizer.  Score optimality and cancellation of the common
constant $C_n$ give
\[
    n\delta_n(\widehat G_n^+)+\lambda_nc(\widehat G_n^+)
    \leq n\delta_n(G^\dagger)+\lambda_nc^\dagger.
\]
\begin{feedbackrevision}
Dropping the nonnegative first term on the left and dividing by
$\lambda_n>0$ yields
\[
    c(\widehat G_n^+)
    \leq c^\dagger+
    \frac{n\delta_n(G^\dagger)}{\lambda_n}.
\]
Equation~\ref{eq:comparator-rate} states that the numerator
$n\delta_n(G^\dagger)$ is bounded in probability. This does not imply
that the numerator converges to zero. However, dividing it by the
deterministic sequence $\lambda_n\to\infty$ makes the ratio converge to
zero in probability. To verify this, fix $\varepsilon>0$ and $\eta>0$.
Boundedness in probability provides a finite $M>0$ such that
$\Pr(n\delta_n(G^\dagger)>M)<\eta$ for all sufficiently large $n$.
For such $n$ large enough that $\varepsilon\lambda_n>M$, we obtain
\[
    \Pr\!\left(\frac{n\delta_n(G^\dagger)}{\lambda_n}
                 >\varepsilon\right)
    =\Pr\!\left(n\delta_n(G^\dagger)>\varepsilon\lambda_n\right)
    \leq\Pr\!\left(n\delta_n(G^\dagger)>M\right)<\eta.
\]
Since $\eta$ can be arbitrarily small, the ratio converges to zero in
probability, which is precisely
\[
    \frac{n\delta_n(G^\dagger)}{\lambda_n}=o_p(1).
\]
Thus the complexity bound above can be written as
$c(\widehat G_n^+)\leq c^\dagger+o_p(1)$.
Because complexity is integer-valued,
$c(\widehat G_n^+)>c^\dagger$ implies
$c(\widehat G_n^+)-c^\dagger\geq1$.
The bound therefore gives
\[
    \Pr\!\left(c(\widehat G_n^+)>c^\dagger\right)
    \leq\Pr\!\left(\frac{n\delta_n(G^\dagger)}{\lambda_n}
                     \geq1\right)\longrightarrow0.
\]
\end{feedbackrevision}
This conclusion
holds simultaneously for every global minimizer because the upper bound does
not depend on which minimizer is chosen.

\medskip
\noindent\textbf{Excluding families whose closure does not contain the true covariance.}
Fix a candidate $G^+\in\Gclass$ such that
$\Sigma^*\notin\Mbar(G^+)$.  \begin{feedbackrevision}
Lemma~\ref{lem:positive-kl-gap} gives
\[
    \Delta(G^+)
    =\inf_{\Sigma\in\Mbar(G^+)}D(\Sigma^*\Vert\Sigma)>0.
\]
Lemma~\ref{lem:profile-convergence} gives
$\delta_n(G^+)\to\Delta(G^+)$ almost surely. Almost-sure convergence
implies convergence in probability, so, for every $\varepsilon>0$,
\[
    \Pr\!\left(|\delta_n(G^+)-\Delta(G^+)|>\varepsilon\right)
    \longrightarrow0.
\]
By the definition of $o_p(1)$ stated above, this is equivalent to
\[
    \delta_n(G^+)=\Delta(G^+)+o_p(1).
\]
Equation~\ref{eq:graph-score} writes each graph score as the sum of the
common data-dependent term $C_n$, the empirical KL-divergence infimum
multiplied by $n$, and the structural penalty. Applying this formula to
both graphs, with $c(G^\dagger)=c^\dagger$, gives
\begin{align*}
    s_n(G^+)-s_n(G^\dagger)
    &=\bigl[C_n+n\delta_n(G^+)+\lambda_nc(G^+)\bigr]\\
    &\quad-\bigl[C_n+n\delta_n(G^\dagger)+\lambda_nc^\dagger\bigr]\\
    &=n\delta_n(G^+)-n\delta_n(G^\dagger)
      +\lambda_n\{c(G^+)-c^\dagger\}.
\end{align*}
The two occurrences of $C_n$ cancel because both scores use the same data;
this term does not depend on the candidate graph.

The convergence above gives $\delta_n(G^+)-\Delta(G^+)=o_p(1)$.
Multiplication by $n$ yields
\[
    n\delta_n(G^+)
    =n\Delta(G^+)+n\{\delta_n(G^+)-\Delta(G^+)\}
    =n\Delta(G^+)+o_p(n).
\]
Indeed, the remainder divided by $n$ is
$\delta_n(G^+)-\Delta(G^+)$, which converges to zero in probability.
For the reference graph, Equation~\ref{eq:comparator-rate} gives
$n\delta_n(G^\dagger)=O_p(1)$. Substituting these two bounds into the
exact score difference therefore gives
\begin{align*}
    s_n(G^+)-s_n(G^\dagger)
    &=n\Delta(G^+)+o_p(n)-O_p(1)\\
    &\quad+\lambda_n\{c(G^+)-c^\dagger\}.
\end{align*}
Here the subtracted $O_p(1)$ term is the nonnegative reference quantity
$n\delta_n(G^\dagger)$, which is bounded in probability.

Both graphs are fixed, so their complexity difference is constant.
Since $\lambda_n/n\to0$, the penalty difference divided by $n$ tends
to zero, even if the candidate has smaller complexity than the reference.
Also, an $O_p(1)$ term divided by $n$ converges to zero in probability.
Consequently,
\[
    \frac{s_n(G^+)-s_n(G^\dagger)}{n}
    =\Delta(G^+)+o_p(1)
    \xrightarrow{p}\Delta(G^+)>0.
\]
In particular, the score difference is at least $n\Delta(G^+)/2$ with
probability tending to one. This lower bound tends to infinity, so
$s_n(G^+)-s_n(G^\dagger)\xrightarrow{p}+\infty$.
\end{feedbackrevision}

Thus, for this candidate, the probability of being a global score
minimizer tends to zero.

The class $\Gclass$ is finite.  A finite union bound therefore makes the
preceding exclusion simultaneous over all candidates whose covariance-family
closures do not contain $\Sigma^*$. Combining that event with the
simultaneous complexity bound shows
that, with probability tending to one, every global minimizer satisfies
\[
    \Sigma^*\in\Mbar(\widehat G_n^+),
    \qquad c(\widehat G_n^+)\leq c^\dagger.
\]
Equivalently,
\[
    \operatorname*{arg\,min}_{G^+\in\Gclass}s_n(G^+)
    \subseteq\mathcal G_0(\Sigma^*),
\]
which proves Equation~\ref{eq:closure-selection-conclusion}.

The proof used only $\lambda_n\to\infty$ for the complexity comparison and
$\lambda_n=o(n)$ for exclusion of misspecified covariance models.  Hence it
also proves the stated extension to general deterministic penalty sequences.
\end{proof}

\subsection{Proof of Corollary~\ref{cor:quasi-consistency}}
\label{app:proof-quasi-consistency}

\begin{proof}
Let
\[
    \mathcal A_n:=\operatorname*{arg\,min}_{G^+\in\Gclass}s_n(G^+)
\]
be the set of all globally minimizing graphs. This set is nonempty because
the candidate class is finite and every graph has a finite score, as shown
in Appendix~\ref{app:observed-score-derivation}. Define the event
\[
    E_n:=\{\mathcal A_n\subseteq\mathcal G_0(\Sigma^*)\}.
\]
Theorem~\ref{thm:closure-consistency} gives $\Pr(E_n)\to1$.

Fix a data realization in $E_n$ and take any
$\widehat G_n^+\in\mathcal A_n$. By the definition of $\mathcal G_0(\Sigma^*)$,
\[
    \Sigma^*\in\Mbar(\widehat G_n^+),\qquad
    c(\widehat G_n^+)\leq c^\dagger.
\]
Assumption~\ref{ass:pointwise-identification} therefore gives
$\widehat G_n^+\cong_O G^\dagger$. By the definition of marginal
quasi-equivalence, this means
\[
    \dim\bigl(\M(\widehat G_n^+)\cap\M(G^\dagger)\bigr)
    =\dim\M(\widehat G_n^+)=\dim\M(G^\dagger).
\]
The reference graph $G^\dagger$ is distribution-equivalent to the true
graph, so $\M(G^\dagger)=\M(G^{*+})$. Substituting this equality into the
preceding identity gives
\[
    \dim\bigl(\M(\widehat G_n^+)\cap\M(G^{*+})\bigr)
    =\dim\M(\widehat G_n^+)=\dim\M(G^{*+}).
\]
Thus $\widehat G_n^+\cong_O G^{*+}$. This step uses equality of covariance
families, not transitivity of quasi-equivalence.

Because the minimizer was arbitrary, the conclusion holds simultaneously
for every graph in $\mathcal A_n$, including when the score has ties. Hence
\[
    E_n\subseteq
    \left\{\mathcal A_n\subseteq
        \{G^+\in\Gclass:G^+\cong_O G^{*+}\}\right\}.
\]
It follows that
\[
    1\geq
    \Pr\!\left\{\mathcal A_n\subseteq
        \{G^+\in\Gclass:G^+\cong_O G^{*+}\}\right\}
    \geq\Pr(E_n)\longrightarrow1,
\]
which proves Equation~\ref{eq:quasi-consistency-conclusion}.
\end{proof}

\subsection{Details and proof of the Bernoulli formulation}
\label{app:bernoulli-details}

This subsection supplies the parameterization, likelihood, and
differentiability details used in Section~\ref{sec:bernoulli}, followed by
the proof of Proposition~\ref{prop:bernoulli-exactness}.

\paragraph{Effective latent loadings.}
For a graph with $\ell$ latent variables, place
$\Omega_L^{1/2}\Lambda$ in $\ell$ rows of $\Gamma$ and fill the remaining
$L_{\max}-\ell$ rows with zeros. The corresponding gates select the
graph's observed edges, active latent slots, and latent-to-observed edges.
Absorbing latent noise standard deviations preserves the covariance
family: conversely, any effective loadings can be represented using unit
latent variances. All configurations share the same continuous parameters
$\vartheta=(W,\Gamma,\Omega_O)$. For every binary $(H,Z)$, the latent
contribution in Equation~\ref{eq:masked-parameters} is positive
semidefinite, so $\Psi(H,Z)\succ0$ because $\Omega_O\succ0$.

\paragraph{Sampled likelihood and complexity.}
For nonsingular $A(M)$, evaluate Equation~\ref{eq:full-observed-nll}
at the masked covariance and remove the Gaussian constant
$np\log(2\pi)/2$. Using the identities derived in
Appendix~\ref{app:observed-score-derivation}, this gives
\begin{equation}
\begin{split}
    L_n(M,H,Z;\vartheta)={}&-n\log|\det A(M)|+\frac n2\log\det\Psi(H,Z)\\
    &+\frac n2\tr\!\left[\Psi(H,Z)^{-1}
    A(M)^\top S_n A(M)\right].
\end{split}
\label{eq:masked-likelihood}
\end{equation}
As in the main text, $L_n=+\infty$ when $A(M)$ is singular.
The sampled structural complexity is
\begin{equation}
    c(M,H,Z)=\|M\|_1+
    \sum_{h=1}^{L_{\max}}Z_h\bigl(1+\|H_{h\cdot}\|_1\bigr),
    \label{eq:masked-complexity}
\end{equation}
where $H_{h\cdot}$ denotes row $h$ of $H$. The first term counts selected
edges among observed variables. Each active slot contributes one latent variable and one
unit for each of its selected edges to observed variables. An inactive
slot contributes neither covariance nor complexity, regardless of its
edge gates. The active latent count is $\sum_hZ_h$, with expectation
$\sum_h\pi_h$. Independence of the gates gives
\[
    \mathbb E[c(M,H,Z)]
    =\sum_{i\neq j}q_{ij}
      +\sum_{h=1}^{L_{\max}}\pi_h
        \left(1+\sum_{j=1}^p r_{hj}\right).
\]
Thus Equation~\ref{eq:exact-bernoulli-objective} is exactly
$\mathbb E[L_n+\lambda_nc]$, not an approximation obtained by replacing
the binary masks in the likelihood with their probabilities.

\paragraph{Differentiability and singular configurations.}
Fix $\vartheta$ and suppose every binary configuration has finite
likelihood. Its probability is a product of gate probabilities or their
complements, with at most one factor involving each scalar probability.
The expectation is a finite sum of these products multiplied by the
configuration scores. It is therefore a polynomial that is affine
(linear plus a constant) in each scalar probability when all others are
fixed. In particular, it is differentiable in the probabilities, but need
not be convex in all of them jointly.

{\color{black}
The same finite-sum representation also gives joint differentiability in
the probabilities and continuous parameters. At a parameter vector
$\vartheta$ for which every binary configuration has finite likelihood,
the noise variances remain positive and all required determinants remain
nonzero on a sufficiently small open neighborhood of $\vartheta$. Each
configuration's likelihood in Equation~\ref{eq:masked-likelihood} is
continuously differentiable there because the matrix inverses and logarithms
in that expression are continuously differentiable on the stated domains.
Adding each configuration's constant complexity penalty preserves this
property. The finite sum of configuration scores weighted by polynomial
configuration probabilities is therefore jointly continuously differentiable
in the gate probabilities and $\vartheta$ on this domain.
}

If some configurations have singular $A(M)$, any positive probability
assigned to one of them makes the objective infinite. A configuration of
zero probability is excluded from the expectation; no product of zero
and infinity is evaluated. Differentiability at a finite objective value
does not automatically extend to probability changes that activate a
singular configuration. Instead, if some probabilities are fixed at zero
or one so that every singular configuration is excluded throughout the
remaining probability domain, the same polynomial argument applies to
the free probabilities on that domain. Probabilities equal to zero or one
are permitted; when all are binary, the distribution selects a single
configuration deterministically.

\begin{feedbackrevision}
\paragraph{Simultaneous invertibility.}
When every observed-edge probability lies strictly between zero and one,
all binary observed-edge masks have positive probability. Finite expected
likelihood then requires
\[
    \det(I_p-M\od W)\neq0
    \qquad\text{for every binary observed-edge mask }M.
\]
For each fixed mask, this determinant is a polynomial in the off-diagonal
entries of $W$, and it equals one at $W=0$. It is therefore not the zero
polynomial. The excluded weights form a finite union of zero sets of
nonzero polynomials. Each such zero set has Lebesgue measure zero, as
explained in Appendix~\ref{app:genericity}, so their union also has measure
zero. The permitted weights form a Euclidean open set because there are
only finitely many determinant conditions and each determinant is
continuous. Thus this requirement does not force a norm bound or an
acyclic graph, although approaching a singular configuration can cause
numerical difficulties.

\end{feedbackrevision}
\paragraph{Proof of Proposition~\ref{prop:bernoulli-exactness}.}

\begin{proof}
Let $\mathcal K$ denote the finite set of all binary triples
$k=(M,H,Z)$.  The optimization has a finite feasible value: take
$W=\Gamma=0$ and $\Omega_O=I_p$.  Moreover, every finite likelihood is
bounded below by $K_n$ from Equation~\ref{eq:parameter-covariance-profile},
and the complexity penalty is nonnegative.  Hence the infima below are finite.
For fixed gate probabilities $(Q,R,\pi)$, write $w_k(Q,R,\pi)$ for the
product Bernoulli probability of configuration $k$.
Then $w_k\geq0$, $\sum_{k\in\mathcal K}w_k=1$, and, for every
$\vartheta$ for which the expected objective is finite,
\[
    \mathbb E[\feedbackedit{F_n^{\mathrm{bin}}}(M,H,Z;\vartheta)]
    =\sum_{k\in\mathcal K}w_k(Q,R,\pi)\feedbackedit{F_n^{\mathrm{bin}}}(k;\vartheta).
\]
Terms with zero probability make no contribution. The remaining terms
form a \emph{convex combination}: a weighted average with nonnegative
weights summing to one. At least one score in this average is no larger
than the average itself. Hence there is a configuration
$k_0=(M_0,H_0,Z_0)$ in the \emph{support}, meaning $w_{k_0}>0$, such that
\begin{equation}
    \feedbackedit{F_n^{\mathrm{bin}}}(k_0;\vartheta)
    \leq \mathbb E[\feedbackedit{F_n^{\mathrm{bin}}}(M,H,Z;\vartheta)].
    \label{eq:support-point-bound}
\end{equation}
For any such tuple $(Q,R,\pi,\vartheta)$,
\begin{align*}
    \mathbb E[\feedbackedit{F_n^{\mathrm{bin}}}(M,H,Z;\vartheta)]
    &\geq \feedbackedit{F_n^{\mathrm{bin}}}(k_0;\vartheta)\\
    &\geq \inf_{\vartheta'}\feedbackedit{F_n^{\mathrm{bin}}}(k_0;\vartheta')\\
    &\geq \min_{k\in\mathcal K}\inf_{\vartheta'}
       \feedbackedit{F_n^{\mathrm{bin}}}(k;\vartheta').
\end{align*}
Taking the infimum over $(Q,R,\pi,\vartheta)$ on the left proves
\[
    \inf_{Q,R,\pi,\vartheta}\mathbb E[\feedbackedit{F_n^{\mathrm{bin}}}]
    \geq\min_{k\in\mathcal K}\inf_\vartheta
       \feedbackedit{F_n^{\mathrm{bin}}}(k;\vartheta).
\]

For the reverse inequality, fix an arbitrary binary configuration
$k=(M,H,Z)$.  Choosing the entries of $Q$, $R$, and $\pi$ equal to the
corresponding zeros and ones in $k$ makes the product Bernoulli distribution
a \emph{point mass} at $k$, meaning $w_k=1$ and all other configuration
probabilities are zero. The expectation then equals
$\feedbackedit{F_n^{\mathrm{bin}}}(k;\vartheta)$ for every $\vartheta$. Taking the infimum over
$\vartheta$ and then choosing the best binary configuration gives
\[
    \inf_{Q,R,\pi,\vartheta}\mathbb E[\feedbackedit{F_n^{\mathrm{bin}}}]
    \leq\min_{k\in\mathcal K}\inf_\vartheta
       \feedbackedit{F_n^{\mathrm{bin}}}(k;\vartheta).
\]
The two inequalities prove Equation~\ref{eq:global-exactness}.

If the infimum of the expected objective is attained at a possibly fractional
tuple, apply Equation~\ref{eq:support-point-bound} to that minimizer.  Binary
gate probabilities concentrated on $k_0$ give a feasible point whose value is
no larger than the attained global minimum.  Its value must therefore be
equal to the minimum, proving the existence of a binary global minimizer.

It remains to show that optimizing over all binary masks has the same optimum
as optimizing over the candidate class $\Gclass$, which requires every active
latent to have at least two observed children.  First, a slot with $Z_h=0$
has no contribution to either $\Psi$ or the complexity in
Equation~\ref{eq:masked-complexity}; it can simply be omitted.  Next, suppose
$Z_h=1$ but $\|H_{h\cdot}\|_1=0$.  The $h$th latent contribution to $\Psi$ is
zero.  Setting $Z_h=0$ therefore preserves the likelihood and strictly lowers
the structural complexity.

Finally, suppose an active slot has exactly one child, indexed by $j$.  Let
$e_j$ be the $j$th standard basis vector in $\mathbb R^p$.  Its contribution
to the effective noise covariance is
\[
    (H_{h\cdot}\od\Gamma_{h\cdot})^\top
    (H_{h\cdot}\od\Gamma_{h\cdot})
    =\Gamma_{hj}^2e_je_j^\top.
\]
Define a new diagonal observed-noise covariance by
\[
    \Omega'_O=\Omega_O+\Gamma_{hj}^2e_je_j^\top
\]
and deactivate slot $h$.  The matrix $\Omega'_O$ remains diagonal and positive
definite, and the total effective covariance $\Psi$ is unchanged.  The
observed coefficient matrix $A(M)$ is also unchanged, so every term in the
likelihood is identical, while the structural complexity decreases by the
latent node and its single outgoing edge.  Repeating this operation removes
all active slots with fewer than two children without increasing the score.
This argument applies separately to a fixed binary configuration.
Because the adjustment of $\Omega_O$ depends on that configuration, it
does not define a simultaneous transformation of a Bernoulli distribution
whose configurations share the same continuous parameters.

After inactive slots are deleted and the remaining slots are relabeled, the
resulting mask represents a graph in $\Gclass$.  Conversely, every graph in
$\Gclass$ can be represented by some binary $(M,H,Z)$. For a fixed mask
representing $G^+$, the infimum of the binary-mask objective over the
continuous parameters is therefore
$K_n+n\delta_n(G^+)+\lambda_nc(G^+)$, by
Equation~\ref{eq:parameter-covariance-profile}.
Since $C_n=K_n+np\log(2\pi)/2$, this value equals
$s_n(G^+)-np\log(2\pi)/2$. The conversion from arbitrary masks to
admissible graphs cannot increase the objective, and every admissible
graph has a binary representation. Therefore the common infimum is
$\min_{G^+\in\Gclass}s_n(G^+)-np\log(2\pi)/2$, and the two formulations
have the same minimizing graphs after the stated conversion.
\end{proof}

\paragraph{Scope of the exactness result.}
The proposition does not assert that all globally minimizing probabilities
are binary. At a finite global minimum, a fractional probability vector
can also be optimal if every configuration assigned positive probability
attains that same minimum score at the shared continuous parameters.
The proof establishes the existence of a deterministic minimizer only
when the infimum is attained. Direct evaluation of the expectation sums
\feedbackedit{over exponentially many configurations. The closed-form complexity
expectation does not remove this cost for the likelihood expectation.
Equality of global infima neither bounds the variance of sampled gradient
estimates nor guarantees that a numerical optimizer reaches either
infimum.}

\section{Scope and genericity of the identification assumptions}
\label{app:assumption-scope}

We examine the three sufficient identification assumptions in
Section~\ref{sec:consistency-result}.
Subsection~\ref{app:genericity} establishes genericity of algebraic
faithfulness and full-dimensional overlap. Structural minimality concerns
the target covariance family, the entire candidate graph class, and the
chosen structural complexity, not the particular value of $\Sigma^*$.
It is not implied by the genericity result.
Subsection~\ref{app:minimality-limitation} shows that the penalty counting
edges and latent variables can violate structural minimality.
Subsection~\ref{app:positive-minimality} presents a weaker condition at the
true covariance and sufficient conditions for a restricted class of
covariance families.

\subsection{Genericity of algebraic faithfulness and full-dimensional overlap}
\label{app:genericity}

Fix a true graph and a finite candidate class. We show that algebraic
faithfulness and full-dimensional overlap hold for almost every admissible
choice of the true model parameters. Here, ``almost every'' means that
the exceptional set has Lebesgue measure zero, or zero volume, in the free
parameter coordinates. This is the meaning of \emph{generic} in this
result; it does not establish these assumptions at every true covariance.

\begin{lemma}[Genericity of faithfulness and model overlap]
\label{lem:genericity}
Assume Proposition~\ref{prop:model-geometry} holds for every graph in the
finite candidate class. Let $d$ be the number of free parameter coordinates
of the true graph, and write $\mathcal P(G^{*+})\subseteq\mathbb R^d$. Set
\[
    V^*:=V(G^{*+})=V(G^\dagger),
    \qquad
    \phi^*:\mathcal P(G^{*+})\longrightarrow V^*,
    \quad \phi^*(\theta)=\Sigma_O(\theta).
\]
Thus $\phi^*$ sends each admissible parameter vector to its observed
covariance. Outside a set of Lebesgue measure zero in $\mathbb R^d$, every
admissible true parameter vector satisfies the following statements:
\begin{enumerate}
    \item Assumption~\ref{ass:algebraic-faithfulness} holds simultaneously
    for all candidates;
    \item Assumption~\ref{ass:model-overlap} holds simultaneously for all
    candidates satisfying its premises.
\end{enumerate}
Moreover, for the second conclusion, the full-dimensional common subset of
the two covariance models can be chosen to contain $\Sigma^*$.
\end{lemma}

\begin{proof}
We first explain how an exceptional algebraic set of covariances gives a
measure-zero set of true parameters. We then apply this argument to
faithfulness and overlap.

\paragraph{Exceptional covariance sets and parameter values.}
Let $E\subsetneq V^*$ be an algebraic set strictly smaller than $V^*$.
There is a polynomial $f$ in the covariance entries that equals zero on
$E$ but not throughout $V^*$: choose a point of $V^*\setminus E$ and a
defining equation of $E$ that does not hold at that point.

Write $\Delta(\theta)=\det(I_p-B_{OO})$. As shown in the proof of
Proposition~\ref{prop:model-geometry}, each entry of $\phi^*(\theta)$ is
a ratio of polynomials in the free parameters, with denominator
$\Delta(\theta)^2$. Substituting these entries into $f$ and putting the
terms over a common denominator gives
\[
    f(\phi^*(\theta))=\frac{P(\theta)}{\Delta(\theta)^k}
\]
for a polynomial $P$ and a nonnegative integer $k$. The polynomial $P$
cannot be identically zero. Otherwise, $f$ would equal zero on every
generated covariance and therefore on its Zariski closure $V^*$,
contradicting the choice of $f$. Since $\Delta(\theta)\neq0$ for admissible
parameters,
\[
    \{\theta\in\mathcal P(G^{*+}):\phi^*(\theta)\in E\}
    \subseteq\{\theta\in\mathbb R^d:P(\theta)=0\}.
\]
The zero set of a nonzero polynomial has Lebesgue measure zero. This follows
by induction on the number of variables: a nonzero polynomial in one
variable has finitely many roots. In several variables, view the polynomial
as a polynomial in the last variable. At least one coefficient polynomial
is nonzero, so the set where all coefficient polynomials are zero has
measure zero by induction. Outside that set, the polynomial has finitely
many roots in the last variable; integrating over the other coordinates
gives the claim. Thus the displayed
set of exceptional parameter values has measure zero. This reasoning is
needed because taking the inverse image of a measure-zero set does not,
in general, preserve measure zero.

\medskip
\noindent\textit{Relevant-candidate algebraic faithfulness.}
Fix a candidate graph $G^+$ such that
\[
    V^*\not\subseteq V(G^+).
\]
The intersection $E=V^*\cap V(G^+)$ is an algebraic set strictly smaller
than $V^*$. The preceding argument shows that the true parameters whose
covariances belong to $E$ form a measure-zero set. Therefore, for this
candidate and outside that exceptional set,
\[
    \Sigma^*=\phi^*(\theta^*)\notin V(G^+).
\]
The candidate class is finite, so the union of these exceptional parameter
sets over all candidates with $V^*\not\subseteq V(G^+)$ still has measure
zero. Outside this union, any candidate whose Zariski closure contains
$\Sigma^*$ must therefore contain all of $V^*=V(G^\dagger)$:
\[
    \Sigma^*\in V(G^+)
    \quad\Longrightarrow\quad
    V(G^\dagger)\subseteq V(G^+)
\]
simultaneously for every candidate.  In particular, because
$\Mbar(G^+)\subseteq V(G^+)\cap\PD^p$
(Appendix~\ref{app:closure-properties}), it implies
Assumption~\ref{ass:algebraic-faithfulness}, whose premise is restricted to
candidates with $c(G^+)\leq c^\dagger$ and
$\Sigma^*\in\Mbar(G^+)$.

\medskip
\noindent\textit{Full-dimensional model overlap.}
Now fix a candidate with
\[
    V(G^+)=V^*=:V.
\]
Using the interior relative to $V$ defined in
Appendix~\ref{app:topology-notation}, a matrix $\Sigma\in S\subseteq V$
belongs to $\operatorname{int}_V S$ exactly when some $\varepsilon>0$
satisfies
\[
    \{\Sigma'\in V:\|\Sigma'-\Sigma\|_F<\varepsilon\}\subseteq S.
\]
The intersection of the ball with $V$, rather than the full ball in
$\Sym^p$, must be contained in $S$. Write
\[
    I_G:=\operatorname{int}_V\M(G^+),
    \qquad
    I_\dagger:=\operatorname{int}_V\M(G^\dagger)
\]
for the interiors of the two covariance families relative to $V$. Both covariance
families have dimension $\dim V$ by Proposition~\ref{prop:model-geometry}.

For a semialgebraic $S\subseteq V$, its boundary relative to $V$ is
$\cl_V S\setminus\operatorname{int}_V S$. A boundary point either lies
outside $S$ but is a limit of points in $S$, or lies in $S$ and is a limit
of points in $V\setminus S$. Hence
\begin{equation}
    \cl_V S\setminus\operatorname{int}_V S
    \subseteq
    (\cl_V S\setminus S)
    \cup
    \bigl(\cl_V(V\setminus S)\setminus(V\setminus S)\bigr)\feedbackedit{.}
    \label{eq:relative-boundary-decomposition}
\end{equation}
For any nonempty semialgebraic set $T$, the added limit points
$\cl_{\Sym^p}T\setminus T$ have dimension strictly smaller than $\dim T$
\citep[Proposition~3.16(2)]{coste2000}. Here $V$ is Euclidean closed, so
the closure of a subset of $V$ within $V$ agrees with its closure in the
full symmetric matrix space. Apply this result to $T=S$ and
$T=V\setminus S$; an empty set contributes no points. Since both sets
have dimension at most $\dim V$, the displayed boundary has dimension
strictly smaller than $\dim V$. These sets are semialgebraic by the
closure and set-operation properties in \citet[Section~2.1.2 and
Corollary~2.5]{coste2000}.

We also need a local-coordinate property of the irreducible real algebraic set $V$
\citep[Section~1, pp.~426--427]{akbulut1981}. Write $r=\dim V$. There is
an algebraic subset $E_V\subsetneq V$, with $\dim E_V<r$, such that near
each point of $V\setminus E_V$, the set $V$ can be described as follows:
$r$ selected covariance coordinates vary over a nonempty open subset of
$\mathbb R^r$, and the remaining coordinates are continuously
differentiable functions of those selected coordinates. The choice of
coordinates may depend on the point. Consequently, intersecting $V$ with
a sufficiently small open ball around such a point gives a semialgebraic
set of dimension $r$. The exclusion of
$E_V$ concerns this local description of the set $V$, not matrix
invertibility; membership in $E_V$ is not defined by a zero covariance
determinant.

Combining this property with the boundary argument shows that each of
the semialgebraic sets
\[
    \cl_V\M(G^+)\setminus I_G,
    \qquad
    \cl_V\M(G^\dagger)\setminus I_\dagger,
    \qquad
    E_V
\]
has dimension strictly smaller than $\dim V$. Denote their union by $B_G$.
A finite union has the maximum dimension of its members. Moreover,
\citet[Theorem~3.20]{coste2000} gives
\[
    \dim\overline{B_G}^{\,\mathrm{Zar}}=\dim B_G<\dim V.
\]
Since $B_G\subseteq V$ and $V$ is algebraic, this Zariski closure is an
algebraic subset strictly smaller than $V$.

Apply the exceptional-parameter argument at the start of the proof with
$E=\overline{B_G}^{\,\mathrm{Zar}}$. The admissible true parameters whose
covariances belong to $B_G$ form a measure-zero set. The union of these
sets over the finitely many candidates with $V(G^+)=V^*$ also has measure
zero. Take a true parameter outside this union and suppose the premises
of Assumption~\ref{ass:model-overlap} hold for $G^+$. Since
\[
    \Sigma^*\in\Mbar(G^+)\subseteq\cl_V\M(G^+)
\]
and $\Sigma^*\notin B_G$, we have $\Sigma^*\in I_G$. Also,
$\Sigma^*\in\M(G^\dagger)$ because $G^\dagger$ and the true graph have
the same covariance family. Exclusion from $B_G$ therefore gives
$\Sigma^*\in I_\dagger$ and $\Sigma^*\notin E_V$.

Consider the common set
\[
    U_{G^+}
    :=I_G\cap I_\dagger\cap(V\setminus E_V)\cap\PD^p.
\]
It contains $\Sigma^*$, is semialgebraic, and is open relative to $V$.
Indeed, $I_G$ and $I_\dagger$ are interiors relative to $V$,
$V\setminus E_V$ is open relative to $V$ because $E_V$ is Euclidean
closed, and $\PD^p$ is open in $\Sym^p$ by the strict inequalities in
Sylvester's criterion. The local-coordinate description at $\Sigma^*$
therefore supplies an open neighborhood of $\Sigma^*$ relative to $V$,
contained in $U_{G^+}$, whose projection contains a nonempty open subset
of $\mathbb R^r$. Our dimension definition gives
$\dim U_{G^+}\geq r$; the reverse inequality follows from
$U_{G^+}\subseteq V$. Thus
\[
    \dim U_{G^+}=\dim V.
\]
Since $U_{G^+}\subseteq\M(G^+)\cap\M(G^\dagger)$,
\[
    \dim(\M(G^+)\cap\M(G^\dagger))\geq\dim V.
\]
On the other hand, the intersection is contained in $\M(G^+)$, whose
dimension is $\dim V$ by Proposition~\ref{prop:model-geometry}.  The reverse
inequality follows, and hence
\[
    \dim(\M(G^+)\cap\M(G^\dagger))=\dim V.
\]
This is Assumption~\ref{ass:model-overlap}, and the construction shows that
the full-dimensional common subset may be chosen to contain $\Sigma^*$.
Finally, the union of the exceptional parameter sets for faithfulness and
overlap has measure zero, so both conclusions hold outside a single such
set.
\end{proof}

Lemma~\ref{lem:genericity} does not prove structural minimality.
Assumption~\ref{ass:structural-minimality} compares strict inclusion of
Zariski closures with the discrete complexity $c(G^+)$ and must be verified or
imposed for the chosen candidate class.  Consequently, if structural
minimality holds, then relevant-candidate algebraic faithfulness and
full-dimensional overlap hold for generic true parameters by
Lemma~\ref{lem:genericity}.  Lemma~\ref{lem:geometric-identification} then
implies that identification at the true covariance also holds generically,
and Corollary~\ref{cor:quasi-consistency} yields consistency up to marginal
quasi-equivalence.  This generic conclusion can fail at exceptional parameter
values: structural minimality does not by itself establish pointwise
identification at every covariance.

\subsection{Failure of structural minimality}
\label{app:minimality-limitation}

The following example shows why structural minimality must be checked for
the chosen complexity.  It does not contradict the conditional recovery
result, but rules out an unconditional quasi-equivalence guarantee over the
full candidate class.

Take $p=4$ and $L_{\max}\geq4$.  Let $G_\square^+$ have no observed edges
and four exogenous latent variables, with respective child sets
\[
    \{1,2\},\qquad\{2,3\},\qquad\{3,4\},\qquad\{4,1\}.
\]
Its complexity is $c(G_\square^+)=8+4=12$, and
\[
    V_\square:=V(G_\square^+)
    =\{\Sigma\in\Sym^4:\Sigma_{13}=\Sigma_{24}=0\},
    \qquad \dim V_\square=8.
\]
To verify the claimed Zariski closure, fix the latent variances to one and take
nonzero loadings on each permitted latent edge.  Each of the four allowed
off-diagonal covariances is the product of the two loadings of its own
latent, so these four covariances can be varied independently to first
order.  The four observed noise variances independently vary the diagonal
entries and compensate for changes in their latent contributions.  The
Jacobian of the covariance map, whose entries are partial derivatives of
covariance coordinates with respect to the free parameters, therefore has
rank eight at such an admissible parameter point with strictly positive
observed noise variances. Selecting eight parameter coordinates and keeping the
others fixed gives an invertible Jacobian for these eight covariance
entries. The inverse function theorem therefore shows that their projection
contains a nonempty open subset of $\mathbb R^8$.
Definition~\ref{def:semialgebraic-dimension} and containment in the displayed
eight-dimensional linear space therefore give covariance-family dimension
eight.  Proposition~\ref{prop:model-geometry} then implies that its Zariski
closure has dimension eight; Lemma~\ref{lem:strict-inclusion-dimension}
then shows that it is the entire linear space.

We next show that every candidate $G^+$ with $V(G^+)=V_\square$ has
complexity at least twelve.  For every candidate graph in $\Gclass$, an
observed covariance $\Sigma_{ij}$ is identically zero over its admissible
parameter space if and only if $i$ and $j$ have no common ancestor, with
each vertex counted as its own ancestor.  Indeed, a noise source can affect
an observed variable only if there is a directed path from the source to
that variable, so absence of a
common ancestor means that no structural noise contributes to both
variables, so their covariance is zero.
Conversely, if a common ancestor exists, choose all permitted coefficients
strictly positive and sufficiently small that the spectral radius of the
full coefficient matrix is less than one.  This choice supplies an
admissible parameter point; it imposes no additional assumption on the
candidate class.  The convergent series
$(I-B)^{-1}=\sum_{k=0}^{\infty}B^k$ then gives a strictly positive
coefficient from each noise source to every vertex reachable from it.
Positive noise variances make the common ancestor's contribution to
$\Sigma_{ij}$ strictly positive, and all other contributions
are nonnegative.  Hence $\Sigma_{ij}$ is not identically zero.  This argument
also applies to graphs containing observed directed cycles.

In the following argument, adjacency refers to the undirected cycle with
pairs $\{1,2\},\{2,3\},\{3,4\},\{4,1\}$. These are exactly the pairs
whose covariances are not identically zero. This cycle describes covariance
pairs, not directed edges among observed variables. The opposite pairs are $\{1,3\}$ and
$\{2,4\}$.  An observed edge $i\to j$ in $G^+$ must join adjacent vertices
of this cycle, since the opposite-pair covariances vanish identically.  Let $k$
be the other neighbor of $i$ in the four-cycle; it is opposite $j$.
Since $\Sigma_{ik}$ is not identically zero, $i$ and $k$ have a common
ancestor.  The edge $i\to j$ makes this source an ancestor of $j$ as well,
contradicting the identity $\Sigma_{jk}=0$.  Thus $G^+$ has no observed
edges.  Any two children of the same latent must consequently be adjacent
in the cycle, since their covariance is not identically zero.  Because no
three vertices of the cycle are pairwise adjacent and the
candidate class requires at least two children, each latent has exactly two
children.  Each of the four adjacent pairs needs its own latent source.
There are therefore at least four latents and eight latent-to-observed edges,
which proves $c(G^+)\geq12$.

In particular, when $G^{*+}=G_\square^+$, every distribution-equivalent
representative has complexity at least twelve, and $c^\dagger=12$.
A complete directed acyclic graph on the four observed vertices, denoted
$G_{\mathrm{sat}}^+$, has only six edges
and generates every positive-definite covariance through sequential linear
Gaussian regressions.  Thus
\[
    V_\square\subsetneq V(G_{\mathrm{sat}}^+)=\Sym^4,
    \qquad c(G_{\mathrm{sat}}^+)=6<c^\dagger,
\]
which violates Assumption~\ref{ass:structural-minimality}.

The score comparison is stronger than this failure of a sufficient
condition. Every graph quasi-equivalent to $G_\square^+$ also has
Zariski closure $V_\square$. To verify this, let
$U=\M(G^+)\cap\M(G_\square^+)$ for such a graph. Quasi-equivalence and
Proposition~\ref{prop:model-geometry} give
\[
    \dim U=\dim V(G^+)=\dim V_\square=8,
    \qquad U\subseteq V(G^+)\cap V_\square.
\]
If the two Zariski closures were different, their intersection would be
a nonempty proper algebraic subset of at least one of them. Both closures
are irreducible, so the proper-subset argument in the proof of
Lemma~\ref{lem:strict-inclusion-dimension} would give intersection
dimension strictly below eight. This contradicts its containment of $U$,
which has dimension eight. Therefore $V(G^+)=V_\square$, and the preceding
complexity bound gives $c(G^+)\geq12$. For any $S_n\succ0$ and
$\lambda_n>0$,
\[
    s_n(G_{\mathrm{sat}}^+)=C_n+6\lambda_n
    < C_n+12\lambda_n
    \leq s_n(G^+)
    \quad\text{whenever }G^+\cong_O G_\square^+,
\]
because $\delta_n(G_{\mathrm{sat}}^+)=0$ and $\delta_n(G^+)\geq0$ for every
graph. Thus the original structural penalty can exclude every
quasi-equivalent graph, even with exact global optimization.  The
quasi-equivalence guarantee therefore cannot extend to all graphs in the
candidate class without conditions excluding this obstruction.
A dimension-based penalty would avoid this particular strict-inclusion
obstruction, but is not the penalty analyzed here.

\subsection{Weaker and sufficient conditions for structural minimality}
\label{app:positive-minimality}

This subsection discusses a weaker condition at the true covariance and a
restricted class for which structural minimality can be verified. The
positive result assumes that the target covariance family admits a
latent-free representation without reciprocal edge pairs, meaning that
$i\to j$ and $j\to i$ never both occur. It does not establish
structural minimality for families with no latent-free representation.

\paragraph{A weaker condition at the true covariance.}
Structural minimality is not needed for
Theorem~\ref{thm:closure-consistency}; it is one sufficient route to the
identification condition used in Corollary~\ref{cor:quasi-consistency}.
The main-text assumption compares the target family with every candidate
graph. For identification at a fixed $\Sigma^*$, it suffices to make this
comparison only for candidates whose covariance-family closures contain
$\Sigma^*$. Specifically, Assumption~\ref{ass:structural-minimality} can be
replaced in Lemma~\ref{lem:geometric-identification} by
\[
    \Sigma^*\in\Mbar(G^+),\quad
    V(G^\dagger)\subsetneq V(G^+)
    \quad\Longrightarrow\quad c(G^+)>c^\dagger
    \qquad(G^+\in\Gclass).
\]
Indeed, every candidate considered in that proof already satisfies
$\Sigma^*\in\Mbar(G^+)$ and $c(G^+)\leq c^\dagger$.
Algebraic faithfulness gives inclusion of Zariski closures; the displayed condition
excludes strict inclusion, and full-dimensional overlap then gives
quasi-equivalence. This weaker condition imposes no requirement on candidates
whose covariance-family closures do not contain $\Sigma^*$.
Restricting the original condition only to $c(G^+)\leq c^\dagger$, however,
would not weaken it: every violation already has that complexity bound.
We retain the global version in the main text because it is a property of
the target family and competing graphs, independent of the realized
covariance. Neither version can be discarded in general without an
alternative identification condition, as Subsection~\ref{app:minimality-limitation}
demonstrates.

\paragraph{The role of the complexity penalty.}
Proposition~\ref{prop:model-geometry} establishes irreducibility and
identifies covariance-family dimensions with those of their Zariski
closures. Lemma~\ref{lem:strict-inclusion-dimension} then shows that strict
inclusion $V(G^\dagger)\subsetneq V(G^+)$ is equivalent to inclusion with
strictly larger covariance-family dimension. If complexity were instead defined by
$c_{\mathrm{dim}}(G^+)=\dim\M(G^+)$, the corresponding structural-minimality
condition would hold automatically. This is not the edge-plus-latent
complexity used in our score, and a standard BIC interpretation of a
dimension-based penalty would still require additional regularity.

{\color{black}
The simple mixed graphs of \citet{amendola2020} have at most one edge of
any type between each pair of observed vertices. Let $k$ count a graph's
directed edges and its bidirected edges, which permit correlations
between structural noises. Their Theorem~3.1 gives
covariance-family dimension $p+k$. With $p$ fixed, a larger dimension
therefore requires more edges. Consequently, the analogous structural-minimality
condition holds automatically if all candidates are simple mixed graphs
and complexity counts their edges. This is a consequence of their dimension
theorem, not a statistical recovery guarantee. In our explicit-latent class,
the edge-plus-latent count need not increase with covariance-family dimension,
so structural minimality requires a separate justification.
}

\begin{feedbackrevision}
\paragraph{Structural complexity and covariance dimension.}
\label{app:complexity-dimension}
Our penalty counts the edges and latent vertices in a graph; it does not
count the independently varying covariance entries. For example, consider
two graphs on two observed variables: one has a single latent parent of
both variables and no observed edge; the other has only the observed edge
$1\to2$. Both generate all of $\PD^2$, a family of dimension three, but
their structural complexities are three and one, respectively. The
construction is given at the end of this subsection.

This distinction is deliberate. The structural penalty expresses a
preference for fewer edges and latent vertices, and its additive form
allows the expected penalty under Bernoulli gates to be evaluated without
computing model dimensions for the sampled graphs. It also introduces a
limitation: greater covariance-family dimension need not imply greater
structural complexity, as Subsection~\ref{app:minimality-limitation}
demonstrates. Replacing this penalty by model dimension would change the
objective and its preference among graphs, rather than merely remove a
redundant parameterization.

\end{feedbackrevision}
\paragraph{\textcolor{black}{A dimension criterion.}}
\begin{feedbackrevision}
The original free parameter coordinates consist of one coefficient per
permitted edge and $p+\ell(G^+)$ noise variances, giving
$p+c(G^+)$ coordinates. A sharper dimension bound follows by absorbing
latent variances into the loadings. Set
$\Gamma=\Omega_L^{1/2}\Lambda$, where the diagonal entries of
$\Omega_L^{1/2}$ are the positive square roots of the latent variances.
Then $\Lambda^\top\Omega_L\Lambda=\Gamma^\top\Gamma$, and $\Gamma$
has the same graph-required zeros as $\Lambda$. Conversely, any such
$\Gamma$ is realized by $\Lambda=\Gamma$ and $\Omega_L=I_{\ell}$.
Thus the covariance family is also the image of
\[
    (B_{OO},\Gamma,\Omega_O)\longmapsto
    (I_p-B_{OO})^{-\top}
    (\Omega_O+\Gamma^\top\Gamma)(I_p-B_{OO})^{-1}.
\]
Its domain has $p+|E(G^+)|$ free coordinates: one per permitted edge and
one per observed noise variance. Positive observed variances and
$\det(I_p-B_{OO})\neq0$ define a nonempty Euclidean open semialgebraic
subset of this coordinate space. The displayed map is rational and
semialgebraic on that domain. The dimension of a semialgebraic image cannot
exceed the dimension of its domain \citep[Theorem~3.18]{coste2000}, so
\begin{equation}
    \dim\M(G^+)\leq p+|E(G^+)|
    =p+c(G^+)-\ell(G^+).
    \label{eq:normalized-dimension-bound}
\end{equation}
This is an upper bound, not a claim that the remaining parameters are
identifiable or that their number equals model dimension. The original
free coordinates remain valid for the geometry and genericity proofs.
\end{feedbackrevision}
Consequently, if the target family satisfies
$\dim\M(G^{*+})=p+c^\dagger$, any candidate of strictly larger model
dimension must have $c(G^+)>c^\dagger$. This verifies structural minimality
even without requiring inclusion of Zariski closures. The following class satisfies
this criterion. Its dimension calculation is the diagonal-noise directed
special case of the simple cyclic mixed-graph dimension result of
\citet[Theorem~3.1 and Lemma~3.3]{amendola2020}{\color{black}. We give a direct
proof and combine this dimension equality with
Equation~\ref{eq:normalized-dimension-bound} to verify structural minimality
against our full candidate class, including graphs with latent variables
or reciprocal observed edges. The competitors need not themselves admit
a latent-free representation.}

\begin{proposition}[A class with a latent-free representative]
\label{prop:positive-minimality}
Suppose $\M(G^{*+})=\M(G_{\mathrm{dir}}^+)$ for a latent-free graph
$G_{\mathrm{dir}}^+\in\Gclass$ with $e$ observed edges and at most one of
$i\to j$ and $j\to i$ for each distinct pair. Directed cycles are allowed.
Then
\[
    \dim\M(G^{*+})=p+e,\qquad c^\dagger=e,
\]
and Assumption~\ref{ass:structural-minimality} holds against every candidate
in $\Gclass$, including graphs with latent variables or reciprocal edges.
Under the sampling conditions of Theorem~\ref{thm:closure-consistency},
global score minimizers therefore recover marginal quasi-equivalence for
generic admissible parameters of $G^{*+}$.
\end{proposition}

\begin{proof}
We first determine the covariance-family dimension and then compare
structural complexities.

\paragraph{Dimension of the latent-free covariance family.}
List the observed edges of $G_{\mathrm{dir}}^+$ as
$i_1\to j_1,\ldots,i_e\to j_e$. Write their coefficients as
$b_k=(B_{OO})_{i_kj_k}$ and the observed noise variances as
$\omega_i=(\Omega_O)_{ii}$. The free parameter vector is
\[
    \theta=(\omega_1,\ldots,\omega_p,b_1,\ldots,b_e)
    \in\mathbb R^{p+e}.
\]
Let $U$ be its admissible domain, defined by $\omega_i>0$ for all $i$
and $\det(I_p-B_{OO})\neq0$. These strict conditions make $U$ open in
$\mathbb R^{p+e}$. The covariance map
\[
    \phi:U\longrightarrow\PD^p,\qquad
    \phi(\theta)=(I_p-B_{OO})^{-\top}\Omega_O(I_p-B_{OO})^{-1}
\]
is rational, continuously differentiable, and semialgebraic on $U$, with image
$\M(G_{\mathrm{dir}}^+)$. The parameter point
$\theta_0=(1,\ldots,1,0,\ldots,0)$ corresponds to
$(B_{OO},\Omega_O)=(0,I_p)$ and lies in $U$. This point is admissible
because an allowed edge coefficient may equal zero.

For a parameter direction $\dot\theta$, write
$D\phi(\theta_0)[\dot\theta]
:=\left.\frac{d}{dt}\phi(\theta_0+t\dot\theta)\right|_{t=0}$.
Differentiating the covariance map gives
\[
    D\phi(\theta_0)[\dot\theta]
    =\dot\Omega+\dot B^\top+\dot B,
\]
where $\dot B$ and $\dot\Omega$ denote changes in the edge coefficients
and diagonal noise variances, respectively. Indeed, the derivative of
$(I_p-t\dot B)^{-1}$ at $t=0$ is $\dot B$, and the product rule gives
the three displayed terms. At $\theta_0$, the partial derivative with
respect to $\omega_i$ is therefore $E_{ii}$, while that with respect to $b_k$ is
$E_{i_kj_k}+E_{j_ki_k}$; here $E_{ij}$ has one unit entry in position
$(i,j)$ and zeros elsewhere.

To connect this derivative calculation to our definition of dimension,
select the following covariance coordinates:
\[
    \pi(\Sigma)
    =(\Sigma_{11},\ldots,\Sigma_{pp},
      \Sigma_{i_1j_1},\ldots,\Sigma_{i_ej_e})
    \in\mathbb R^{p+e}.
\]
Symmetry allows each off-diagonal entry to be identified with its
upper-triangular coordinate. Since no reciprocal edge pair is present,
these $e$ off-diagonal coordinates are distinct. Thus $\pi$ is a coordinate
projection up to a reordering of its output coordinates, which preserves
open sets. Define $F=\pi\circ\phi$, which maps $U$ into
$\mathbb R^{p+e}$. In the parameter order specified above, its Jacobian is
\[
    DF(\theta_0)=I_{p+e}.
\]
Thus the inverse function theorem applies: there are open neighborhoods
$U_0\subseteq U$ of $\theta_0$ and $W_0\subseteq\mathbb R^{p+e}$ of
$F(\theta_0)$ such that $F$ maps $U_0$ one-to-one onto $W_0$, with a
continuously differentiable inverse. In particular,
\[
    W_0=F(U_0)\subseteq\pi\bigl(\M(G_{\mathrm{dir}}^+)\bigr).
\]
The projected covariance family therefore contains a nonempty open ball
in $\mathbb R^{p+e}$. Definition~\ref{def:semialgebraic-dimension} gives
$\dim\M(G_{\mathrm{dir}}^+)\geq p+e$. Conversely, the semialgebraic-image
dimension bound stated above gives
\[
    \dim\M(G_{\mathrm{dir}}^+)=\dim\phi(U)
    \leq\dim U=p+e.
\]
Combining the two inequalities, and using the assumed equality of
covariance families, yields
\[
    \dim\M(G^{*+})=\dim\M(G_{\mathrm{dir}}^+)=p+e.
\]
The Jacobian calculation thus establishes the dimension using the existing
definition; it does not introduce a different definition or assume that
the Jacobian has the same rank at every parameter value.

\paragraph{Minimum complexity and structural minimality.}
For any distribution-equivalent candidate $G^+$, the universal dimension
bound gives
\[
    p+e=\dim\M(G^+)\leq \feedbackedit{p+c(G^+)-\ell(G^+)}.
\]
Thus $\feedbackedit{c(G^+)\geq e+\ell(G^+)\geq e}$, while
$G_{\mathrm{dir}}^+$ itself has complexity $e$.
It follows that $c^\dagger=e$. For every candidate of strictly larger
dimension,
\[
    \feedbackedit{p+c(G^+)-\ell(G^+)}\geq\dim\M(G^+)
    >\dim\M(G^{*+})=p+e,
\]
so $c(G^+)>e=c^\dagger$. This proves global structural minimality without
restricting the competing graphs. Lemmas~\ref{lem:genericity} and
\ref{lem:geometric-identification} then establish the identification
condition for generic true parameters, and
Corollary~\ref{cor:quasi-consistency} gives the final assertion.
\end{proof}

\paragraph{A cyclic example with covariance-family dimension eight.}
Take four observed vertices, no latents, and the directed cycle
\[
    1\longrightarrow2\longrightarrow3\longrightarrow4\longrightarrow1.
\]
Proposition~\ref{prop:positive-minimality} gives $c^\dagger=4$ and
$\dim\M(G^{*+})=8<10=\dim\Sym^4$. Its covariance family therefore has
smaller dimension than the space of all positive-definite covariances.
Any strictly larger Zariski closure must have complexity at least five,
even when latent competitors are allowed. For generic true parameters,
global score minimizers recover marginal quasi-equivalence. This conclusion
does not identify a unique orientation or give a guarantee for a local
optimizer.

\paragraph{A latent example with an observed representative.}
A single exogenous latent with exactly two observed children and no observed
edge generates every positive-definite covariance on that pair. To see this,
write the desired covariance as
$\bigl[\begin{smallmatrix}a&b\\b&d\end{smallmatrix}\bigr]\succ0$.
Positive definiteness gives $d>0$ and $ad>b^2$, so $b^2/d<a$.
For $b\neq0$, choose $u^2\in(b^2/d,a)$ and set $v=b/u$.
Unit latent variance, loadings $(u,v)$, and positive observed noise variances
$(a-u^2,d-v^2)$ give the desired covariance. For $b=0$, set both loadings
to zero, which is admissible because permitted coefficients may be zero,
and use observed noise variances $(a,d)$.
\begin{feedbackrevision}
The latent-free graph $1\to2$ generates the same covariance using edge
coefficient $b/a$ and noise variances $a$ and $d-b^2/a$, which are
positive. Both families therefore equal $\PD^2$ and have dimension three.
The latent graph has five original parameter coordinates and four after
absorbing its latent variance, so even the normalized parameter count
exceeds its model dimension. The proposition applies with $c^\dagger=1$,
although the latent graph has complexity three.
\end{feedbackrevision}
Disjoint unions of such pairs, isolated vertices, and the
directed graphs in the proposition also admit representatives of the same
type. Thus the data-generating graph may contain both directed cycles and
latent variables in separate components. The conclusion does not establish
structural minimality for families that require latent variables in every
representation, or identify the number of latent variables in the
data-generating graph.

\section{\textcolor{black}{Practical optimization}}
\label{sec:concrete}

{\color{black}
This appendix describes an optional approximation for numerical optimization
of the exact Bernoulli objective in Equation~\ref{eq:exact-bernoulli-objective}.
For latent-variable models, we use the Binary Concrete construction of \citet{maddison2017} to obtain a differentiable surrogate for the discrete gates.
This approximation is not required for the differentiability of the exact
expectation or for Proposition~\ref{prop:bernoulli-exactness}.
}
Let
$\sigma(t)=(1+\exp(-t))^{-1}$ and parameterize
$q_{ij}=\sigma(\alpha_{ij})$, $r_{hj}=\sigma(\beta_{hj})$, and
$\pi_h=\sigma(\gamma_h)$.  If $G=\log U-\log(1-U)$ with
$U\sim\operatorname{Uniform}(0,1)$, the corresponding Binary Concrete gate is
\begin{equation}
    \widetilde K=\sigma\!\left(\frac{a+G}{\tau}\right),
    \label{eq:binary-concrete-gate}
\end{equation}
where $a$ is the relevant \emph{logit}: a probability $q\in(0,1)$ is
represented by $a=\log(q/(1-q))$. The \emph{temperature} $\tau>0$
controls how closely the continuous gate approximates a binary gate.
Applying
Equation~\ref{eq:binary-concrete-gate} to each nonconstant gate, with independent
noise draws and $\widetilde M_{ii}=0$, gives
$\widetilde M,\widetilde H,\widetilde Z$. These gates are differentiable
functions of the logits for each fixed noise draw; the noise distribution
does not depend on the parameters. For each realization with invertible
$I_p-\widetilde M\od W$, we can therefore differentiate the sampled
likelihood while holding the noise draws fixed. The approximate objective
at temperature $\tau$ is
\begin{equation}
\begin{split}
    F_{n,\tau}={}&\mathbb{E}\bigl[
    L_n(\widetilde M,\widetilde H,\widetilde Z;\vartheta)\bigr]\\
    &+\lambda_n\!\left[\sum_{i\neq j}q_{ij}
    +\sum_h\pi_h\left(1+\sum_jr_{hj}\right)\right].
\end{split}
\label{eq:concrete-objective}
\end{equation}

{\color{black}
\paragraph{Invertibility of binary and continuous masks.}
Invertibility for every binary mask does not imply invertibility for
every continuous mask. For example, with two observed variables and
\[
    W=\begin{pmatrix}0&2\\2&0\end{pmatrix},
    \qquad \det(I_2-M\od W)=1-4M_{12}M_{21},
\]
every binary mask gives determinant one or minus three. In contrast,
the continuous mask with $\widetilde M_{12}=\widetilde M_{21}=1/2$
gives $\det(I_2-\widetilde M\od W)=0$. Binary Concrete therefore does
not guarantee invertibility for all continuous masks.
\par}

Using sampled derivatives as an unbiased gradient estimator for
$F_{n,\tau}$ additionally requires justification for interchanging
differentiation and expectation, for example through an integrable bound on
the derivatives that is uniform over a parameter neighborhood.
Invertibility of $I_p-\widetilde M\od W$ with probability one is not
sufficient: the absolute derivatives can have infinite expectation when
this matrix approaches a noninvertible matrix. Thus differentiability for
individual noise realizations does not by itself justify differentiating
the expectation.

We deliberately retain the expected \emph{discrete} complexity in closed
form: at positive temperature, $\mathbb{E}[\widetilde M_{ij}]\neq q_{ij}$ in
general. Replacing the closed-form penalty by a sum of sampled continuous
gate values would therefore change the complexity penalty and introduce
additional sampling variation. For numerical optimization, a positive
lower bound on each observed-noise variance can be imposed through
$(\Omega_O)_{jj}=\varepsilon+\operatorname{softplus}(d_j)$ with
$\varepsilon>0$, where $\operatorname{softplus}(t)=\log(1+e^t)$ and
$d_j\in\mathbb R$ is optimized. To assess proximity to a noninvertible
matrix, one can monitor the smallest singular value
$\sigma_{\min}(I_p-\widetilde M\od W)$, where
$\sigma_{\min}(A)=\min_{\|x\|_2=1}\|Ax\|_2$. Monitoring alone does not
justify interchanging differentiation and expectation. The optional restriction
$\|\,|W|\,\|_2\leq1-\varepsilon_A$, with $0<\varepsilon_A<1$, gives
$\|\widetilde M\od W\|_2\leq\|\,|W|\,\|_2$ and hence
$\sigma_{\min}(I_p-\widetilde M\od W)\geq\varepsilon_A$
for every binary mask or continuous mask with entries in $[0,1]$.
Together with the lower bound on noise variances and bounds on the remaining
parameters in a neighborhood of their current values, this supplies
derivative bounds on that neighborhood at fixed
$\tau>0$.  Both restrictions are stronger than the statistical model and can
exclude admissible covariances or well-defined cyclic systems.

\paragraph{Numerical implementation used in the experiments.}
The implementation in Section~\ref{sec:experiments} differs between the observed-only and latent-variable settings. 
When no latent variables are present, the latent activation and latent-edge
gates disappear, and we optimize the exact Bernoulli formulation in
Equation~\ref{eq:exact-bernoulli-objective} using 256 Monte Carlo samples of
hard observed-edge masks per update. Gradients of the observed-edge logits
are estimated using a leave-one-out score-function estimator. We additionally use a discrete compound support-refinement step: when the score plateaus, deletion and reversal-plus-deletion candidates are evaluated, accepted deletions are permanently removed from the search space, and optimization continues from the best improving support.

When latent variables are present, we use a straight-through Binary Concrete estimator. The forward likelihood is evaluated using hard binary gates, while the backward pass differentiates through their soft Binary Concrete counterparts. Thus the experimental implementation does not evaluate the forward likelihood on fractional graph masks. The value of $\tau$ is also fixed throughout training. This estimator is a numerical approximation and should be distinguished from Proposition~\ref{prop:bernoulli-exactness}, which concerns the exact Bernoulli expectation.

\paragraph{Discrete support selection and refitting.}
After either optimization procedure, the learned gate probabilities and continuous parameters are converted to a binary support $(M,H,Z)$. Inactive latent slots and active slots with fewer than two selected children are removed using the conversion in Appendix~\ref{app:bernoulli-details}. For each resulting fixed graph, we refit $(W,\Gamma,\Omega_O)$ using the original Gaussian negative log-likelihood. To target the discrete score in Equation~\ref{eq:graph-score}, this refit optimizes over positive diagonal $\Omega_O$ and invertible $A(M)=I_p-M\od W$ without imposing the optional norm restriction on $W$. Candidate graphs are then compared using the refitted discrete penalized score. Exact thresholds, initialization, optimization budgets, and refitting settings are reported in Appendix~\ref{app:experimental-details}.

These numerical procedures do not change the scope of the theoretical guarantees. Theorem~\ref{thm:closure-consistency} applies when the resulting graph is a global minimizer of the discrete score, and Corollary~\ref{cor:quasi-consistency} additionally requires Assumption~\ref{ass:pointwise-identification}. Guarantees for Monte Carlo optimization, straight-through gradients, discrete support refinement, and finite numerical refitting require additional analysis.

\clearpage

\section{Additional experimental details and sensitivity analyses}
\label{app:experimental-details}

\subsection{Synthetic data generation}
\label{app:synthetic-details}

Synthetic observations follow the linear Gaussian SCM in Section~\ref{sec:problem-setting} with $n=10{,}000$ training samples per trial. We use the observed-graph generation procedure from the public implementation of \citet{ghassami2020}. The parameter called maximum degree is a hard bound on total observed degree, i.e., observed in-degree plus observed out-degree; self-loops are excluded. 

For every selected observed edge, the coefficient sign is sampled uniformly and its magnitude is sampled uniformly from $[0.2,0.8]$. Coefficient draws are accepted only when the observed coefficient matrix satisfies spectral radius below $0.99$ and $\operatorname{cond}(I-B_{OO})\le20$. Observed noise variances are drawn independently from $\operatorname{Uniform}[1,3]$. Each exogenous latent variable independently selects an integer number of distinct observed children uniformly from $\{2,3,4\}$ and then selects that many observed nodes uniformly without replacement; child sets of different latent variables may overlap. Nonzero latent loadings use the same random-sign and $[0.2,0.8]$ magnitude distribution as observed coefficients. Latent variances are normalized to one, with their scale absorbed into the latent loading rows.

Latent percentages specify the nominal ratio $\ell/p$. We evaluate $\ell/p\in\{0.1,0.2,0.3\}$ and resolve the integer count as $\lfloor p(\ell/p)+0.5\rfloor$, with at least one latent whenever the requested ratio is positive. The estimator receives $L_{\max}=2\lfloor p(\ell/p)+0.5\rfloor+1$ latent slots for latent experiments and $L_{\max}=0$ in the no-latent experiments. We uses 10 trials with different seeds.

\subsection{Optimization and evaluation settings}
\label{app:optimization-details}

The numerical procedures for the exact-Bernoulli and straight-through
Binary Concrete searches are described in Appendix~\ref{sec:concrete}.
Table~\ref{tab:experimental-settings} summarizes the principal
implementation, graph-extraction, refitting, and evaluation settings used
in the experiments. During optimization, 
Adam is run for $1000$, $1500$, $3000$, $4000$, $6000$, $8000$, $10000$,
and $10000$ updates for $p=4,8,16,32,64,128,256$, and $512$, respectively.
For the GNW experiments with $p=100$, we use $10000$ updates.

\begin{table}[h!]
    \centering
    \caption{Principal numerical settings used in the experiments.}
    \label{tab:experimental-settings}
    \setlength{\tabcolsep}{5pt}
    \renewcommand{\arraystretch}{1.10}
    \begin{tabular}{@{}ll@{}}
        \toprule
        Setting & Value \\
        \midrule
        \multicolumn{2}{@{}l}{\textit{Structure optimization}} \\
        Mask samples per update & $256$ \\
        Adam learning rate & $10^{-2}$ \\
        Binary Concrete temperature & $\tau=1$ (fixed) \\
        Structure-search initializations & $1$ per trial \\
        Initial observed-edge probability & $0.01$ \\
        Initial latent-edge probability & $0.5$ \\
        Initial latent activation probability & $0.99$ \\
        Initial coefficient scale & $0.5$ \\
        Candidate latent bound & $L_{\max}=2\lfloor p(\ell/p)+0.5\rfloor+1$ \\
        \addlinespace
        \multicolumn{2}{@{}l}{\textit{Graph extraction and refitting}} \\
        Gate probability threshold & $0.5$ \\
        Coefficient threshold & $10^{-4}$ \\
        \addlinespace
        \multicolumn{2}{@{}l}{\textit{Compatibility evaluation}} \\
        Covariance targets per direction & $K_{\mathrm{eval}}=20$ \\
        Compatibility-fit starts & $5$ \\
        Compatibility maximum iterations & $1000$ \\
        Compatibility tolerance & $10^{-8}$ \\
        \addlinespace
        \multicolumn{2}{@{}l}{\textit{Observed-only support refinement}} \\
        Plateau window & $50$ updates \\
        Relative score tolerance & $10^{-4}$ \\
        Edges deleted per branch & $1$ \\
        Deletion branches & $10$ \\
        Reversal-plus-deletion branches & $10$ \\
        \bottomrule
    \end{tabular}
\end{table}

After graph extraction, inactive latent slots and active slots with fewer
than two retained children are removed. Continuous parameters are then
refit on each fixed support using the unrelaxed Gaussian likelihood before
the discrete score and covariance-family compatibility are evaluated.

\subsection{Bidirectional compatibility evaluation}
\label{app:compatibility-evaluation}

For evaluation we use the symmetric Gaussian divergence
\[
J(S,\Sigma)
=
\frac12\left\{
D(S\Vert\Sigma)+D(\Sigma\Vert S)
\right\},
\]
where $D$ is defined in Equation~\ref{eq:gaussian-kl}. Let $G^{*+}$ denote
the reference graph and $\widehat G^+$ the recovered graph. For generic
admissible parameter values
$\theta_1^*,\ldots,\theta_K^*\in\mathcal P(G^{*+})$, define
$\Sigma_k^*=\Sigma_{O}(\theta_k^*;G^{*+})$ and
\begin{equation}
d_k^{*\to\widehat G}
:=
\inf_{\widehat\theta\in\mathcal P(\widehat G^+)}
J\!\left(
\Sigma_k^*,
\Sigma_{O}(\widehat\theta;\widehat G^+)
\right).
\end{equation}
The forward discrepancy is
\begin{equation}
d_{\mathrm{fwd}}
:=
\max_{1\leq k\leq K}
d_k^{*\to\widehat G}.
\end{equation}

For the reverse direction, choose generic admissible parameter values
$\widetilde\theta_1,\ldots,\widetilde\theta_K
\in\mathcal P(\widehat G^+)$ and let
$\widetilde\Sigma_k=
\Sigma_{O}(\widetilde\theta_k;\widehat G^+)$. Define
\begin{equation}
d_k^{\widehat G\to *}
:=
\inf_{\theta\in\mathcal P(G^{*+})}
J\!\left(
\widetilde\Sigma_k,
\Sigma_{O}(\theta;G^{*+})
\right),
\end{equation}
and
\begin{equation}
d_{\mathrm{rev}}
:=
\max_{1\leq k\leq K}
d_k^{\widehat G\to *}.
\end{equation}

We use $K=20$ in the reported experiments and summarize bidirectional
compatibility by $d_{\mathrm{fwd}}+d_{\mathrm{rev}}$. The continuous
parameters of the fitted graph are reoptimized independently for each
target covariance while its graph support remains fixed.

\subsection{Baseline implementations and runtime protocol}
\label{app:baseline-details}

Ghassami et al.~\citep{ghassami2020} is evaluated using the public dglearn implementation with its original configuration: hill-climbing search, one initialization, at most 1000 iterations, maximum cycle/SCC parameter five, virtual refinement, and support reduction. DCCD-CONF \citep{sethuraman2025} is evaluated using its public implementation with its provided hyperparameters: 500 epochs, batch size 512, learning rate $10^{-2}$, $\lambda_c=\rho=10^{-2}$, Lipschitz constant $0.9$, linear activation, adjacency threshold $0.8$, and confounding-covariance threshold $0.01$. These baseline configurations are default values used in their corresponding repository. For DCCD-CONF, the observational data are supplied as a single regime with no intervention targets so that no additional interventional information is provided. The~\citet{amendola2020} baseline follows the greedy cyclic mixed-graph search with add, remove, and edge-reversal moves; to keep its neighborhood search computationally feasible, at most 64 candidate branches are evaluated at each step. All methods evaluated at a shared setting receive identical training data and trial seeds.

All experiments were run on compute nodes equipped with AMD EPYC 9334 32-Core processors and 8\,GB of memory. DCCD-CONF and Am\'endola et al.\ were evaluated up to $p=64$, while dglearn was evaluated up to $p=128$, beyond which their computation became infeasible for the experiments.

\subsection{GeneNetWeaver details}
\label{app:gnw-details}

We construct fixed GNW reference networks by first generating a weakly connected directed acyclic graph over the observed variables with a bounded total degree. Directed cycles are then introduced by adding back edges while preserving the prescribed degree constraint. Exogenous latent variables are subsequently added before simulation. Each latent variable has no parents and is connected to a randomly selected subset of observed variables, allowing different latent variables to share observed children.

Data are generated using the GNW multifactorial simulator with all perturbation inputs set to zero, so that no observed variable is directly intervened upon. Each fixed network yields 20,000 observations. In each trial, 16,000 observations are used for structure learning and 4,000 are held out. The network topology and simulated dataset remain fixed across trials, while the train--test partition, stochastic optimization, mask sampling, and compatibility-evaluation draws vary. 

\subsection{Sensitivity to graph density and latent confounding}
\label{sec:experiment-stress}

We stress test recovery over all nine combinations of maximum observed degree
$\{1,2,4\}$ and latent ratio $\ell/p\in\{0.1,0.2,0.3\}$. These experiments
vary the structure of the data-generating graphs and therefore evaluate
sensitivity to graph density and latent confounding.

\begin{figure*}[t]
    \centering

    \begin{subfigure}[t]{0.32\linewidth}
        \centering
        \includegraphics[width=\linewidth]{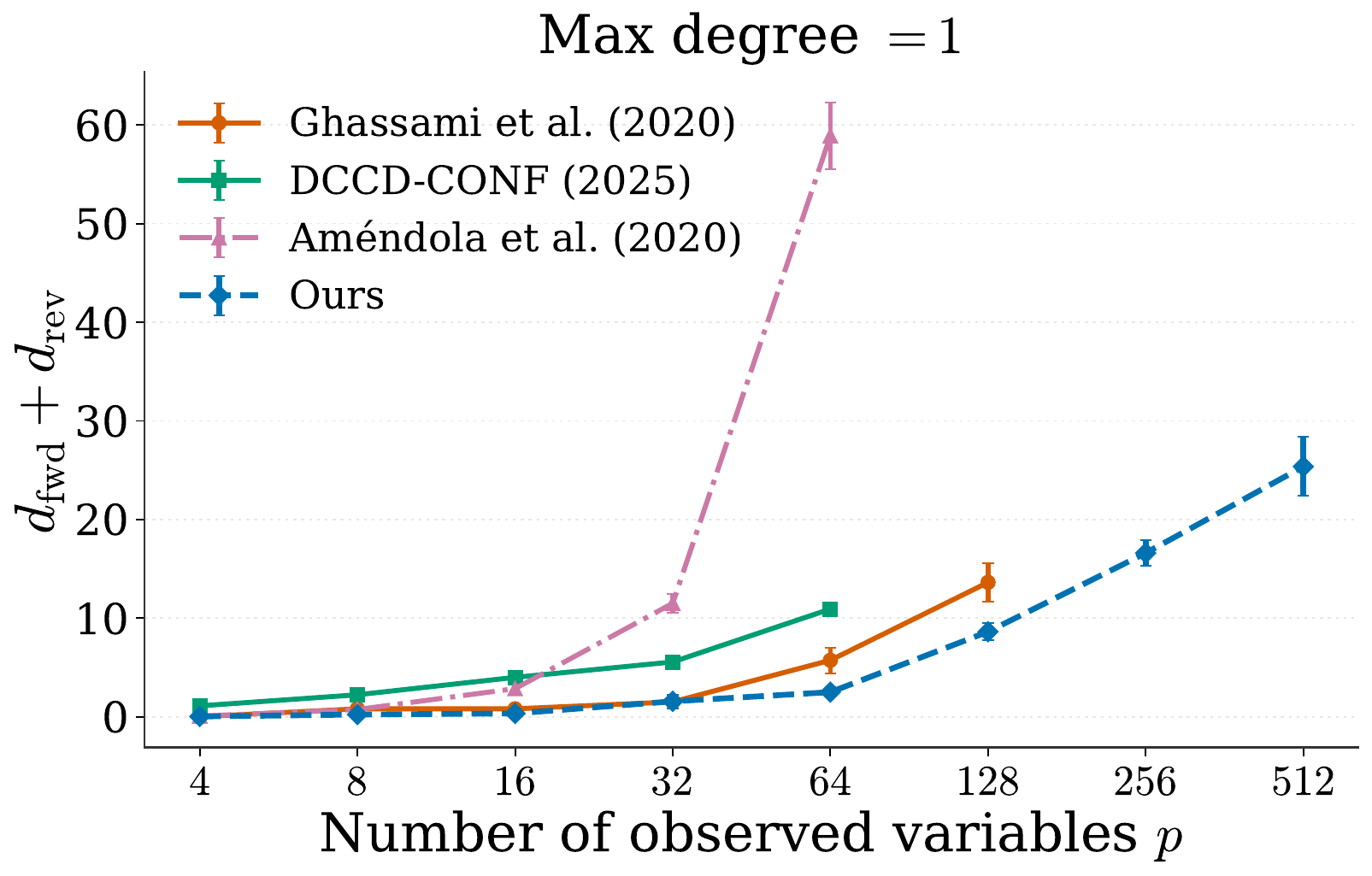}
        \caption{$10\%$ latent, degree 1}
    \end{subfigure}
    \hfill
    \begin{subfigure}[t]{0.32\linewidth}
        \centering
        \includegraphics[width=\linewidth]{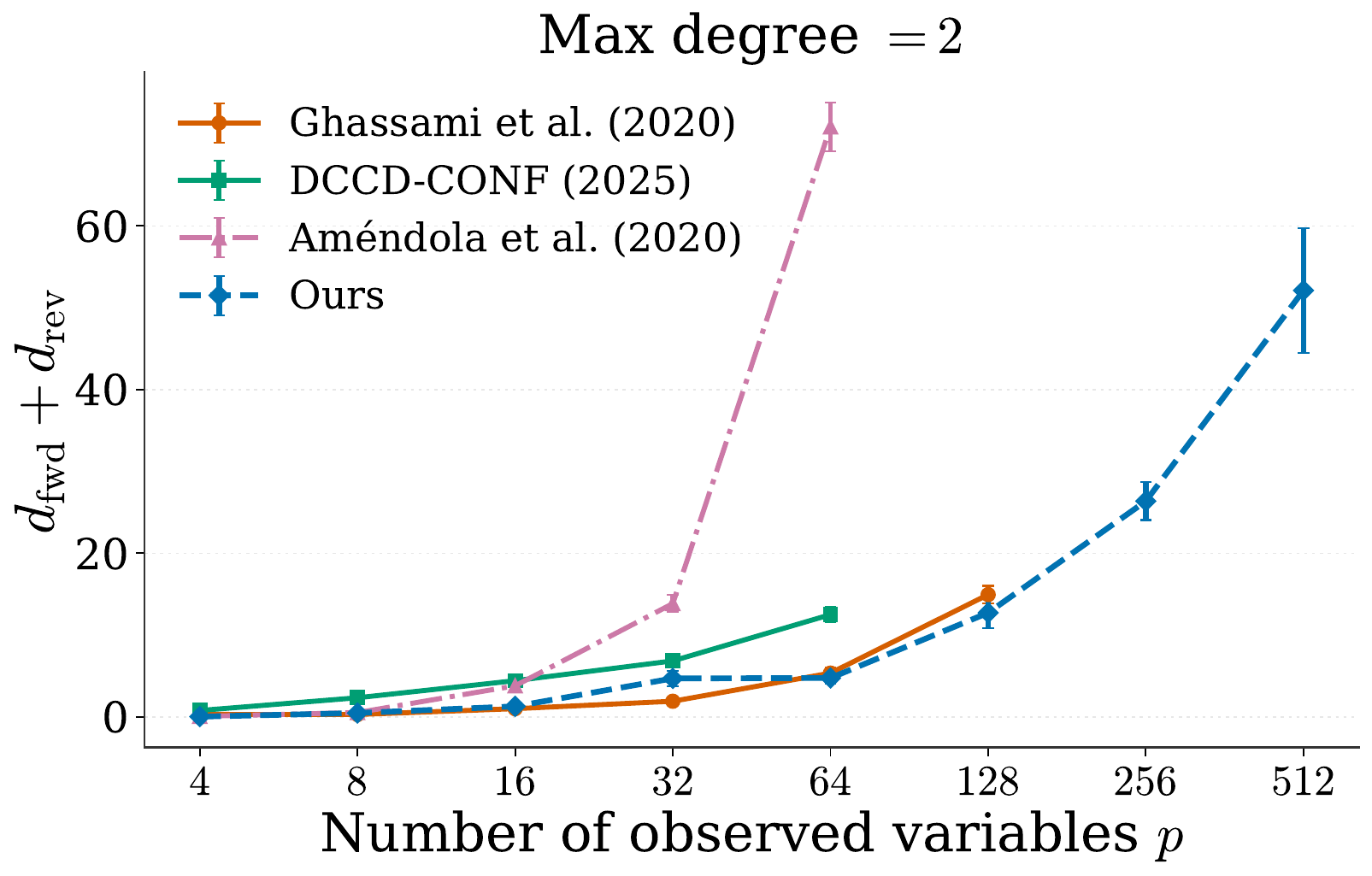}
        \caption{$10\%$ latent, degree 2}
    \end{subfigure}
    \hfill
    \begin{subfigure}[t]{0.32\linewidth}
        \centering
        \includegraphics[width=\linewidth]{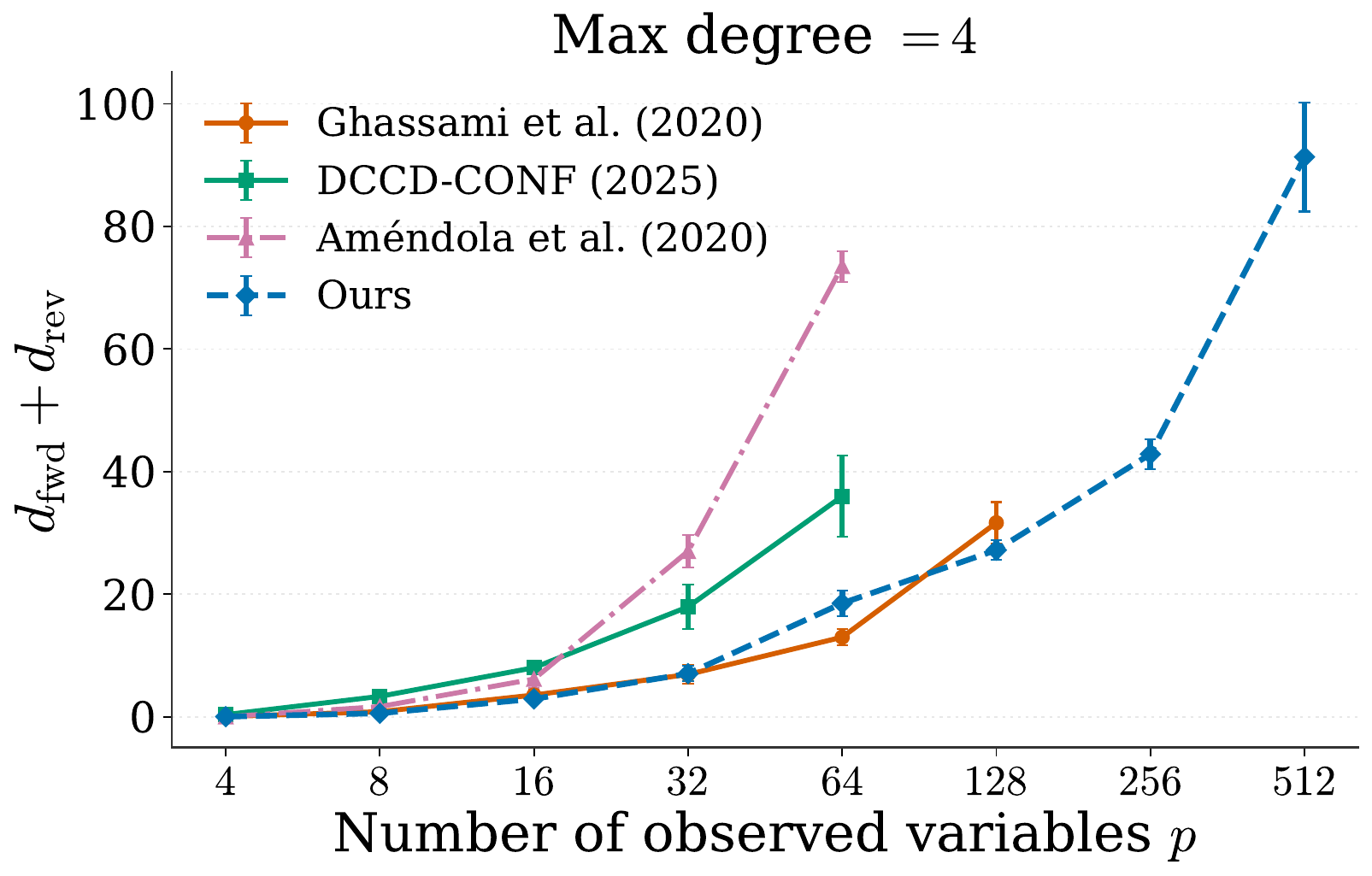}
        \caption{$10\%$ latent, degree 4}
    \end{subfigure}

    \vspace{1mm}

    \begin{subfigure}[t]{0.32\linewidth}
        \centering
        \includegraphics[width=\linewidth]{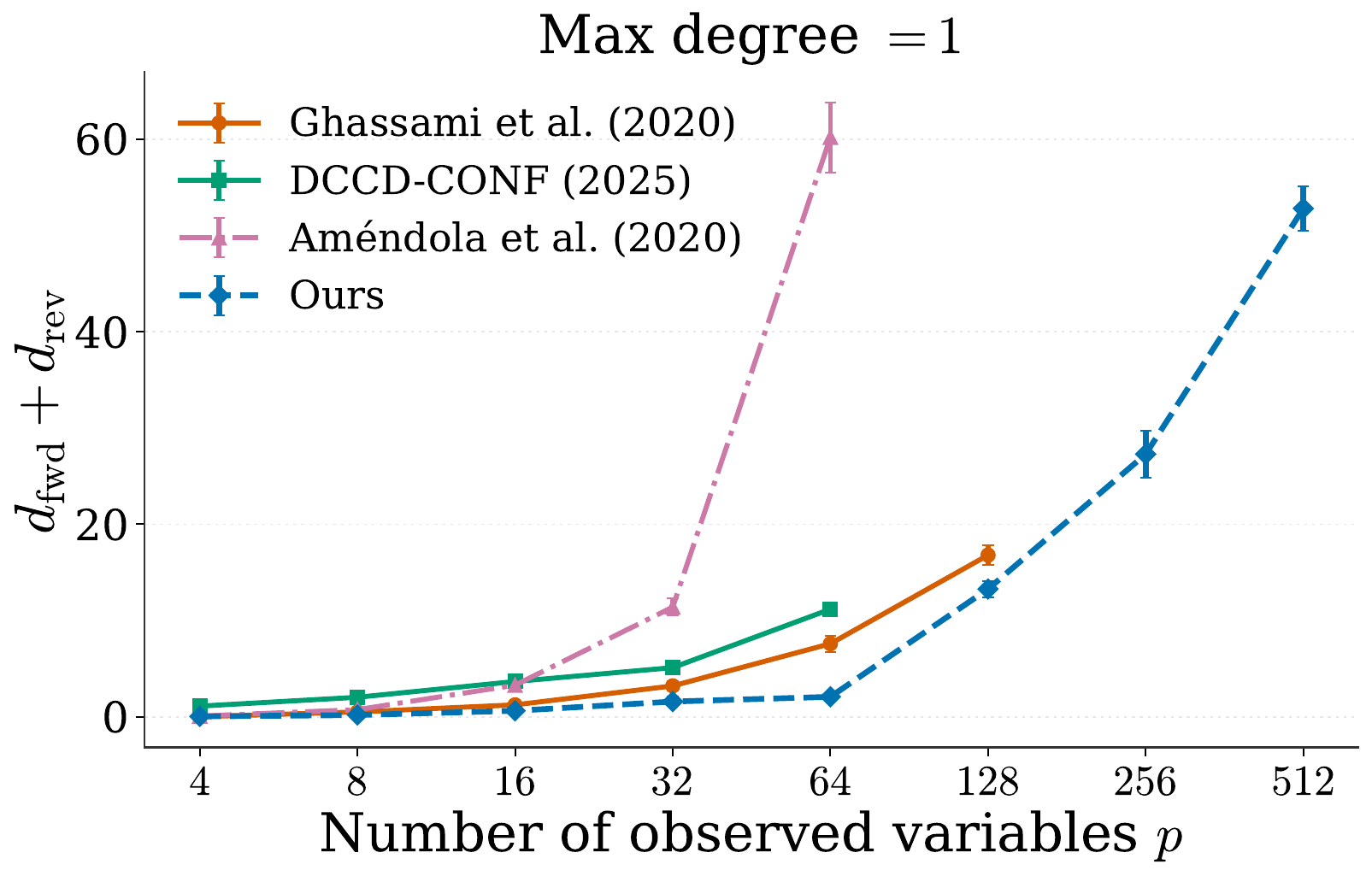}
        \caption{$20\%$ latent, degree 1}
    \end{subfigure}
    \hfill
    \begin{subfigure}[t]{0.32\linewidth}
        \centering
        \includegraphics[width=\linewidth]{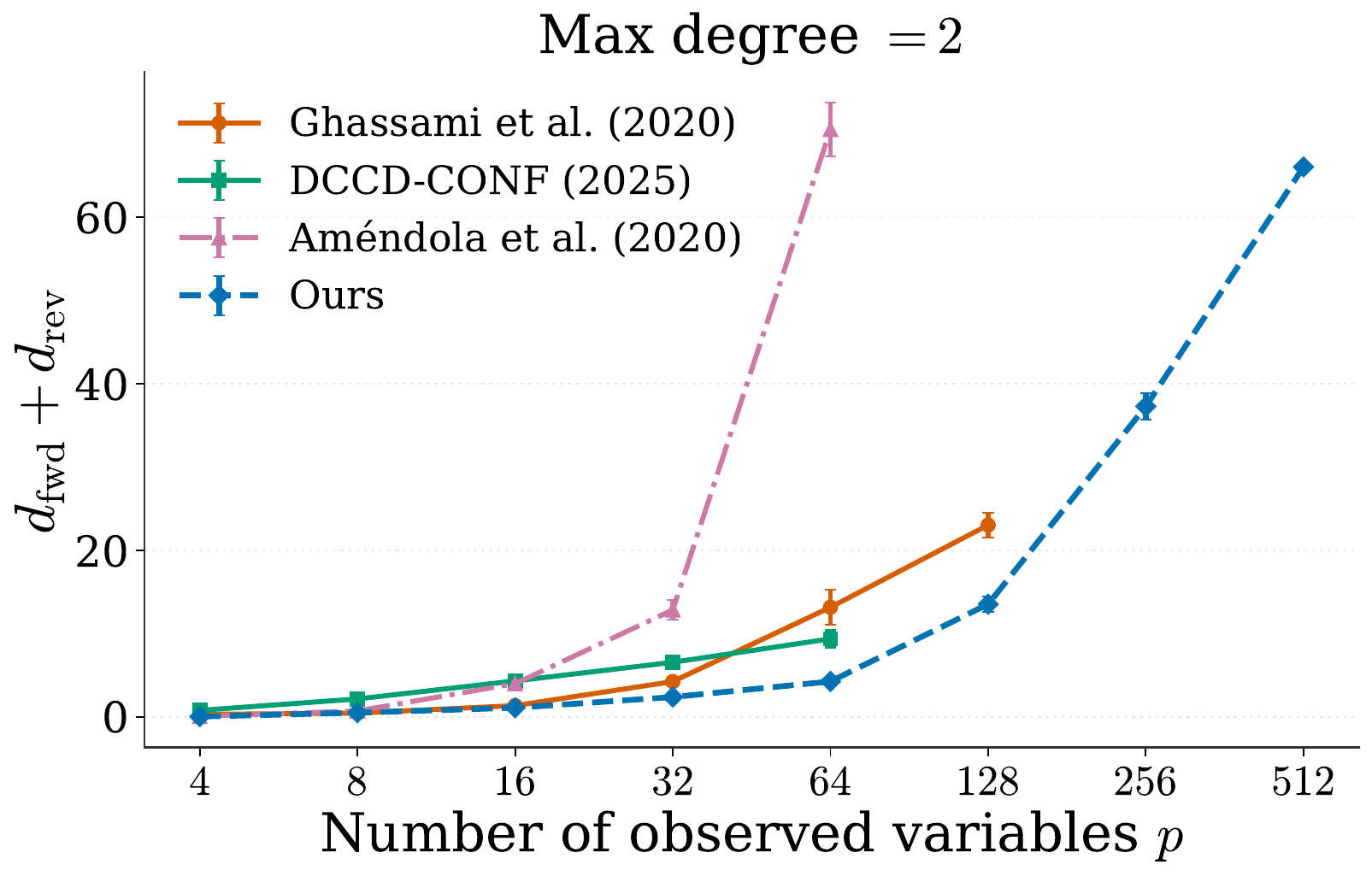}
        \caption{$20\%$ latent, degree 2}
    \end{subfigure}
    \hfill
    \begin{subfigure}[t]{0.32\linewidth}
        \centering
        \includegraphics[width=\linewidth]{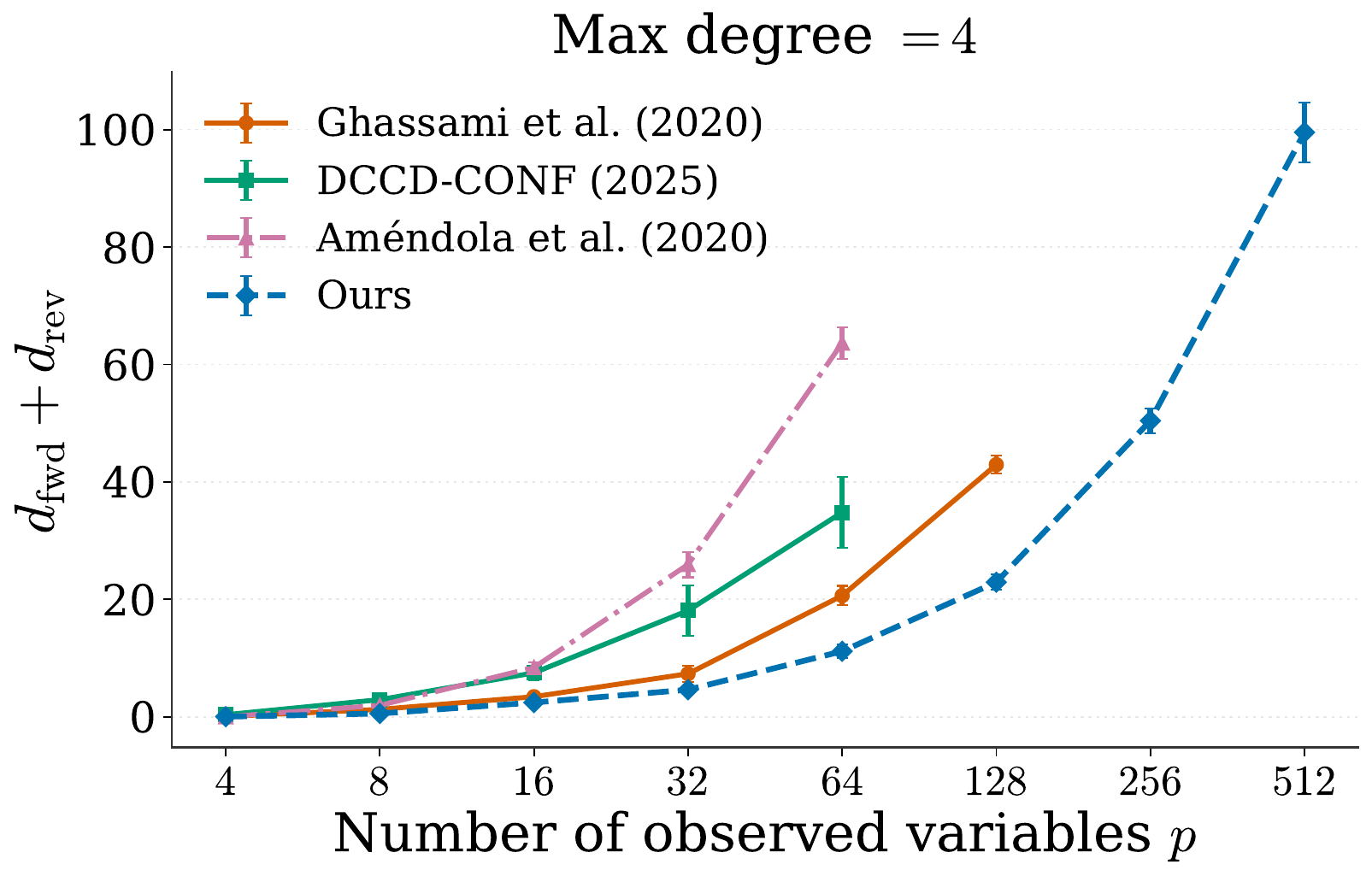}
        \caption{$20\%$ latent, degree 4}
    \end{subfigure}

    \vspace{1mm}

    \begin{subfigure}[t]{0.32\linewidth}
        \centering
        \includegraphics[width=\linewidth]{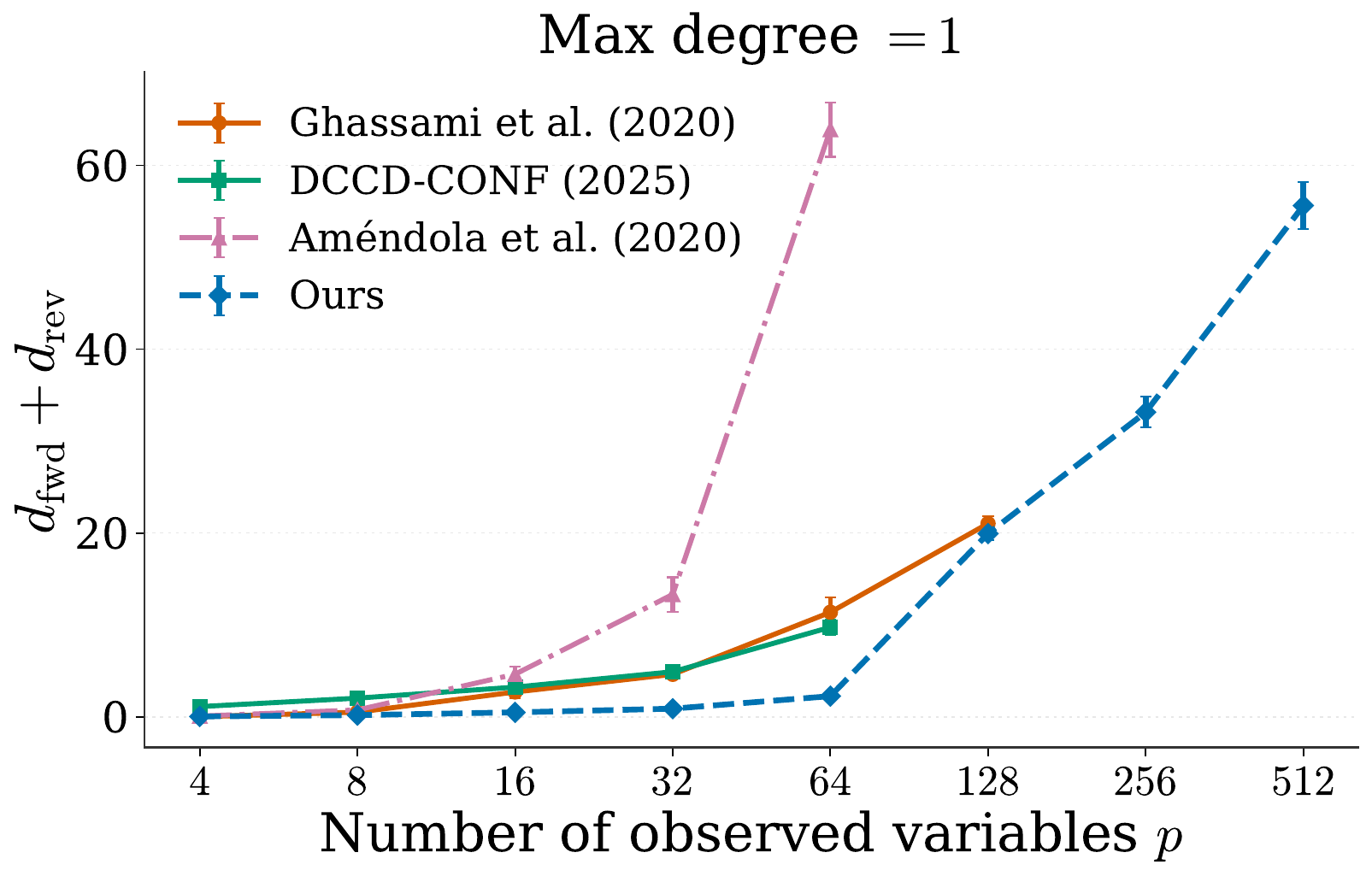}
        \caption{$30\%$ latent, degree 1}
    \end{subfigure}
    \hfill
    \begin{subfigure}[t]{0.32\linewidth}
        \centering
        \includegraphics[width=\linewidth]{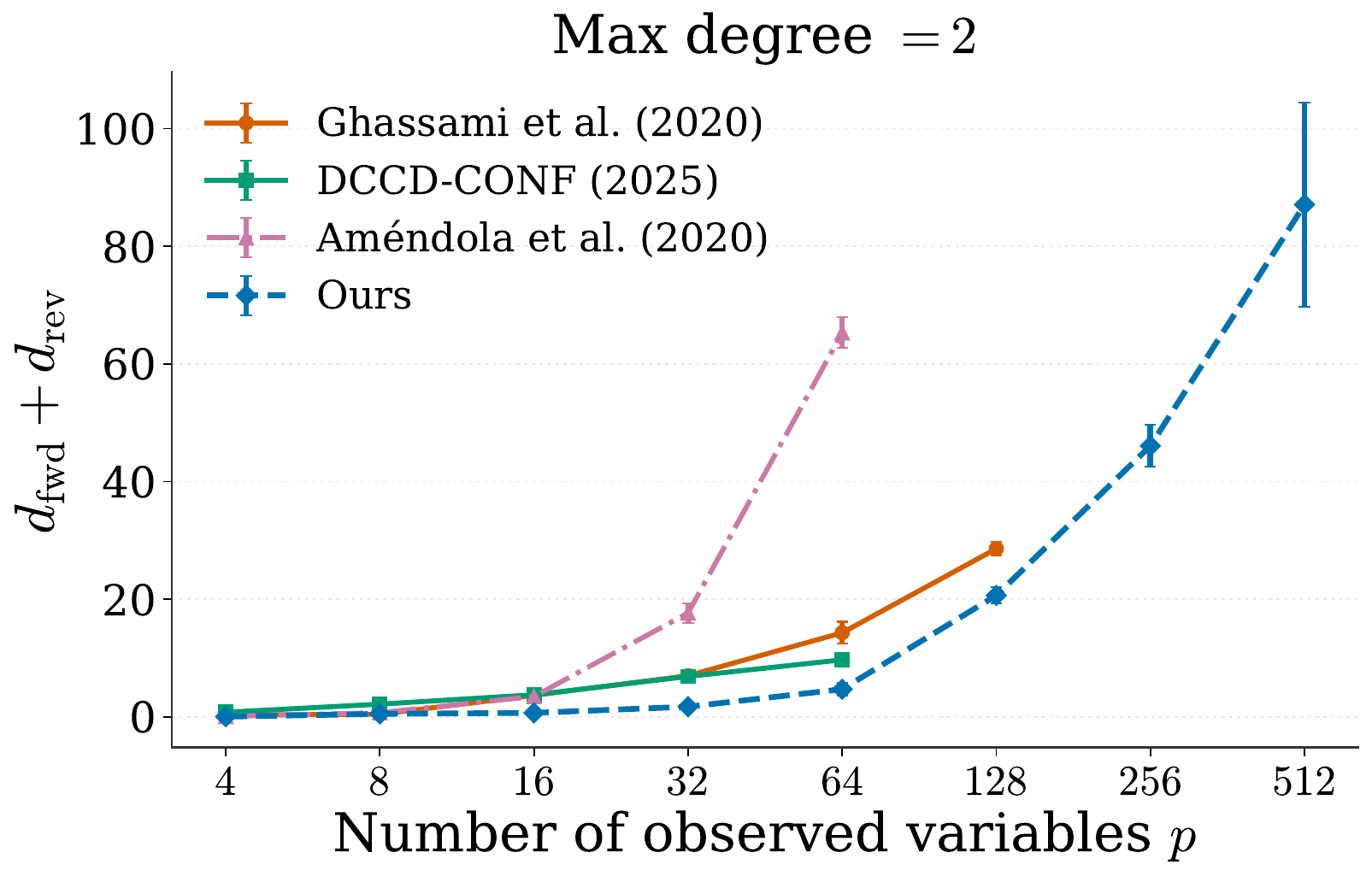}
        \caption{$30\%$ latent, degree 2}
    \end{subfigure}
    \hfill
    \begin{subfigure}[t]{0.32\linewidth}
        \centering
        \includegraphics[width=\linewidth]{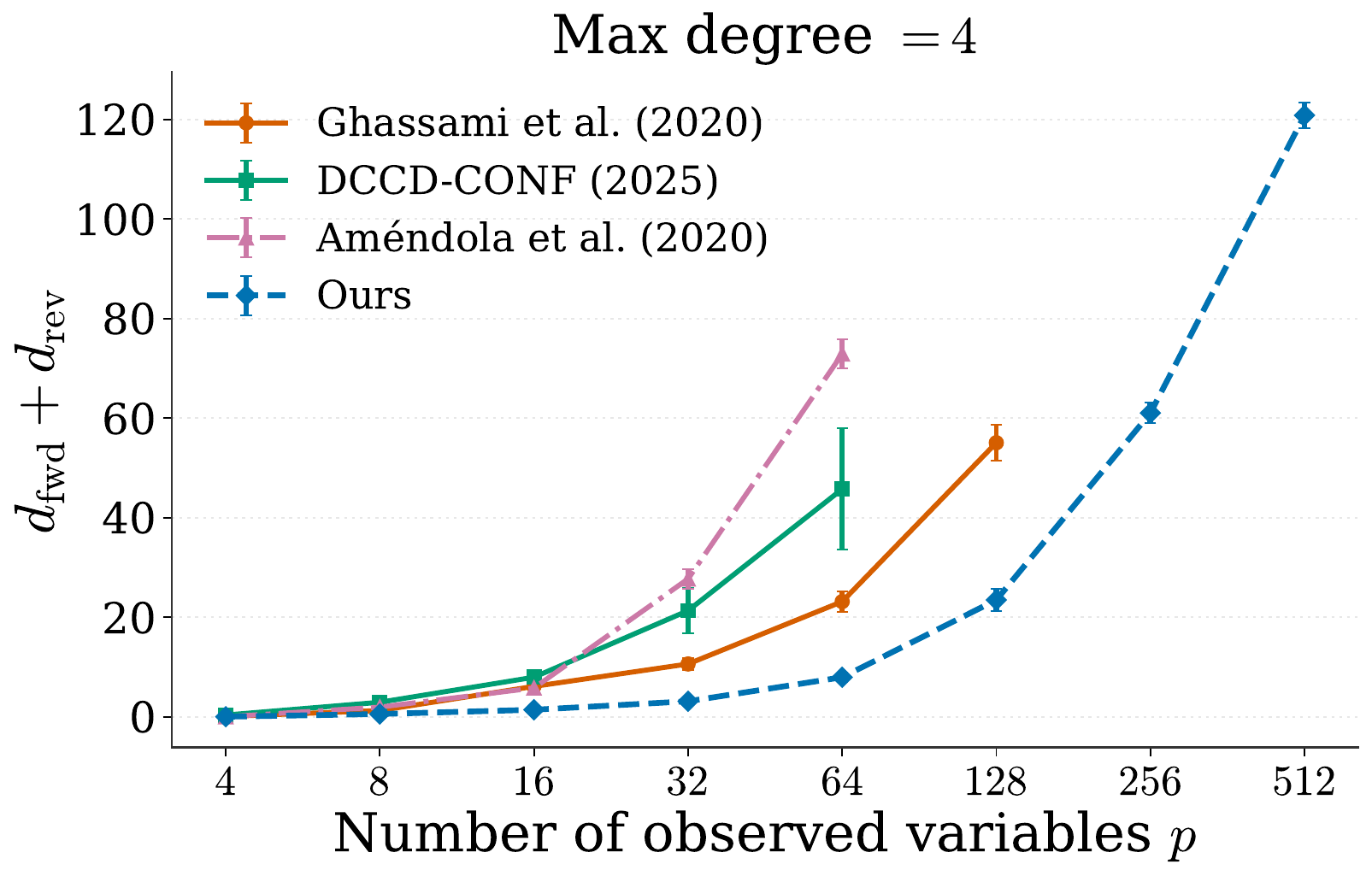}
        \caption{$30\%$ latent, degree 4}
    \end{subfigure}

    \caption{Sensitivity to graph density and latent confounding. Rows correspond
    to latent ratios $\ell/p$ of $10\%$, $20\%$, and $30\%$, and columns to
    maximum observed degrees of $1$, $2$, and $4$. Points show mean
    $d_{\mathrm{fwd}}+d_{\mathrm{rev}}$ with SEM across trials; lower is better.}
    \label{fig:stress-grid}
\end{figure*}

\paragraph{Effect of density and latent confounding.}
Increasing the maximum observed degree permits denser cyclic structure, while increasing $\ell/p$ introduces more exogenous common causes and enriches the disturbance term $\Lambda^\top\Omega_L\Lambda$. The results show that recovery depends jointly on graph size, density, and latent structure rather than on any single factor alone. Because confounding also depends on the sampled child sets and loading magnitudes, $\ell/p$ should be interpreted as the number of latent common causes rather than a direct measure of confounding strength. Each baseline is evaluated up to the largest graph size for which its computation remains feasible.

\clearpage
\begingroup
\color{black}

\section{Related work}
\label{app:related-work}

\paragraph{Distributional comparisons for cyclic Gaussian models.}
\citet{ghassami2020} study cyclic linear Gaussian models without latent
confounders. They compare the distribution families induced by compatible
parameters, define quasi-equivalence through positive-measure overlap,
and prove consistency of global penalized-likelihood minimizers under
generalized faithfulness and an additional structural condition
\citep[Definition~9, Assumption~1, Theorem~3]{ghassami2020}.
We extend the overlap-based perspective to observed covariance families
after marginalizing latent variables, requiring an intersection with the
dimension of both families; Appendix~\ref{app:quasi-equivalence-rationale}
explains the distinction. In the same no-latent setting,
\citet{yi2024} propose Filter, Rank, and Prune. Under their assumptions,
filtering retains every edge of at least one objective-minimizing graph
with high probability for sufficiently large samples
\citep[Theorem~3.4]{yi2024}. This is a filtering guarantee, not a
quasi-equivalence consistency result for the final graph.

\paragraph{Discovery with cycles and latent confounding.}
\citet[Theorem~12]{hyttinen2012} use randomized interventions to identify
observed coefficients and structural-noise covariances in linear cyclic
models under intervention and solvability conditions. Correlated noises
represent confounding without specifying individual latent variables.
\citet{forre2018} study nonlinear cyclic models with latent confounders
using observational and interventional conditional-independence
information. They assume that these independences agree exactly with
$\sigma$-separation, their graphical criterion for cyclic models; linear
models need not satisfy this assumption
\citep[Definition~2.15, Remark~2.16]{forre2018}. We instead use
observational data and compare complete Gaussian covariance families.

\paragraph{Comparison with Am\'endola et al.}
\label{app:amendola-comparison}
\citet[Section~2]{amendola2020} study cyclic linear Gaussian mixed graphs.
Bidirected edges permit correlated structural noises. Their simple-graph
restriction allows at most one edge of any type per vertex pair, excluding
reciprocal directed edges and directed--bidirected pairs. Our explicit
latent model allows both reciprocal edges and a directed edge between
variables sharing a latent parent, but the different representations of
confounding need not define the same covariance families.

Their Theorem~3.1 gives dimension $p+k$ for every simple mixed graph, where
$k$ counts directed and bidirected edges. Their Theorem~4.1 gives a
graphical sufficient condition for two such graphs to have equal
Euclidean covariance-family closures. Neither theorem establishes
statistical consistency of their score or search. Our results instead
concern selection by global score minimizers as the sample size increases,
with recovery up to marginal quasi-equivalence under identification
conditions.

Their proposed penalty adds an edge-count term to dimension-based BIC
\citep[Equations~(15)--(16)]{amendola2020}; ours counts directed edges and
latent variables. On latent-free directed graphs without reciprocal
edges, ordinary BIC and our score select the same globally optimal graphs;
this equivalence need not hold with their additional edge-count penalty.
Appendix~\ref{app:positive-minimality} relates their dimension theorem to
our structural-minimality condition for targets admitting such a
latent-free representative, allowing all our latent and cyclic
competitors. It does not verify structural minimality for general
confounded targets.

Their algorithm uses greedy edge additions, deletions, and reversals with
parameter fitting \citep[Section~5]{amendola2020}. We optimize Bernoulli
inclusion probabilities, although our implementation also uses discrete
refinement in no-latent experiments and fixed-graph refitting
(Appendix~\ref{sec:concrete}).

\paragraph{Latent confounding in acyclic models.}
\citet{kaltenpoth2023} learn independent exogenous latent confounders in
acyclic linear Gaussian models. Their identification results restrict the
sizes and incoming edges of groups sharing a latent confounder;
identifying observed directed coefficients additionally requires equal observed-noise
variances \citep[Sections~2.1--2.3, Theorem~2]{kaltenpoth2023}.
SPOT uses mixed graphs without directed cycles and represents confounding
through bidirected edges, rather than optimizing an explicit latent count
\citep[Sections~2--3]{ma2024spot}. Our model allows cycles and unequal
noise variances, with recovery up to marginal quasi-equivalence under
identification conditions rather than unique identification.

\paragraph{Continuous structure optimization.}
NOTEARS uses a smooth equality constraint characterizing acyclic
coefficient matrices \citep[Theorem~1]{zheng2018}. We require invertibility
of $I_p-B_{OO}$, not acyclicity. DCCD-CONF learns nonlinear cyclic models
from interventional data, with correlated Gaussian structural noises and
Bernoulli adjacency gates \citep[Sections~2--3]{sethuraman2025}.
Its straight-through estimator evaluates binary masks but computes
derivatives through continuous gates. Its theoretical result uses a
population score, with likelihood averaged under the data-generating
distributions. Under its assumptions and a suitable penalty, graphs
selected by exact maximizers admit the same family of distributions as
the true graph for the supplied interventions
\citep[Theorem~2, Definitions~A.8 and~A.10]{sethuraman2025}. Our guarantee
instead uses observational data and concerns increasing sample sizes. Bernoulli
graph gates are therefore not novel here; our formulation additionally
gates explicit latent activation and connections and proves equality of
global infima with the corresponding discrete objective.

The Concrete distribution provides continuous approximations to discrete
variables \citep[Section~3.2, Appendix~B]{maddison2017}.
Appendix~\ref{sec:concrete} distinguishes fully continuous masks from
straight-through optimization and describes our implementation.
Proposition~\ref{prop:bernoulli-exactness} concerns the exact Bernoulli
expectation, not either approximation, and does not guarantee numerical
convergence to a global minimum.

\endgroup

\end{document}